\pdfoutput=1
\documentclass{article} % For LaTeX2e
\usepackage{iclr2022_conference,times}

\usepackage
[acronym,smallcaps,nowarn,section,nogroupskip,nonumberlist]{glossaries}
\glsdisablehyper{}
\glsdisablehyper{}

\newacronym{KL}{KL}{Kullback-Leibler}
\newacronym{ELBO}{elbo}{\emph{evidence lower bound}}
\newacronym{MCMC}{mcmc}{Markov chain Monte Carlo}

\newacronym{ML}{ML}{machine learning}
\newacronym{VAE}{VAE}{variational auto-encoder}
\newacronym{VAETV}{VAE-TV}{variational auto-encoder with twin hidden variables}
\newacronym{AE}{AE}{auto-encoder}
\newacronym{SGD}{SGD}{stochastic gradient descent}
\newacronym[sort=beta]{BVAE}{\(\beta\)-vae}{}
\newacronym{TC}{TC}{total correlation}
\newacronym{DGM}{DGM}{deep generative model}
\newacronym{MMD}{MMD}{maximum mean discrepancy}
\newacronym{GAN}{GAN}{generative adversarial network}
\newacronym{AAE}{AAE}{adversarial auto-encoder}
\newacronym{CCM}{CCM}{constant curvature manifold}
\newacronym{IWAE}{IWAE}{importance weighted auto-encoder}
\newacronym{MC}{MC}{Monte Carlo}
\newacronym{NN}{NN}{neural network}
\newacronym{FNN}{FNN}{feedforward neural network}
\newacronym{MLP}{MLP}{multilayer perceptron}
\newacronym{SVM}{SVM}{support vector machine}
\newacronym{KDE}{KDE}{kernel density estimation}
\newacronym{RKHS}{RKHS}{reproducing kernel Hilbert space}
\newacronym{VI}{VI}{variational inference}
\newacronym{maxlike}{ML}{maximum likelihood}
\DeclareMathOperator{\E}{\mathbb{E}}
\DeclareMathOperator{\R}{\mathbb{R}}

\DeclareMathOperator{\X}{\mathcal{X}}
\DeclareMathOperator{\Y}{\mathcal{Y}}
\DeclareMathOperator{\Z}{\mathcal{Z}}
\DeclareMathOperator{\ELBO}{\mathcal{L}}

\DeclareMathOperator{\pdata}{p_{\text{data}}(x)}

\DeclareRobustCommand{\KL}[2]{\ensuremath{D_{\textrm{KL}}\left(#1\;\|\;#2\right)}}

\DeclareMathOperator*{\argmax}{arg\,max}

\usepackage[utf8]{inputenc} % allow utf-8 input
\usepackage[T1]{fontenc}    % use 8-bit T1 fonts
\usepackage{hyperref}       % hyperlinks
\usepackage[capitalise,noabbrev]{cleveref}
\usepackage{url}            % simple URL typesetting
\usepackage{booktabs}       % professional-quality tables
\usepackage{amsfonts}       % blackboard math symbols
\usepackage{nicefrac}       % compact symbols for 1/2, etc.
\usepackage{microtype}      % microtypography

\usepackage{nccmath}

\usepackage{menukeys}
\usepackage{multirow}
\usepackage{graphicx}
\usepackage{mathrsfs}
\usepackage{paralist}
\usepackage{subfig}
\usepackage{amsthm}
\usepackage{enumitem}
\usepackage{xspace}

\usepackage{amsmath,amssymb}
\usepackage[ruled]{algorithm2e}

\usepackage{xcolor}

\newcommand\Todo[1]{\textcolor{red}{\\Todo: #1}}

\newcommand{\method}{InteL-VAE\xspace}
\newcommand{\methods}{InteL-VAEs\xspace}
\usepackage{bbm}
\usepackage{floatrow}
\usepackage{float}

\usepackage{caption}
\usepackage{booktabs}

\usepackage{rotating}
\usepackage{dsfont}
\usepackage{tabularx}
\Crefname{equation}{Eq.}{Eqs.}
\Crefname{figure}{Fig.}{Figs.}
\Crefname{tabular}{Tab.}{Tabs.}
\Crefname{section}{Sec.}{Secs.}

\usepackage{thm-restate}

\usepackage[
  separate-uncertainty = true,
  multi-part-units = repeat
]{siunitx}

\usepackage{wrapfig}
\usepackage{pdfcomment}

\newcommand\given[1][]{\:#1\vert\:}

\hypersetup{
  colorlinks,
  linkcolor={red!50!black},
  citecolor={blue!50!black},
  urlcolor={blue!80!black}
  }

\makeatletter
\renewcommand\paragraph{\@startsection{paragraph}{4}{\z@}%
{0.6ex \@plus.2ex \@minus.2ex}%
{-1em}%
{\normalfont\normalsize\bfseries}}
\makeatother

\usepackage[title]{appendix}

\usepackage{verbatim}
\title{On Incorporating Inductive Biases into VAEs}

\author{
	Ning Miao\textsuperscript{1*}\quad Emile Mathieu\textsuperscript{1}\quad N. Siddharth\textsuperscript{2}\quad Yee Whye Teh\textsuperscript{1}\quad Tom Rainforth\textsuperscript{1*}\thanks{\textsuperscript{1}Department of Statistics, University of Oxford,~ \textsuperscript{2}University of Edinburgh}\thanks {\textsuperscript{*}Correspondence to: Ning Miao <ning.miao@stats.ox.ac.uk>, Tom Rainforth <rainforth@stats.ox.ac.uk>}
}

\renewcommand\footnotemark{}

\iclrfinalcopy % Uncomment for camera-ready version, but NOT for submission.

\begin{document}

\maketitle

\renewcommand{\thefootnote}{\arabic{footnote}}
\setcounter{footnote}{0}
\begin{abstract}
We explain why directly changing the prior can be a surprisingly ineffective mechanism for incorporating inductive biases into \glspl{VAE}, and introduce a simple and effective alternative approach: \emph{Intermediary Latent Space VAEs} (\methods).
\methods use an intermediary set of latent variables to control the stochasticity of the encoding process, before mapping these in turn to the latent representation using a parametric function that encapsulates our desired inductive bias(es).
This allows us to impose properties like sparsity or clustering on learned representations, and incorporate human knowledge into the generative model.  
Whereas changing the prior only indirectly encourages behavior through regularizing the encoder, \methods are able to directly enforce desired characteristics.
Moreover, they bypass the computation and encoder design issues caused by non-Gaussian priors, while allowing for additional flexibility through training of the parametric mapping function.
We show that these advantages, in turn, lead to both better generative models and better representations being learned.

\end{abstract}

\section{Introduction}
\label{sec:intro}
\Glspl{VAE} provide a rich class of \glspl{DGM} with many variants~\citep{kingma2013auto, rezende2015variational, burda2016importance, gulrajani2016pixelvae, vahdat2020nvae}. Based on an encoder-decoder structure, \glspl{VAE} encode datapoints into latent embeddings before decoding them back to data space.
By parameterizing the encoder and decoder using expressive neural networks, \glspl{VAE} provide a powerful basis for learning both generative models and representations.

The standard \gls{VAE} framework assumes an isotropic Gaussian prior.
However, this can cause issues, such as when one desires the learned representations to exhibit some properties of interest, for example sparsity~\citep{tonolini2020variational} or clustering~\citep{dilokthanakul2016deep}, or when the data distribution has very different topological properties from a Gaussian, for example multi-modality~\citep{shi2020dispersed} or group structure~\citep{falorsi2018Explorations}.
Therefore, a variety of recent works have looked to use non-Gaussian priors~\citep{van2017neural,tomczak2018vae,casale2018gaussian,razavi2019generating,bauer2019resampled}, often with the motivation of adding inductive biases into the model~\citep{s-vae18,mathieu2019disentangling,nagano2019Wrapped,skopek2019Mixedcurvature}.

In this work, we argue that this approach of using non-Gaussian priors can be a problematic, and even ineffective, mechanism for adding \emph{inductive biases} into~\glspl{VAE}.
Firstly, non-Gaussian priors will often necessitate complex encoder models to maintain consistency with the prior's shape and dependency structure~\citep{webb2018faithful}, which typically no longer permit simple parameterization.
Secondly, the latent encodings are still not guaranteed to follow the desired structure because the `prior' only appears in the training objective as a regularizer on the encoder.
Indeed,~\cite{mathieu2019disentangling} find that changing the prior is typically insufficient in practice to learn the desired representations at a \emph{population level}, with mismatches occurring between the data distribution and learned model.

To provide an alternative, more effective, approach that does not suffer from these pathologies, 
we introduce \emph{Intermediary Latent Space VAEs} (\methods), an extension to the standard VAE framework that allows a wide range of powerful inductive biases to be incorporated while maintaining an isotropic Gaussian prior.
This is achieved by introducing an \emph{intermediary} set of latent variables that deal with the stochasticity of the encoding process \emph{before} incorporating the desired inductive biases via a parametric function that maps these intermediary latents to the latent representation itself, with the decoder taking this final representation as input.  See~\cref{fig:model} for an example.

%ICLR
\begin{figure}[t]
 \centering
 \includegraphics[width=0.77\textwidth]{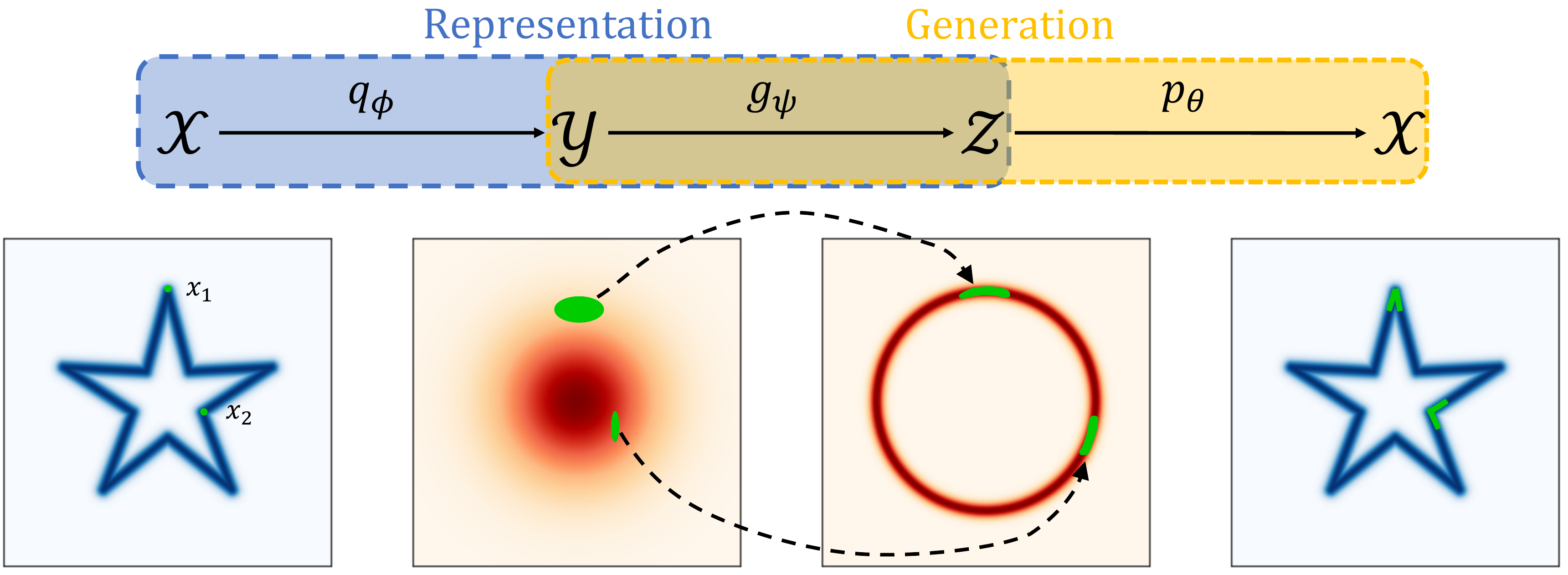}
 \vspace{-4pt}
 \caption{Example \method with star-like data. 
  We consider the auto-encoding for two example datapoints ($x_1$ and $x_2$, shown in green), which are first stochastically mapped to $\Y$ using a Gaussian encoder.
 This embedding is then pushed forward to $\Z$ using the \emph{non-stochastic} mapping $g_{\psi}$, which is a radial mapping to enforce a spherical distribution.
 Decoding is then done in the standard way from $\Z$, with the complexity of the decoder mapping simplified by the induced structural properties of $\Z$.
 \vspace{-12pt}
 }
 \label{fig:model}
\end{figure}

The \method framework provides a variety of advantages over directly replacing the prior.
Firstly, it directly enforces our inductive biases on the representations, rather than relying on the regularizing effect of the prior to encourage this implicitly.
Secondly, it provides a natural congruence between the generative and representational models via sharing of the mapping function, side-stepping the issues that non-Gaussian priors can cause for the inference model.
Finally, it allows for more general and more flexible inductive biases to be incorporated, by removing the need to express them with an explicit density function and allowing for parts of the mapping to be learned during training.

To further introduce a number of novel specific realizations of the \method framework, showing how they can be used to incorporate various inductive biases, enforcing latent representations that are, for example, multiply connected, multi-modal, sparse, or hierarchical. Experimental results show their superiority compared with baseline methods in both generation and feature quality, most notably providing state-of-the-art performance for learning sparse representations in the \gls{VAE} framework.

To summarize, we 
a) highlight the need for inductive biases in \glspl{VAE} and explain why directly changing the prior is a suboptimal means for incorporating them;
b) propose \methods as a simple but effective general framework to introduce inductive biases; and
c) introduce specific \method variants which can learn improved generative models and representations over existing baselines on a number of tasks.
Accompanying code is provided at {\scriptsize \url{https://github.com/NingMiao/InteL-VAE}}.

\vspace{-4pt}
\section{The Need for Inductive Biases in VAEs}
\vspace{-4pt}
\label{sec:background}
Variational auto-encoders (\glspl{VAE}) are deep stochastic auto-encoders that can be used for learning both deep generative models and low-dimensional representations of complex data.
Their key components are an encoder, $q_{\phi}(z|x)$, which probabilistically maps from data $x\in\X$ to latents $z \in \Z$; a decoder, $p_{\theta}(x|z)$, which probabilistically maps from latents to data; and a prior, $p(z)$, that completes the generative model, $p(z)p_{\theta}(x|z)$, and regularizes the encoder during training.
The encoder and decoder are parameterized by deep neural networks and are simultaneously trained using a dataset $\{x_1, x_2,..., x_N\}$ and a variational lower bound on the log-likelihood, most commonly,
\begin{align}
\label{eq:ELBO}
    \ELBO(x,\theta, \phi):=\E_{z\sim q_{\phi}(z|x)} \left[\log p_{\theta}(x|z)\right]- \KL{q_{\phi}(z|x)}{p(z)}.
\end{align}
Namely, we optimize $\ELBO(\theta, \phi) := \E_{x\sim \pdata}\left[\ELBO(x,\theta, \phi)\right]$, where $\pdata$ represents the empirical data distribution.
Here the prior is typically fixed to a standard Gaussian, i.e.~$p(z)=\mathcal{N}(z;0,I)$.

While it is well documented that this standard \gls{VAE} setup with a `Gaussian' latent space can be suboptimal~\citep{davidson2018Hyperspherical,mathieu2019disentangling,tomczak2018vae,bauer2019resampled,tonolini2020variational}, there is perhaps less of a unified high-level view on exactly when, why, and how one should change it to incorporate inductive biases.
Note here that the prior does not play the same role as in a Bayesian model: because the latents themselves are somewhat arbitrary and the model is learned from data, it does not encapsulate our initial beliefs in the way one might expect.

We argue that there are two core reasons why inductive biases can be important for \glspl{VAE}: (a) standard \glspl{VAE} can fail to encourage, and even prohibit, desired structure in the \emph{representations} we learn; and (b) standard \glspl{VAE} do not allow one to impart prior information or desired topological characteristic into the \emph{generative model}.

Considering the former, one often has some a priori desired characteristics, or constraints, on the representations learned~\citep{bengio2013representation}.
For example, sparse features can be desirable because they can improve data efficiency~\citep{yip1997sparse}, and provide robustness to noise~\citep{4483511,ahmad2019can} and attacks~\citep{gopalakrishnan2018combating}.
In other settings one might desire clustered~\citep{jiang2017variational}, disentangled~\citep{ansari2019hyperprior,pmlr-v80-kim18b,higgins2018towards} or hierarchical representations~\citep{song2013wavelbp,sonderby2016ladder,zhao2017learning}.
The KL-divergence term in~\cref{eq:ELBO} regularizes the encoding distribution towards the prior and, as a standard Gaussian distribution typically does not exhibit our desired characteristics, this regularization can significantly hinder our ability to learn representations with the desired properties.

Not only can this be problematic at an individual sample level, it can cause even more pronounced issues at the \emph{population level}: desired structural characteristics of our representations often relate to the pushforward distribution of the data in the latent space, $q_{\phi}(z):=\E_{\pdata}[q_{\phi}(z|x)]$, which is both difficult to control and only implicitly regularized to the prior~\citep{hoffman2016elbo}.

\begin{wrapfigure}[10]{r}{0.41\textwidth}
\vspace{-18pt}
 \centering
 \subfloat[Data]{
 \label{fig:motivation_MoG_data}
 \centering
 \includegraphics[width=0.4\textwidth]{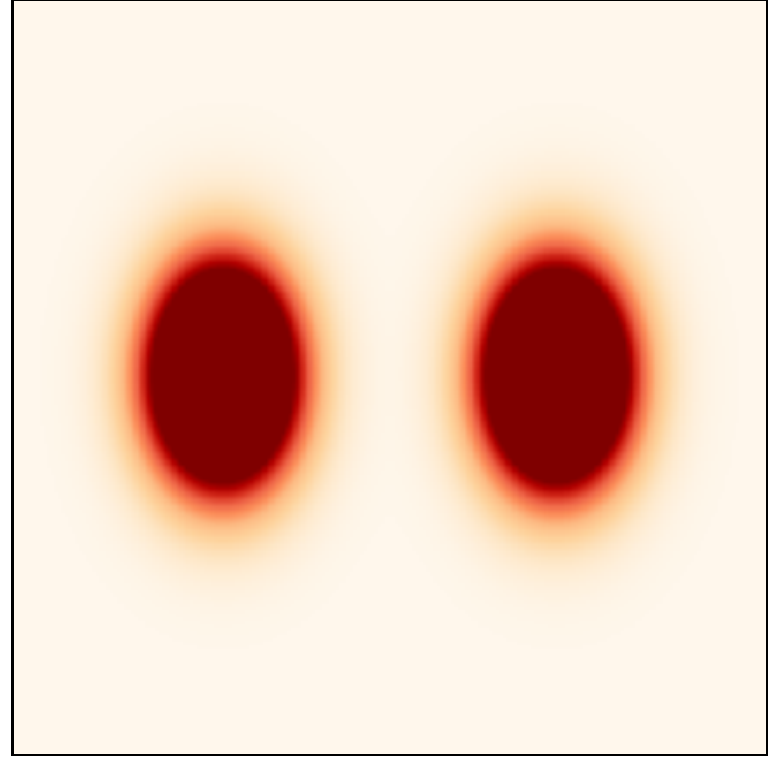}
 }
 \subfloat[\gls{VAE}]{
 \label{fig:motivation_MoG_VAE}
 \centering
 \includegraphics[width=0.4\textwidth]{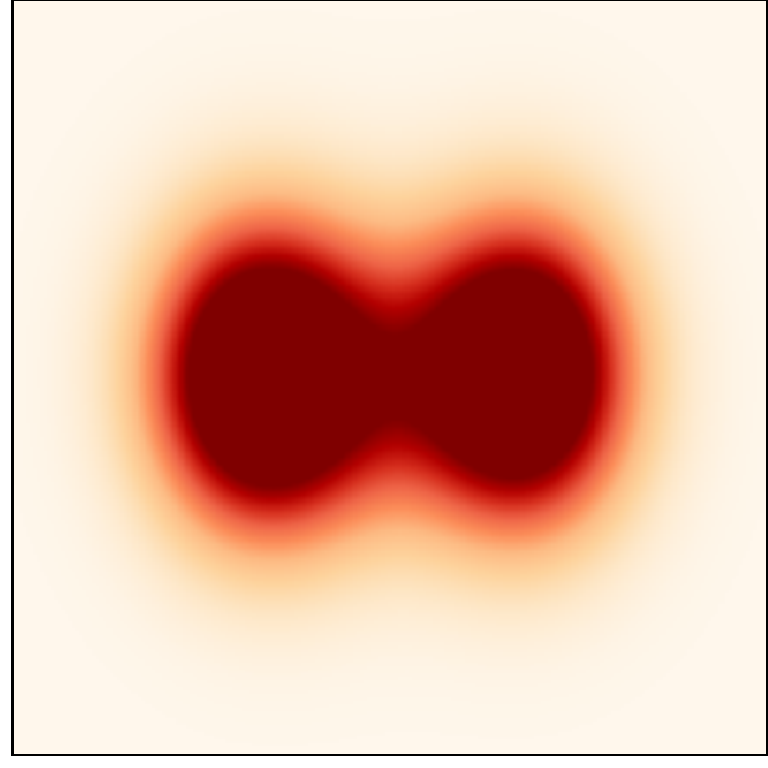}
 }
 \vspace{-8pt}
 \caption{VAE learned generative distribution $\mathbb{E}_{p(z)}[p_{\theta}(x|z)]$ for mixture data.}
 %\caption{MoG data distribution~$\pdata$, and learned distribution~$\mathbb{E}_{p(z)}[p_{\theta}(x|z)]$ by \gls{VAE}.}
 \label{fig:motivation_MoG}
\end{wrapfigure}
Inductive biases can also be essential to the generation quality of \glspl{VAE}: because the generation process of standard \glspl{VAE} is essentially pushing-forward the Gaussian prior on $\Z$ to data space $\X$ by a `smooth' decoder, there is an underlying inductive bias that standard \glspl{VAE} prefer sample distributions with similar topology structures to Gaussians.
As a result, \glspl{VAE} can perform poorly when the data manifold exhibits certain different topological properties~\citep{caterini2020variational}.
For example, they can struggle when data is clustered into unconnected components as shown in~\cref{fig:motivation_MoG}, or when data is not simply-connected.
This renders learning effective mappings using finite datasets and conventional architectures (potentially prohibitively) difficult.
In particular, it can necessitate large Lipschitz constants in the decoder, causing knock-on issues like unstable training and brittle models~\citep{scaman2018lipschitz}, as well as posterior collapse~\citep{van2017neural, alemi2018fixing}.
In short, the Gaussian prior of a standard \gls{VAE} can induce fundamental topological differences to the true data distribution~\citep{falorsi2018Explorations, shi2020dispersed}.

\vspace{-1pt}
\section{Shortfalls of VAEs with non-Gaussian Priors}
\label{sec:motivations}
\vspace{-1pt}
\begin{wrapfigure}[17]{r}{0.51\textwidth}
 \vspace{-21pt}
 \centering
 \subfloat[Directly replacing $p(z)$]{
 \label{fig:encoder_result}
 \centering
 \includegraphics[width=0.45\textwidth]{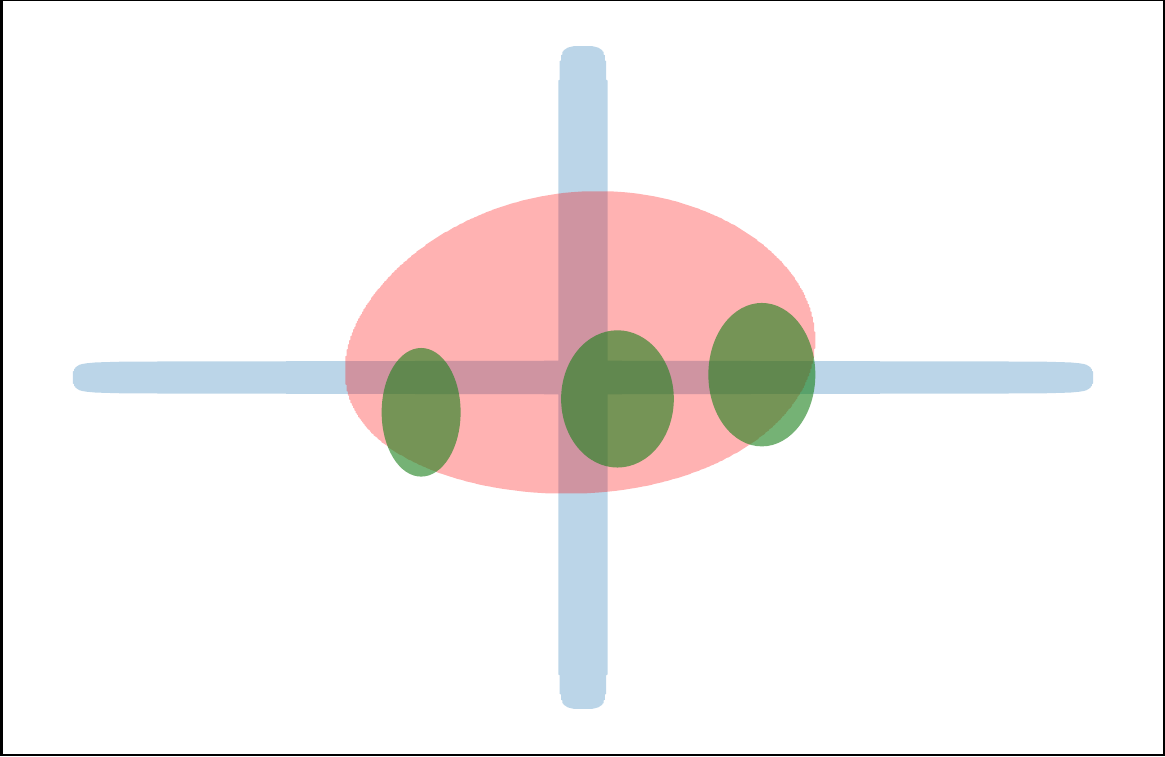}
 }
 \subfloat[\method]{
 \label{fig:encoder_result_ours}
 \centering
 \includegraphics[width=0.45\textwidth]{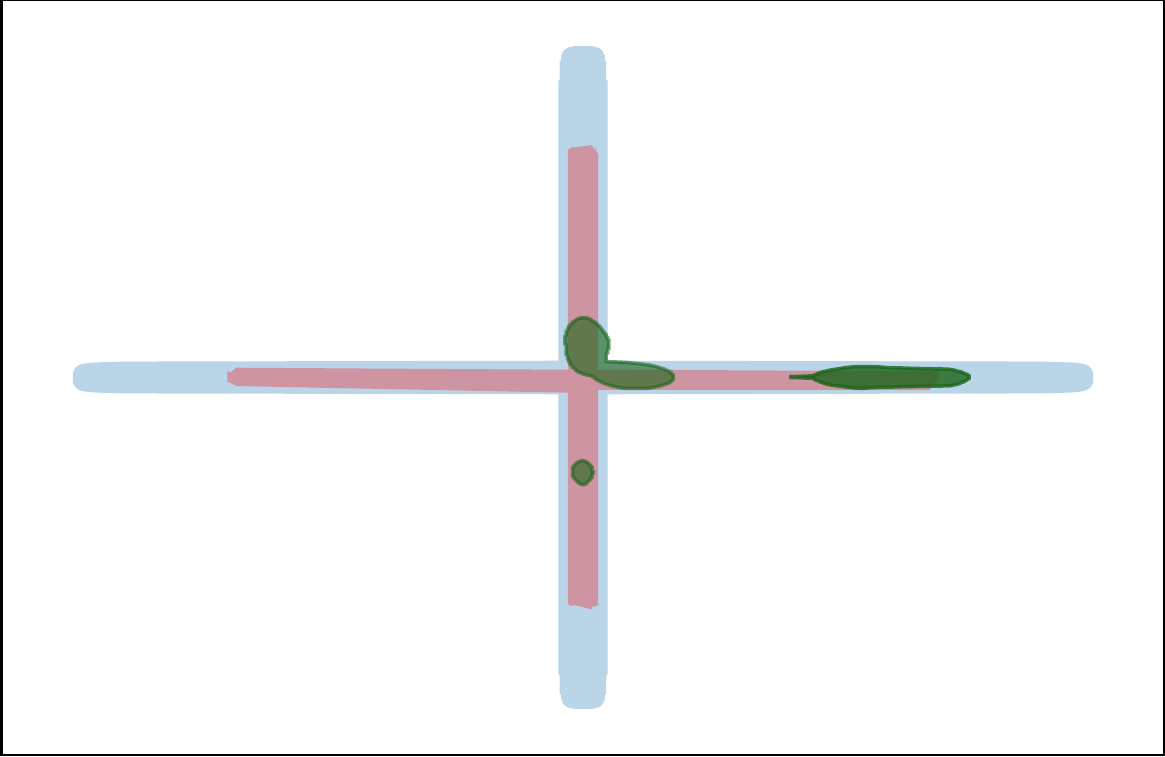}
 }
 \vspace{-6pt}
 \caption{Prior-encoder mismatch. We train (a) a \gls{VAE} with a sparse prior and (b) an \method with a sparse inductive bias on 2 dimensional sparse data.
 Figure shows target latent distribution $p(z)$ (blue), learned
 variational embeddings $q_{\phi}(z|x)$ of exemplar data (green), and data pushforward $q_{\phi}(z)$ (red shadow) for each method. Simply replacing the prior does not help the VAE match prior structure on either a per-sample or population level, whereas InteL-VAE produces an effective match.}
 \label{fig:encoder}
\end{wrapfigure}
Though directly replacing the Gaussian prior with a different prior sounds like a simple solution, effectively introducing inductive biases can, unfortunately, be more complicated.

Firstly, the only influence of the prior during training is as a regularizer on the encoder through the $\KL{q_{\phi}(z|x)}{p(z)}$ term.
This regularization is always competing with the need for effective reconstructions and only has an indirect influence on $q_{\phi}(z)$.
As such, simply replacing the prior can be an ineffective way of inducing desired structure at the population level~\citep{mathieu2019disentangling}, particularly if $p(z)$ is a complex distribution that it is difficult to fit (see, e.g.,~\cref{fig:encoder_result}).
Mismatches between $q_{\phi}(z)$ and $p(z)$ can also have further deleterious effects on the learned generative model: the former represents the distribution of the data in latent space during training, while the latter is what is used by the learned generative model, leading to unrepresentative generations if there is mismatch.

Secondly, it can be extremely difficult to construct appropriate encoder mappings and distributions for non-Gaussian priors.
While the typical choice of a mean-field Gaussian for the encoder distribution is simple, easy to train, and often effective for Gaussian priors, it is often inappropriate for other choices of prior.
For example, in~\cref{fig:encoder}, we consider replacement with a sparse prior.
A VAE with a Gaussian encoder struggles to encode points in a manner that even remotely matches the prior.
One might suggest replacing the encoder distribution as well, but this has its own issues, most notably that other distributions can be hard to effectively parameterize or train.
In particular, the form of the required encoding noise might become heavily spatially variant; in our sparse example, the noise must be elongated in a particular direction depending on where the mean embedding is.
If the prior has constraints or topological properties distinct from the data, it can even be difficult to learn a mean encoder mapping that respects these, due to the continuous nature of neural networks.
%For example, a standard encoder architecture would struggle to match a prior defined on a hypersphere.

\vspace{-1pt}
\section{The \method Framework}
\vspace{-1pt}
\label{sec:method}
\begin{comment}
%Moved to page 2 for ICLR
\begin{figure}[t]
 \centering
 \includegraphics[width=0.93\textwidth]{figure/model.pdf}
 \vspace{-4pt}
 \caption{Example \method with star-like data. 
  We consider the auto-encoding for two example datapoints ($x_1$ and $x_2$, shown in green), which are first stochastically mapped to $\Y$ using a Gaussian encoder.
 This embedding is then pushed forward to $\Z$ using the \emph{non-stochastic} mapping $g_{\psi}$, which is a radial mapping to enforce a spherical distribution.
 Decoding is then done in the standard way from $\Z$, with the complexity of the decoder mapping simplified by the induced structural properties of $\Z$.
 \vspace{-12pt}
 }
 \label{fig:model}
\end{figure}
\end{comment}

To solve the issues highlighted in the previous section, and provide a principled and effective method for adding inductive biases to \glspl{VAE}, we propose \emph{Intermediary Latent Space VAEs} (\methods).
The key idea behind \methods is to introduce an \emph{intermediary} set of latent variables $y\in\Y$, used as a stepping stone in the construction of the \emph{representation} $z\in\Z$.
Data is initially encoded in $\Y$ using a conventional \gls{VAE} encoder (e.g.~a mean-field Gaussian) before being passed through a \emph{non-stochastic} mapping $g_{\psi}: \Y\mapsto \Z$ that incorporates our desired inductive biases and which can be trained, if needed, through its parameters $\psi$.
The prior is defined on $\Y$ and taken to be a standard Gaussian, $p(y)=\mathcal{N}(y;0,I)$,
while our representations, $z=g_{\psi}(y)$, correspond to a pushforward of $y$.
By first encoding datapoints to $y$, rather than $z$ directly, 
we can deal with all the encoder and prior stochasticity in this first, well-behaved, latent space, while maintaining $z$ as our representation and using it for the decoder $p_{\theta}(x|z)$.
%We can incorporate various desired properties in our representations by designing appropriate forms for $g_{\psi}$ (see \cref{sec:examples_and_experiments}). 
In principle, $g_{\psi}$ can be any arbitrary parametric (or fixed) mapping, including non-differentiable or even discontinuous functions. 
However, to allow for reparameterized gradient estimators~\citep{kingma2013auto,rezende2015variational}, we will restrict ourselves to $g_{\psi}$ that are sub-differentiable (and thus continuous) with respect to both their inputs and parameters.
Note that setting $g_{\psi}$ to the identity mapping recovers a conventional \gls{VAE}.

\vspace{-4pt}
As shown in \cref{fig:model}, the auto-encoding process is now $\X\xrightarrow{q_{\phi}} \Y\xrightarrow{g_{\psi}} \Z\xrightarrow{p_{\theta}} \X$.
This three-step process no longer unambiguously fits into the encoder-decoder terminology of the standard VAE and permits a variety of interpretations; for now we take the convention of calling $q_{\phi}(y|x)$ the encoder and $p_{\theta}(x|z)$ the decoder, but also discuss some alternative interpretations below.
We emphasize here that these no longer respectively match up with our representation model---which corresponds to passing an input into the encoder and then mapping the resulting encoding using $g_{\psi}$---and our generative model---which corresponds to $\mathcal{N}(y;0,I)p_{\theta}(x|z=g_{\psi}(y))$, such that we sample a $y$ from the prior and then pass this through through $g_{\psi}$ and the decoder in turn.
%Once trained, our learned generative model corresponds to $p(y)p_{\theta}(x|z=g_{\psi}(y))$, such that we can generate samples by firstly drawing from a standard Gaussian to generate a $y$, then passing this sample through $g_{\psi}$ and the decoder in turn.
%Our learned representation of a datapoint $x$, meanwhile, is constructed by passing $x$ through the encoder $q_{\phi}$ and then $g_{\psi}$ in turn.
%Our learned stochastic representation of a datapoint $x$, meanwhile, is given by $z=g_{\psi}(y)$, where $y\sim q_{\phi}(y|x)$.
%When using a Gaussian encoder, we can also construct the deterministic representation $g_{\psi}(\mu_{\phi}(x))$, where $\mu_{\phi} : \X \mapsto \Y$ is the mean function, for downstream tasks.
%.\footnote{Presuming that we use a Gaussian encoder.  As with standard \glspl{VAE}, other choices are possible here, such as normalizing flows~\citep{kingma2016improved}, wherein we would instead draw a sample from $q_{\phi}(y|x)$.}
%Putting these together, we see that $g_{\psi}$ can be thought of as being a mapping that is \emph{shared} between the encoder and decoder, thereby naturally ensuring consistency between the two.

The mapping $g_{\psi}$ introduces inductive biases into \emph{both} the generative model and our representations by imposing a particular form on $z$, such as the spherical structure enforced in \cref{fig:model} (see also \cref{sec:examples_and_experiments}).
It can be viewed as a \emph{shared module} between them, ensuring congruence between the two.
This congruence allows us to more directly introduce inductive biases through careful construction of $g_{\psi}$, without complicating the process of learning an effective inference network.
In particular, because $\Y$ is treated as our latent space for the purposes of training, we sidestep the inference issues that non-Gaussian priors usually cause.
Moreover, because all samples must explicitly pass through $g_{\psi}$ during both training and generation, we can more directly ensure the desired structure is enforced without causing a mismatch in the latent distribution between training and deployment.
%Furthermore, by pushing-forward the Gaussian prior on $\Y$ to a distribution on $\Z$, $g_{\psi}$ defines a representation distribution \emph{implicitly}.
%This allows for the simple definition of complex distributions on representations, and also for certain aspects of the distributions to be learned during training by the trainable parameters in $g_{\psi}$. 
%In summary, \method provides a general framework to introduce inductive biases to \glspl{VAE}, which can be trained as easily and stably as a standard VAE.

%\subsection{Training}

\textbf{Training}~~As with standard~\glspl{VAE}, training of an~\method is done by maximizing a variational lower bound (ELBO) on the log evidence, which we denote $\ELBO_{\Y}$.  Most simply, we have
\begin{align}
\begin{split}
    \log p_{\theta,\psi}(x) :=&\,
    \log \left(\E_{p(y)} \left[p_{\theta}(x|g_{\psi}(y))\right]\right) = \log \left(\E_{q_{\phi}(y|x)} \left[ \frac{p_{\theta}(x|g_{\psi}(y))\mathcal{N}(y;0,I)}{q_{\phi}(y|x)}\right]\right) \\
    \geq& \E_{q_{\phi}(y|x)} [\log p_{\theta}(x|g_{\psi}(y))] 
        - \KL{q_{\phi}(y|x)}{\mathcal{N}(y; 0, I)} 
        =: \ELBO_{\Y}(x,\theta,\phi,\psi).
\end{split}
\end{align}
Note that the regularization is on $y$, but our representation corresponds to $z = g_{\psi}(y)$.
Training corresponds to the optimization $\argmax_{\theta,\phi,\psi} \E_{x\sim \pdata} \left[\ELBO_{\Y}(x,\theta,\phi,\psi)\right]$, which can be performed using stochastic gradient ascent with reparameterized gradients in the standard manner.
Although inductive biases are introduced, the calculation, and optimization, of $\ELBO _{\Y}$ is thus equivalent to the standard ELBO. 
In particular, parameterizing $q_{\phi}(y|x)$ with a Gaussian distribution still yields an analytical  $\KL{q_{\phi}(y|x)}{\mathcal{N}(y; 0, I)}$ term.
%if $q_{\phi}(y|x)$ is a Gaussian encoding distribution on $\Y$, then the $\KL{q_{\phi}(y|x)}{\mathcal{N}(y; 0, I)}$ term can still be calculated analytically.

\begin{comment}
%ICLR
We note that \methods can also be used with variational bounds more generally, rather than just the standard ELBO.
During training, we can simply treat $g_{\psi}$ as if it were part of the decoder and $y$ as our latent variables.
Consequently the approach can be trivially extended to more general variational bounds and \gls{VAE} setups, such as IWAE~\citep{burda2016importance}, InfoVAE~\citep{zhao2019infovae}, and SMC--based methods~\citep{le2018auto,naesseth2018variational,maddison2017filtering}.
\end{comment}

%\subsection{Intuitions}
%\label{sec:intuitions}

\textbf{Alternative Interpretations}~~ It is interesting to note that our representation, $g_{\psi}(y)$, only appears in the context of the decoder in this training objective.
As such, we see that an important alternative interpretation of \methods is to consider $g_{\psi}$ as being a customized first layer in the decoder, and our test--time representations as partial decodings of the latents $y$.
This viewpoint allows it to be applied with more general bounds and VAE variants (e.g.~\citet{burda2016importance,le2018auto,maddison2017filtering,naesseth2018variational,zhao2019infovae}), as it requires only a carefully customized decoder architecture during training and an adjusted mechanism for constructing representations at test--time.
%Another, as discussed in~\Cref{sec:intuitions}, is to consider $z$ as our latents and $g_{\psi}$ as part of the encoder which is shared with the prior.

%We can provide further insight into \methods by noting that in the special case where $g_{\psi}$ is invertible, then $\ELBO _{\Y}$ equals the standard ELBO on the space of $\Z$ as follows

%We now provide further insights into \methods, starting with the following 
Yet another interpretation is to think about \methods as implicitly defining a conventional \gls{VAE} with latents $z$, but where both the non-Gaussian prior, $p_{\psi}(z)$, and our encoder distribution, $q_{\phi,\psi}(z|x)$, are themselves defined implicitly as pushforwards along $g_{\psi}$, which acts 
% and contain $g_{\psi}$ 
as a shared module that instills a natural compatibility between the two.
Formally we have the following theorem.
%
%\begin{theorem}
\begin{restatable}{theorem}{maintheorem}
Let $p_{\psi}(z)$ and $q_{\phi,\psi}(z|x)$ represent the respective pushforward distributions of $\mathcal{N}(0,I)$ and $q_{\phi}(y|x)$ induced by the mapping $g_{\psi} : \Y \mapsto \Z$. 
The following holds for all measurable $g_{\psi}$:
\begin{align}
\label{eq:bound_on_KL}
    \KL{q_{\phi,\psi}(z|x)}{p_{\psi}(z)} \le \KL{q_{\phi}(y|x)}{\mathcal{N}(y; 0, I)}.
\end{align}
If $g_\psi$ is also an invertible function then the above becomes an equality and $\ELBO _{\Y}$ equals the standard ELBO on the space of $\Z$ as follows
\begin{align}
\label{eq:elbo_equivalence}
    \ELBO_{\Y}(x,\theta,\phi,\psi) = 
    \E_{q_{\phi,\psi}(z|x)} [\log p_{\theta}(x|z)]-\KL{q_{\phi,\psi}(z|x)}{p_{\psi}(z)}.
\end{align}
\end{restatable}
%\end{theorem}
%

\begin{comment}
\ICLR
\begin{proof}
At a high--level, the result follows predominantly from the data processing inequality \citep{sason2019DataProcessing}.
Informally, further processing of random variables can only bring them ``closer'' in the space of distributions, leading to~\cref{eq:bound_on_KL}.
Equality is reached when no information is lost, which occurs when $g_\psi$ is invertible.
Coupling this with the fact that the reconstruction terms are identical under the mapping then leads to~\cref{eq:elbo_equivalence}.
See \cref{sec:proof} for a complete formal proof.
\vspace{-4pt}
\end{proof}
\end{comment}
%While this equivalence provides 
%As well as providing reassurance that $\ELBO _{\Y}$ is a sensible objective, 
%it is perhaps the form of $\ELBO _{\Y}$ itself that is most insightful.
%In particular, it 
%this result shows us how the \method framework can be thought of a natural means to define the implicit distributions $p_{\psi}(z)$ and $q_{\phi,\psi}(z|x)$ in a way that ensures consistency between the two.
The proof is given in \cref{sec:proof}.
Here,~\eqref{eq:bound_on_KL} shows that the divergence in our representation space $\Z$ is never more than that in $\Y$, or equivalently that the implied ELBO on the space of $\Z$ is always at least as tight as that on $\Y$;~\eqref{eq:elbo_equivalence} shows they are exactly equal if $g_{\psi}$ is invertible.
%During training, we can think of $\Y$ as our latent space and \methods as being equivalent to a vanilla \glspl{VAE} with Gaussian priors and encoders, but with a particular form of decoder that starts with the mapping $g_{\psi}$.
As the magnitude of $\KL{q_{\phi}(y|x)}{\mathcal{N}(y; 0, I)}$ in an \method will remain comparable to the KL divergence in a standard Gaussian prior \gls{VAE} setup, this, in turn, ensures that $\KL{q_{\phi,\psi}(z|x)}{p_{\psi}(z)}$ does not become overly large.
This is in stark contrast to the conventional non-Gaussian prior setup, where it can be difficult to avoid $\KL{q_{\phi}(z|x)}{p_{\psi}(z)}$ exploding without undermining reconstruction~\citep{mathieu2019disentangling}.
%Namely, rather than thinking about $\KL{q_{\phi,\psi}(z|x)}{p_{\psi}(z)}$ directly, we reason about $\KL{q_{\phi}(y|x)}{\mathcal{N}(y; 0, I)}$ instead, such that stochasticity of the latent space is dealt with in the well-behaved `Gaussian' space $\Y$, rather than the potentially complex $\Z$.
%This means it is easy to define an appropriate encoder distribution, as a standard mean-field Gaussian encoder is inherently compatible with the prior.
%: we are simply taking a KL divergence between two Gaussians.
The intuition here is that having the stochasticity in the encoder \emph{before} it is passed through $g_{\psi}$ ensures that the form of the noise in the embedding is inherently appropriate for the space: the same mapping is used to warp this noise as to define the generative model in the first place.
For example, when $g_{\psi}$ is a sparse mapping, the Gaussian noise in $q_{\phi}(y|x)$ will be compressed to a sparse subspace by $g_{\psi}$, leading to a sparse variational posterior $q_{\phi,\psi}(z|x)$ as shown in~\cref{fig:encoder_result_ours}.
In particular, $q_{\phi}(y|x)$ does not need to learn any complex spatial variations that result from properties of $\Z$.
In turn, \methods further alleviate issues of mismatch between $p_{\psi}(z)$ and $q_{\phi,\psi}(z)$.

\textbf{Further Benefits}~~ A key benefit of \methods is that the extracted features are \emph{guaranteed} to have the desired structure. 
Take the spherical case for example, all extracted features $g_{\psi}(\mu_{\phi}(x))$ lie within a small neighborhood of the unit sphere. 
By comparison, %\cite{mathieu2019disentangling} and other 
methods based on training loss modifications, e.g.~\cite{mathieu2019disentangling}, often fail to generate features with the targeted properties.

A more subtle advantage is that we do not need to explicitly specify $p_{\psi}(z)$. 
This can be extremely helpful when we want to specify complex inductive biases: designing a non-stochastic mapping is typically much easier than a density function, particularly for complex spaces.
Further, this can make it much easier to parameterize and learn aspects of $p_{\psi}(z)$ in a data-driven manner (see e.g.~Sec.~\ref{sec:sparse}). 
%if desired, rather than having them hard coded.
%For example, in Sec.~\ref{sec:sparse}, the algorithm automatically learns a sparse mapping that best fits the real data distribution.

%\vspace{-4pt}
\section{Related work}
\label{sec:related}
\vspace{-2pt}
%In this section, we describe three existing methods of using non-Gaussian priors in \gls{VAE} and compare them with \method. 

%Besides a method that allows efficient use of non-Gaussian priors, \method can also be regarded as a novel way of introducing inductive biases to generative models.
%For example, by using the spherical mapping, we tell the model that there is a hole in support of the data distribution, and by the sparse mapping, we introduce the belief that most data points lie in a union of low-dimensional manifolds.
\textbf{Inductive biases}~~ There is much prior work on introducing human knowledge to deep learning models by structural design, such as CNNs~\citep{lecun1989backpropagation}, RNNs~\citep{hochreiter1997long} and  transformers~\citep{vaswani2017attention}.
However, most of these designs are on the \emph{sample} level, utilizing low--level information such as transformation invariances or internal correlations in each sample. 
By contrast, \methods provide a convenient way to incorporate \emph{population} level knowledge---information about the global properties of data distributions can be effectively utilized. % by \method .

\textbf{Non-Gaussian priors}~~ There is an abundance of prior work utilizing non-Gaussian priors to improve the fit and generation capabilities of \glspl{VAE}, including MoG priors~\citep{dilokthanakul2016deep,shi2020dispersed}, sparse priors~\citep{mathieu2019disentangling, tonolini2020variational, barello2018SparseCoding}, Gaussian-process priors~\citep{casale2018gaussian} and autoregressive priors~\citep{razavi2019generating, van2017neural}. However, these methods often require specialized algorithms to train and are primarily applicable only to specific kinds of data.
Moreover, as we have explained, changing the prior alone often provides insufficient pressure on its own to induce the desired characteristics.
%For example, \cite{dilokthanakul2016deep} and \cite{shi2020dispersed} use MoG priors to fit clustered data distributions. \cite{casale2018gaussian} use a Gaussian Process prior to account for between-sample covariances in time-series data. And \cite{razavi2019generating} use the autoregressive prior~\citep{van2017neural} to generate high-fidelity images.
Others have proposed non-Gaussian priors to reduce the prior-posterior gap, such as Vamp-VAE~\citep{tomczak2018vae} and LARS~\citep{bauer2019resampled}, but these are tangential to our inductive bias aims. 
%Though effective in some cases, these methods do not allow for explicit expression of inductive biases which can be critical, and is a key feature of our approach.

%emile{Could cut this paragraph here}
%To extract features with desired properties, \cite{mathieu2019disentangling}  directly change the prior in the original Euclidean space, for example, explicitly introducing a Student-t, a sparse, and a clustered distribution as the priors.
%Besides leading to topological and numerical challenges as described in \cref{sec:encoder_mismatch}, this method cannot guarantee the extracted features to be disentangled, sparse, or clustered.
%This method can adopt to any prior distributions. However, though a regularization term is added to push aggregated posterior $Q_Z$ to approximate $P_Z$, whether $Q_Z$ will have the disentangled, sparse or clustered properties is not guaranteed.
%Moreover, these methods require the p.d.f of the target prior, which is hard to obtain or even don't exist (for low-dimensional priors).

\textbf{Non-Euclidean latents}~~
A related line of work has focused on non-Euclidean latent spaces.
For instance \cite{davidson2018Hyperspherical} leveraged a von Mises-Fisher distribution on a hyperspherical latent space, \cite{falorsi2018Explorations} endowed the latent space with a $\text{SO}(3)$ group structure, and \cite{mathieu2019Continuous,ovinnikov2019Poincar,nagano2019Wrapped} with hyperbolic geometry.
Other spaces like product of constant curvature spaces \citep{skopek2019Mixedcurvature} and embedded manifolds \citep{rey2019Diffusion} have also been considered.
However, these works generally require careful design and training.
%This work is of strong interest but requires subtle design and corresponding care in training.
% A second kind of approach employs a non-Euclidean latent space ($\Z$), implying that one may choose a new prior that is easy to parameterize on it, rather than a Gaussian prior, defined on a Euclidean space. 
% For example, \cite{s-vae18} replace a Euclidean latent space with a hyperspherical latent space and apply a von Mises-Fisher distribution as the variational family. Other kind of latent spaces, such as hyperbolic~\citep{mathieu2019Continuous,nagano2019Wrapped}, product of constant curvature
% Riemannian manifolds~\citep{skopek2019Mixedcurvature} and even Lie groups~\citep{falorsi2018Explorations} work well on particular data. %, such as tree-like data.
% This method does not harm the efficiency of \gls{VAE}, but is only applicable to specific prior distributions and requires subtle design and corresponding care in learning for each case. % kind of distribution. 

\textbf{Normalizing flows}~~ Our use of a non-stochastic mapping shares some interesting links to normalizing flows (NFs)~\citep{rezende2015variational, papamakarios2019normalizing,grathwohl2018FFJORD,dinh2017Density,huang2018Neural,papamakarios2018Masked}.  
Indeed a NF would be a valid choice for $g_{\psi}$, albeit an unlikely one due to their architectural constraints.
However, unlike previous use of NFs in \glspl{VAE}, our $g_{\psi}$ is crucially \emph{shared} between the generative and representational models,
%the encoder and decoder, 
rather than just being used in the encoder, while the KL divergence in our framework is taken before, not after, the mapping.
Moreover, the underlying motivation, and type of mapping typically used, differs substantially: our mapping is used to introduce inductive biases, not purely to improve inference.
Our mapping is also more general than a NF (e.g.~it need not be invertible) and does not introduce additional constraints or computational issues.

\vspace{-2pt}
\section{Specific Realizations of the \method Framework}
\vspace{-2pt}
\label{sec:examples_and_experiments}
We now present several novel example \methods, introducing various inductive biases through different choices of $g_{\psi}$. 
We will start with artificial, but surprisingly challenging, examples where some precise topological properties of the target distributions are known, incorporating them directly through a fixed $g_{\psi}$.
We will then move onto experiments where we impose a fixed clustering inductive bias when training on image data, allowing us to learn \methods that account effectively for multi-modality in the data distribution.
Finally, we consider the example of learning sparse representations of high--dimensional data.
Here we will see that it is imperative to exploit the ability of \methods to learn aspects of $g_{\psi}$ during training, providing a flexible inductive bias framework, rather than a pre-fixed mapping.
By comparing \methods with strong baselines, we show that \methods are effective in introducing these desired inductive biases, and consequently both improve generation quality and learn better data representations for downstream tasks.
One note of particular importance is that we find that \methods provide state-of-the-art performance for learning sparse \gls{VAE} representations.
A further example of using \methods to learn hierarchical representations is presented in~\cref{sec:hierarchical}, while full details on the various examples are given in~\cref{sec:mapping_detail}.

%We now propose several specific realizations of the deterministic mapping $g_{\psi}$ and demonstrate their usefulness on a variety of problems.
%In particular, we introduce mappings $g_{\psi}$ to induce multiply-connected, multi-modal, sparse, and hierarchical representations.
%We will drop the $\psi$ subscript from $g_{\psi}$ in cases where there are no tunable parameters.
% In this section, we assess the effectiveness of \method with different latent mappings $g_{\psi}$. 
% By properly designing the structural mapping $g_{\psi}$, we show \method can deal with a variety of data distributions, with different topological structures or complex hierarchies.
% By comparing \method with baseline models on several synthetic and real data sets, we show that \method improves the generation quality as well as extracts more useful features, which also boost the performance of downstream tasks.

\subsection{Multiple--Connectivity}
\label{sec:mapping_multiply_connected}

\begin{wrapfigure}[20]{r}{0.45\columnwidth}
\vspace{-5em}
% \hspace{-1em}
%  \centering
 \subfloat{
 \label{fig:square_real}
%  \centering
 \includegraphics[trim=50 50 50 50, clip, width=0.28\textwidth]{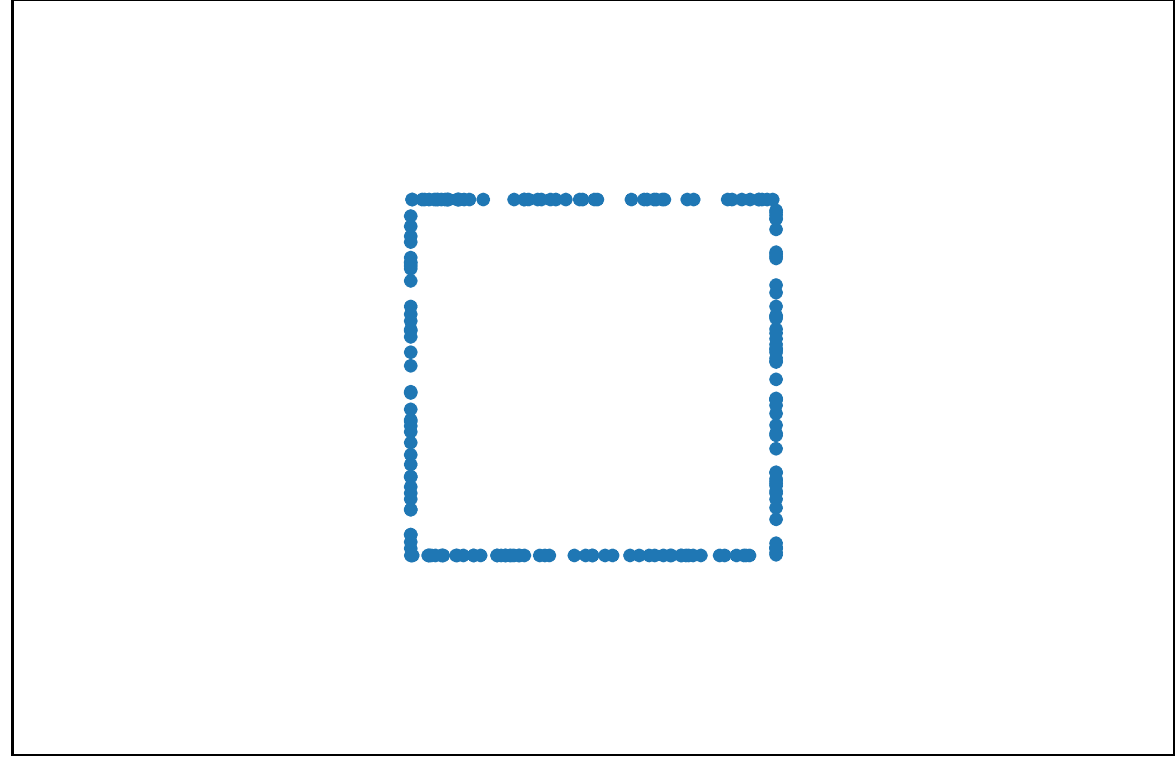}
 }
 ~
 \subfloat{
 \label{fig:square_vanilla}
%  \centering
 \includegraphics[trim=50 50 50 50, clip, width=0.28\textwidth]{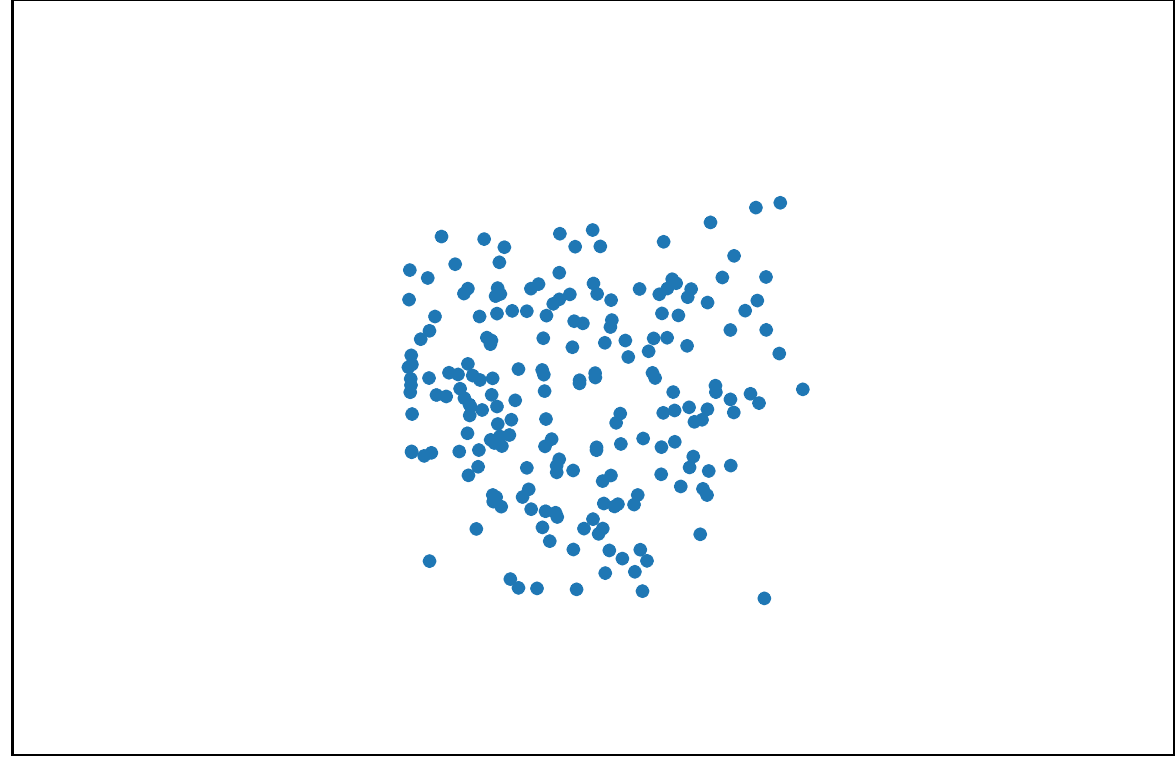}
 }
 ~
 \subfloat{
 \label{fig:square_ours}
%  \centering
 \includegraphics[trim=50 50 50 50, clip, width=0.28\textwidth]{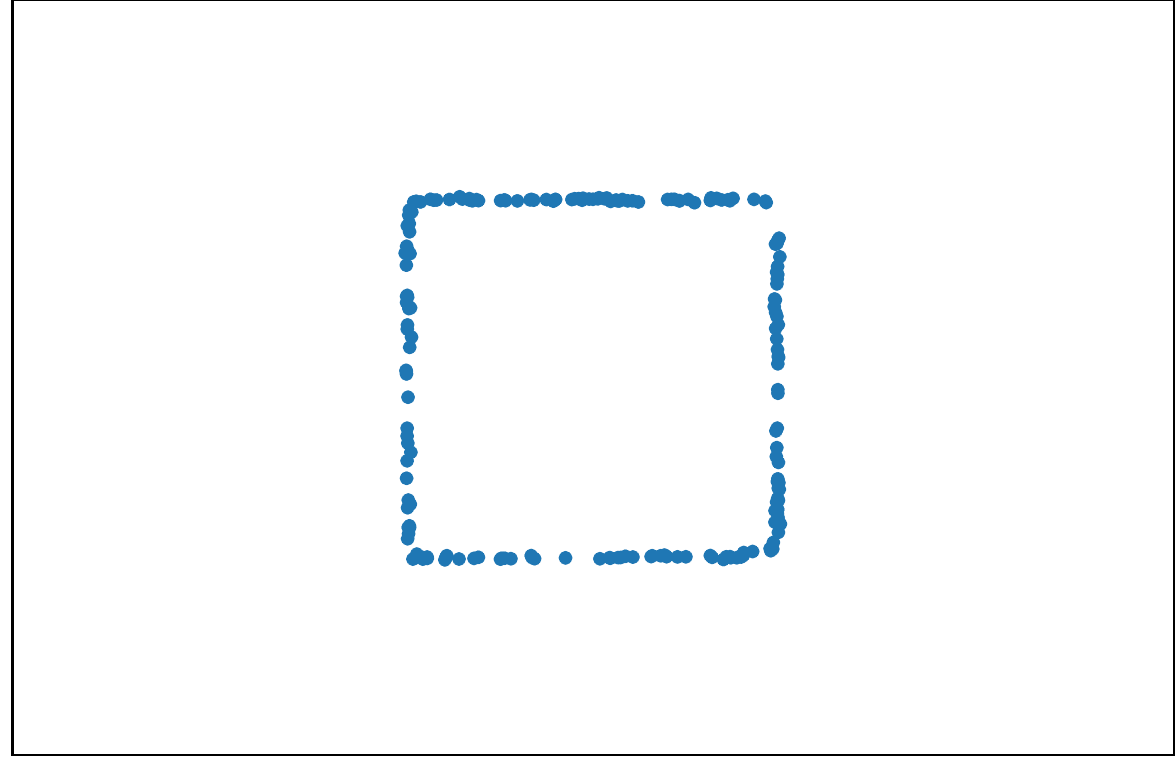}
 }
\hfill
%  \vspace{-10pt}
  \subfloat{
 \label{fig:star_real}
%  \centering
 \includegraphics[trim=15 40 15 20, clip, width=0.28\textwidth]{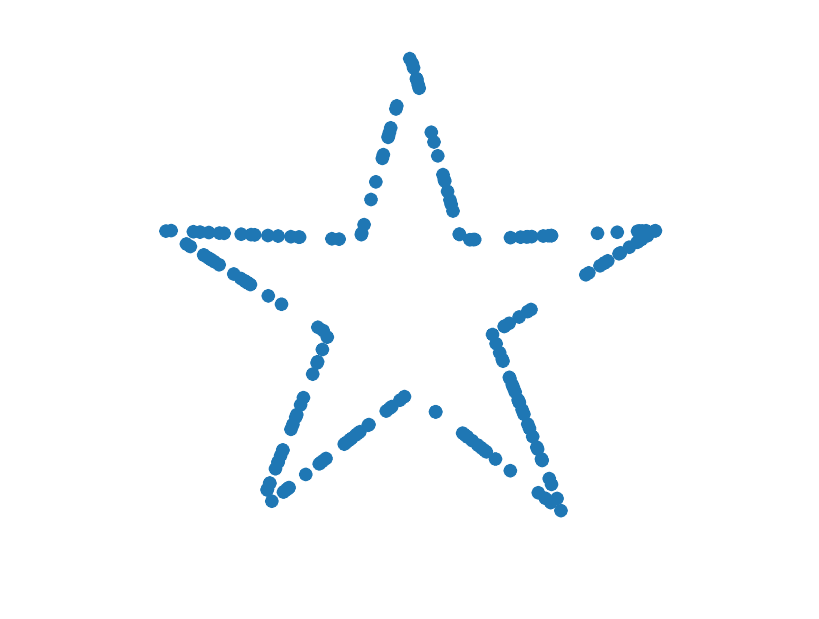}
 }
 ~
 \subfloat{
 \label{fig:star_vanilla}
%  \centering
 \includegraphics[trim=15 40 15 20, clip, width=0.28\textwidth]{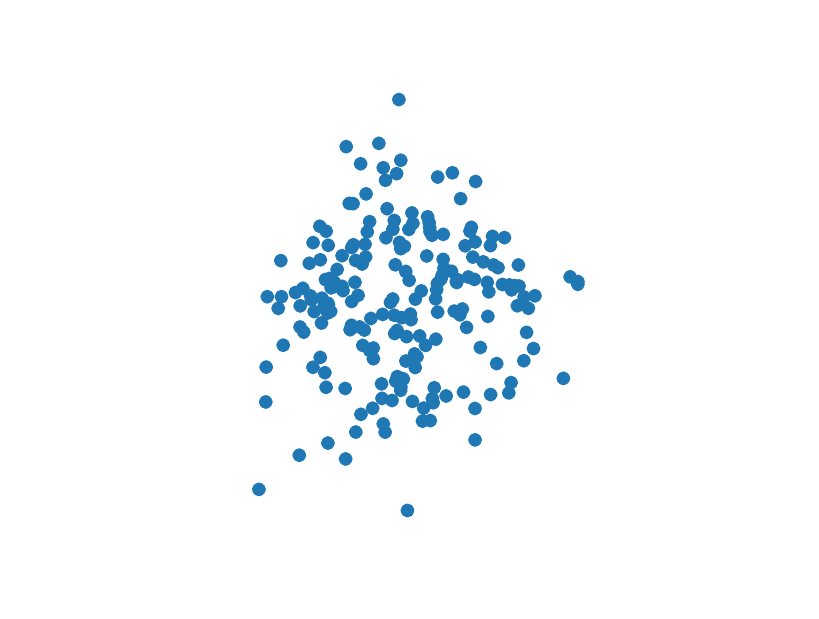}
 }
 ~
 \subfloat{
 \label{fig:star_ours}
%  \centering
 \includegraphics[trim=15 40 15 20, clip, width=0.28\textwidth]{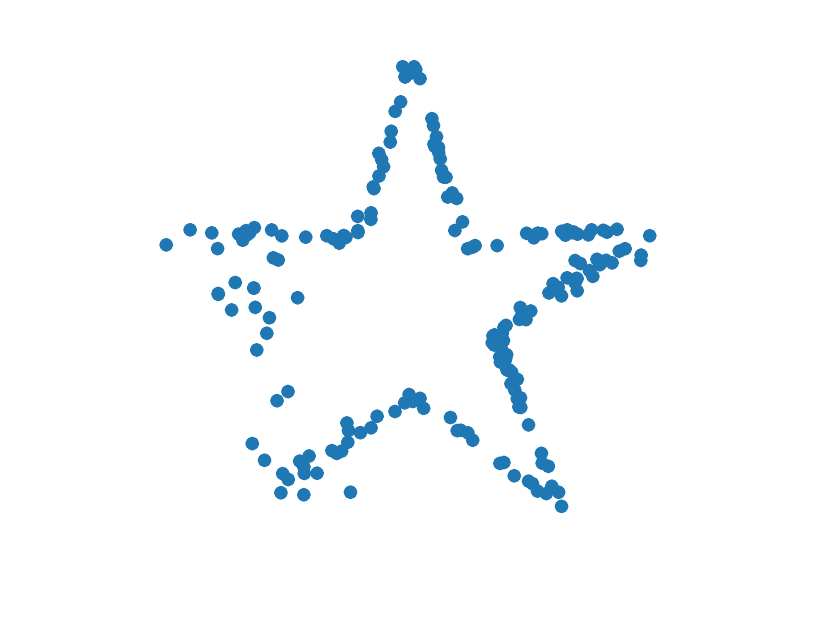}
 }
 \hfill
%  \vspace{-10pt}
  \subfloat{
 \label{fig:infinity_real}
%  \centering
 \includegraphics[trim=15 40 15 40, clip, width=0.28\textwidth]{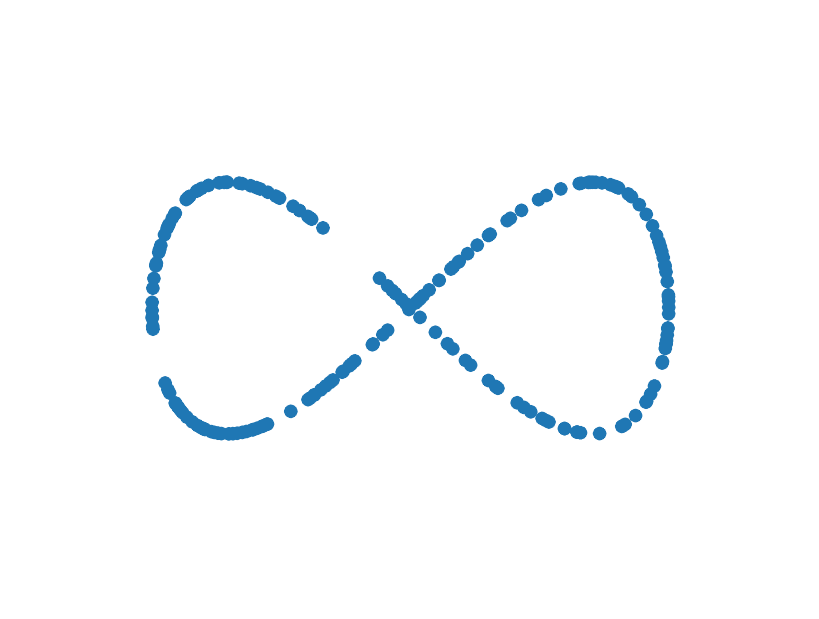}
 }
 ~
 \subfloat{
 \label{fig:infinity_vanilla}
%  \centering
 \includegraphics[trim=15 40 15 40, clip, width=0.28\textwidth]{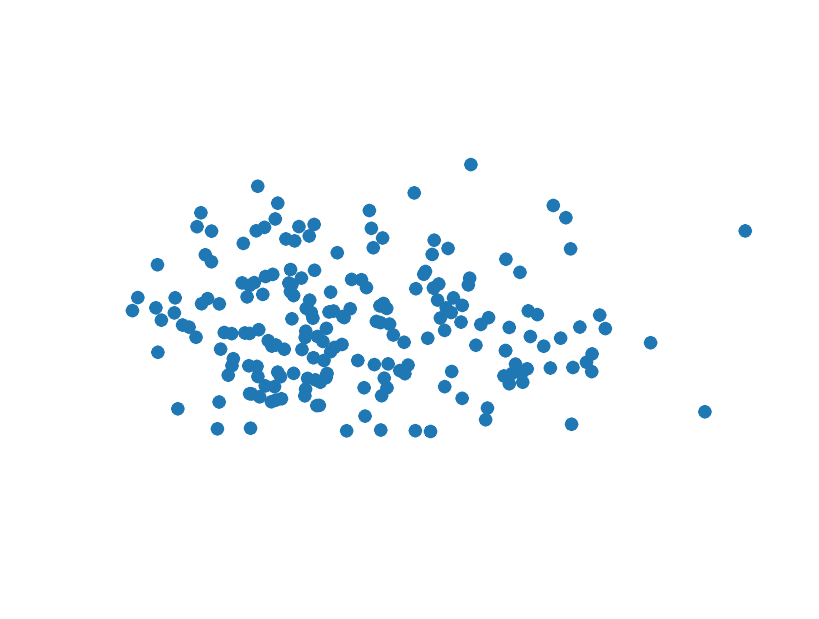}
 }
 ~
 \subfloat{
 \label{fig:infinity_ours}
%  \centering
 \includegraphics[trim=15 40 15 40, clip, width=0.28\textwidth]{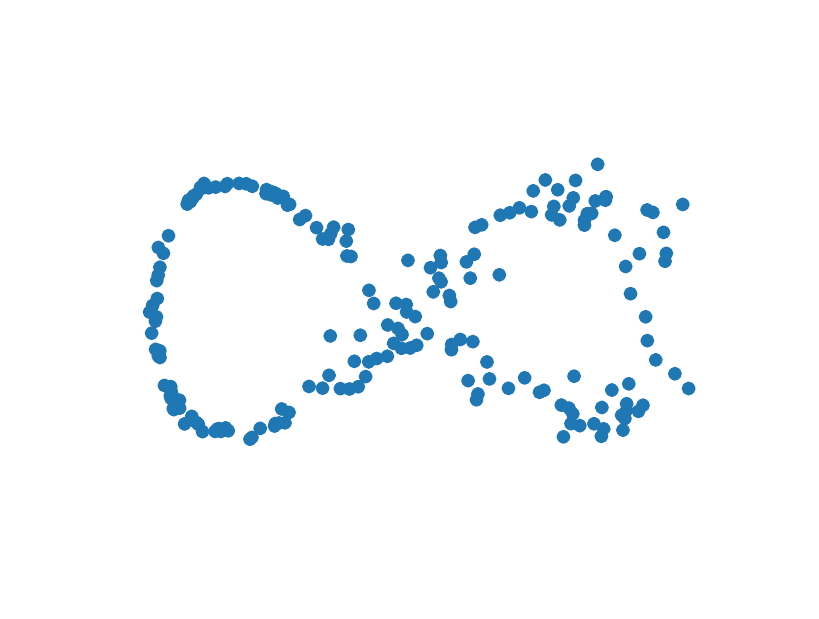}
 }
 \hfill
 \setcounter{subfigure}{0}
%  \vspace{-10pt}
 \subfloat[Data]{
 \label{fig:cluster_real}
%  \centering
 \includegraphics[trim=15 25 15 50, clip, width=0.26\textwidth]{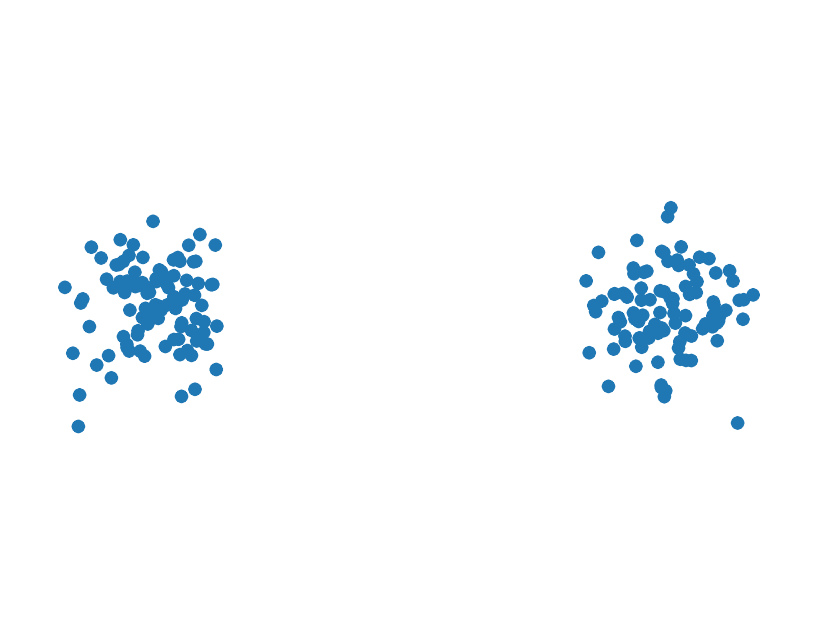}
 }
 ~
 \subfloat[VAE]{
 \label{fig:cluster_vanilla}
%  \centering
 \includegraphics[trim=15 25 15 50, clip, width=0.26\textwidth]{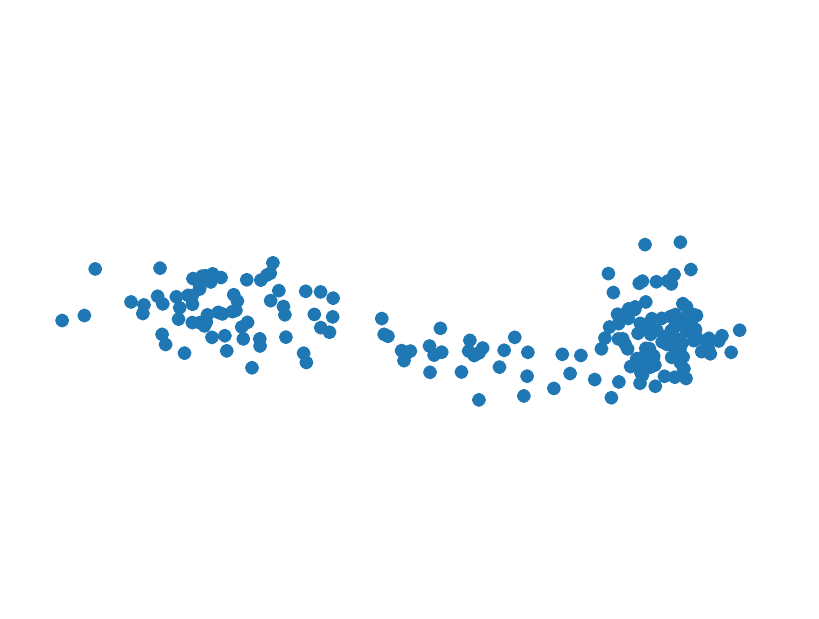}
 }
 ~
 \subfloat[\method]{
 \label{fig:cluster_ours}
%  \centering
 \includegraphics[trim=15 25 15 50, clip, width=0.26\textwidth]{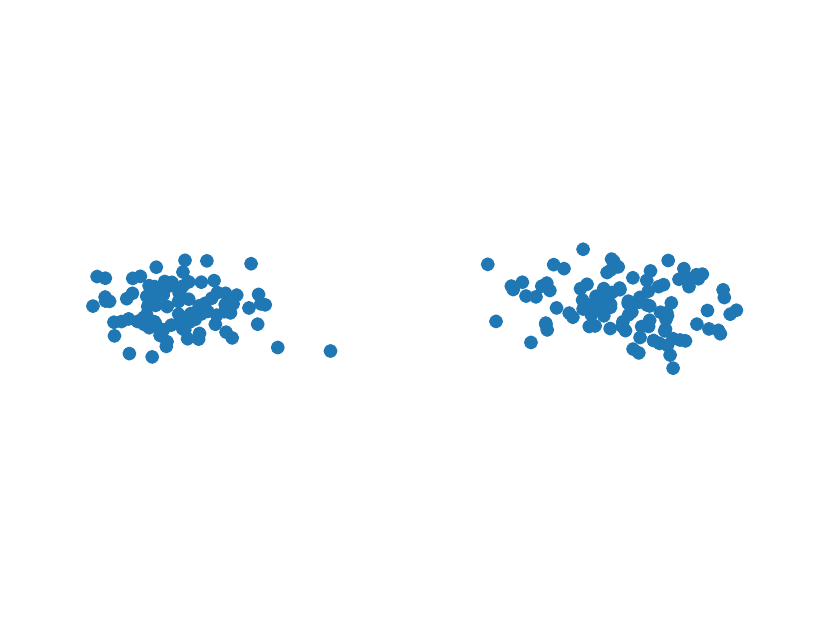}
 } 
% \hfill
 \vspace*{-1.3ex}
 \caption{
 Training data and samples from learned generative models of vanilla-\gls{VAE} and \method for multiply-connected and clustered distributions.
 \method uses [Rows 1,2] circular prior with one hole, [Row 3] multiply-connected prior with two holes, and [Row 4] clustered prior.
 Vamp-\gls{VAE} behaves similarly to a vanilla \gls{VAE}; its results are presented in~\cref{fig:examples_low_dim}.
  \vspace{-1.0em}
 }
 \label{fig:examples_low_dim}
\end{wrapfigure}
Data is often most naturally described on non-Euclidean spaces such as circles, e.g.\ wind directions~\citep{mardia2000Directional}, and other multiply-connected shapes, e.g.\ holes in disease databases~\citep{liu1997discovering}. 
For reasons previously explained in \cref{sec:background}, standard \glspl{VAE} cannot practically model such topologies, which prevents them from learning generative models which match even the simplest data distributions with non-trivial topological structures, as shown in \cref{fig:cluster_vanilla}.

Luckily, by designing $g_{\psi}$ to map the Gaussian prior to a simple representative distribution in a topological class, we can easily equip \methods with the knowledge to approximate any data distributions with similar topological properties.
Specifically, by defining $g_{\psi}$ as the orthogonal projection to $\mathbb{S}^1$, $g_{\psi}(z)=z/(||z||_2+\epsilon)$, we map the Gaussian prior approximately to a uniform distribution to $\mathbb{S}^1$, where $\epsilon$ is a small positive constant to ensure the continuity of $g_{\psi}$ near the origin. From Rows 1 and 2 of \cref{fig:examples_low_dim}, we find that this inductive bias gives \methods the ability to learn various distributions with a hole.
We can add further holes by simply `gluing' point pairs.
%, as described in \cref{sec:mapping_multiply_connected_appendix}; 
For example, for two holes we can use
\begin{align}
    g_2(y)&=\text{Concat}\left(g_{1}(y)_{[:,1]},~ g_{1}(y)_{[:,2]}\sqrt{(4/3-(1-|g_{1}(y)_{[:,1]}|)^2)}-1/\sqrt{3}\right),
\end{align}
which first map $y$ to approximately $S^1$, and then glues $(0, 1)$ and $(0, -1)$ together to create new holes (see \cref{fig:glue} for an illustration).
%In the glue function $z_1$ and $z_2$ are the first and second dimensions of the input.
Furthermore, we can continue to glue points together to achieve a higher number of holes $h$, and thus more complex connectivity.
Row 3 of \cref{fig:examples_low_dim} gives an example of learning an infinity sign by introducing a `two-hole' inductive bias.

Compared with vanilla-VAE and Vamp-VAE, which try to find a convex hull for real data distributions, \methods can deal with distributions with highly non-convex and very non-smooth supports~(see \cref{fig:examples_low_dim} and \cref{sec:mapping_multiply_connected_appendix}).
We emphasize here that our inductive bias does not contain the information about the precise shape of the data, only the number of holes.
We thus see that \methods can provide substantial improvements in performance by incorporating only basic prior information about the topological properties of the data, which point out a way to approximate distributions on more complex structures, such as linear groups~\citep{gupta2018topology}.

% \subsection{Clustering}
\subsection{Multi--Modality}
\label{sec:clustered}

Many real-world datasets exhibit multi-modality.
For example, data with distinct classes are often naturally clustered into (nearly) disconnected components representing each class.
However, vanilla \glspl{VAE} generally fail to fit multi-modal data due to the topological issues explained in \cref{sec:background}.
Previous work~\citep{johnson2017Composing,mathieu2019disentangling} has thus proposed the use of a multi-modal prior, such as a mixture of Gaussian (MoG) distribution, so as to capture all components of the data.
Nonetheless, \glspl{VAE} with such priors often still struggle to model multi-modal data because of mismatch between $q_{\phi}(z)$ and $p(z)$ or training instability issues.

%For ICLR
\begin{wrapfigure}[14]{R}{0.45\textwidth}
 \vspace{-28pt}
 \centering
 \subfloat[]{
 \label{fig:cluster_dist_a_large}
 \centering
 \includegraphics[width=0.22\textwidth]{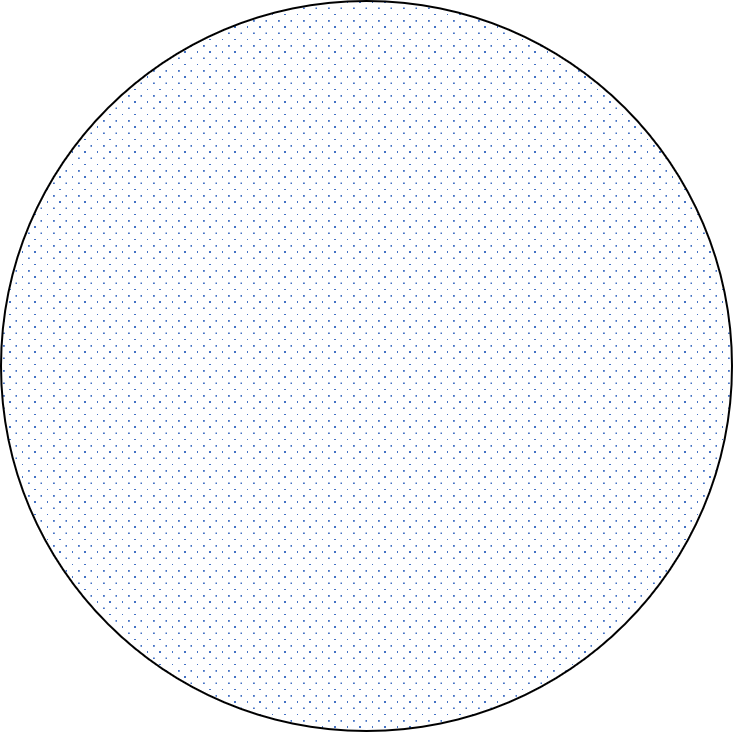}
 }
 \hspace{6pt}
 \subfloat[]{
 \label{fig:cluster_dist_b_large}
 \centering
 \includegraphics[width=0.22\textwidth]{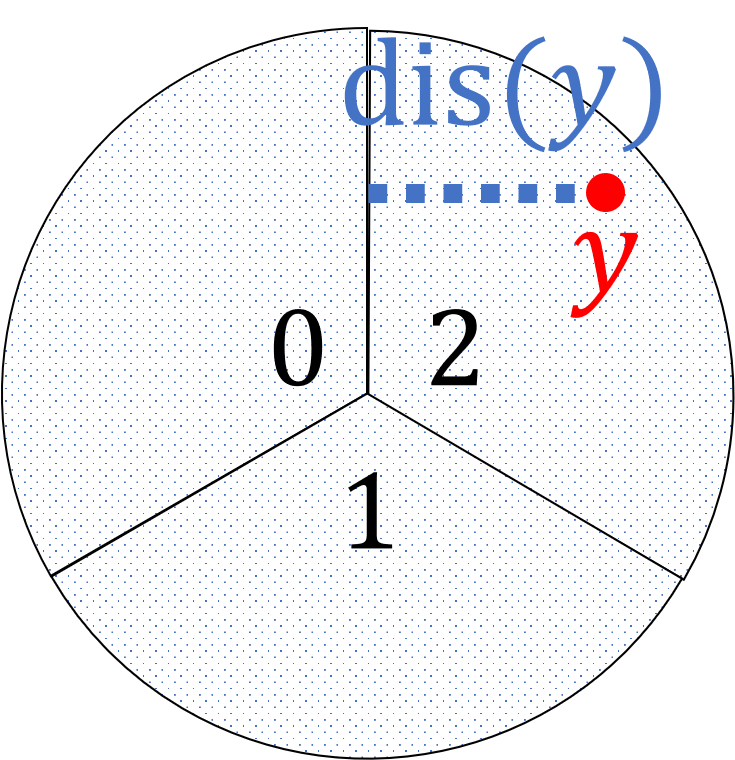}
 }
 \hspace{6pt}
 \subfloat[]{
 \label{fig:cluster_dist_c_large}
 \centering
 \includegraphics[width=0.31\textwidth]{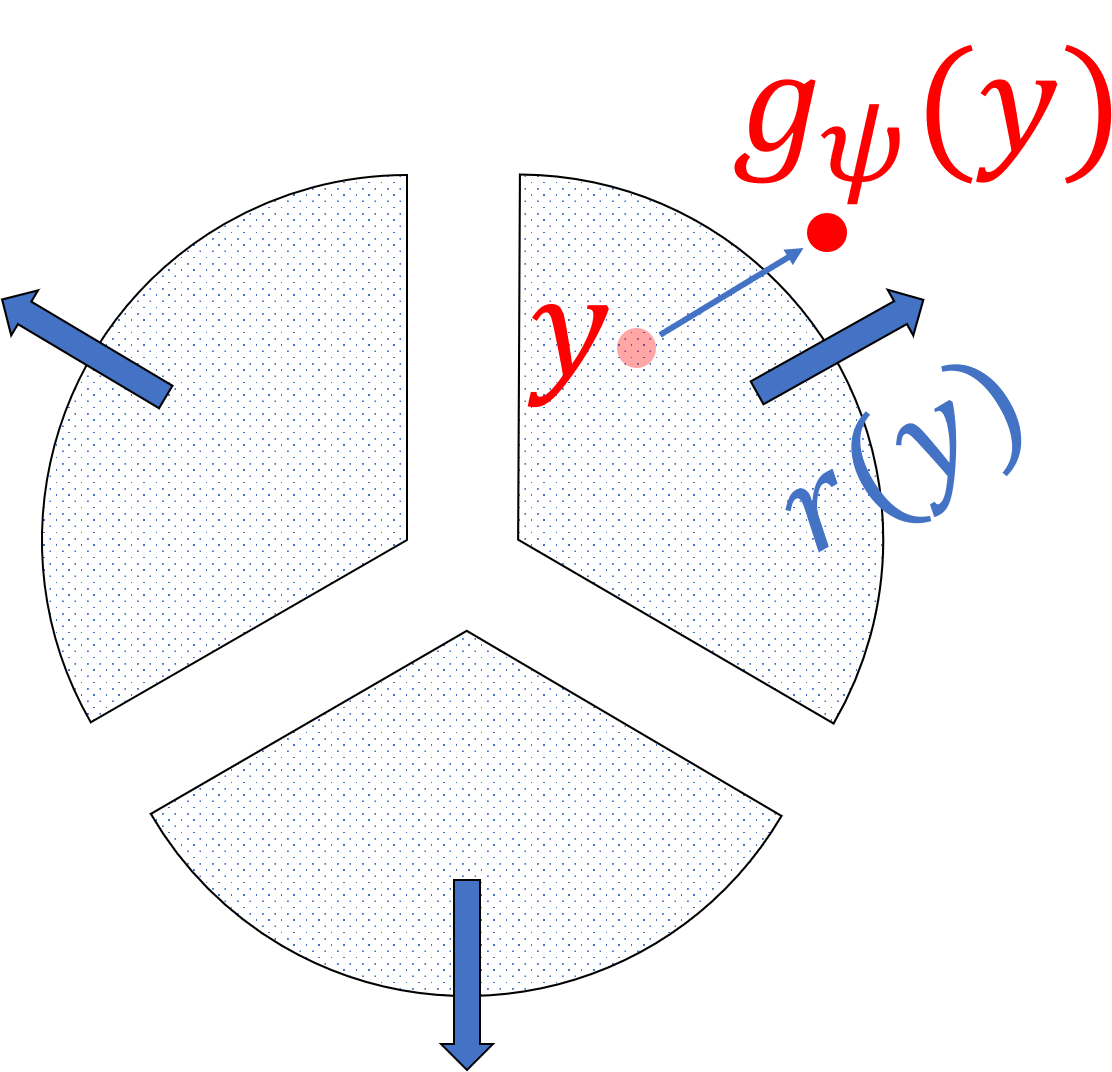}
 }
 \caption{\small{Illustration of clustered mapping where $K=3$. The circle represents a density isoline of a Gaussian.  Note that not all points in the sector are moved equally: points close to the boundaries between sectors are moved less, with points on the boundary themselves not moved at all.}}
 
 \label{fig:cluster_mapping_large}
\end{wrapfigure}
We tackle this problem by using a mapping $g_\psi$ which contains a clustering inductive bias.
The high-level idea is to design a mapping $g_\psi$ with a localized high Lipschitz constant that `splits' the continuous Gaussian distribution into $K$ disconnected parts and then pushes them away from each other. In particular, we split $\Y$ it into $K$ equally sized sectors using its first two dimensions (noting it is not needed to split on all dimensions to form clusters), as shown in~\cref{fig:cluster_mapping_large}.
For any point $y$, we can easily get the center direction $\text{r}(y)$ of the sector that $y$ belongs to and the distance $\text{dis}(y)$ between $y$ and the sector boundary. 
Then we define $g_\psi(y)$ as:
\begin{align}
    \label{eqn:cluster_large}
    g_\psi(y) = y + {c_1}\text{dis}(y)^{c_2}\text{r}(y),
\end{align}
where $c_1$ and $c_2$ are empirical constants.
We can see that although $g_\psi$ has very different function on different sectors, it is still continuous on the whole plane with $g_\psi(y)=y$ on sector boundaries, which is desirable for gradient-based training. 
See \cref{sec:clustered_appendix} for more details.

To assess the performance of our approach, we first consider a simple 2-component MoG synthetic dataset in the last row of \cref{fig:examples_low_dim}.
We see that the vanilla VAE fails to learn a clustered distribution that fits the data, while the \method sorts this issue and fits the data well. 

\begin{wraptable}[11]{r}{0.45\textwidth}
 \centering
  \vspace{-12pt}
   {\small
    \begin{tabular}{ lc } 
         \toprule
        %  Method & FID Score~($\downarrow$)  \\ \midrule
        %  \gls{VAE} &${42.0}_{\pm 1.1}$\\
        %  GM-VAE&${41.0}_{\pm 4.7}$\\
        %  MoG-VAE&${41.2}_{\pm 3.3}$\\
        %  Vamp-VAE&${38.8}_{\pm 2.4}$\\
        %  VAE with Sylvester NF &${35.0}_{\pm 0.9}$\\
        %  \method&\textbf{${32.2}_{\pm1.5}$} \\
          Method & FID Score~($\downarrow$)  \\ \midrule
         \gls{VAE} &$42.0 \pm 1.1$\\
         GM-VAE&$41.0 \pm 4.7$\\
         MoG-VAE&$41.2 \pm 3.3$\\
         Vamp-VAE&$38.8 \pm 2.4$\\
         VAE with Sylvester NF &$35.0 \pm 0.9$\\
         \method&\textbf{$32.2 \pm1.5$} \\
         \bottomrule
    \end{tabular}}
    \vspace*{-1ex}
 \vspace*{-0.5ex}
 \caption{Generation quality on MNIST.
 Shown is mean FID score (lower better) $\pm$ standard deviation over 10 runs.
 \label{tab:mnist}}
\end{wraptable}
To provide a more real-world example, we train an \method and a variety of baselines on the \textbf{MNIST} dataset, comparing the generation quality of the learned models using the FID score~
\citep{heusel2017gans} in \cref{tab:mnist}.
%We compare \method with baselines including standard \gls{VAE}, \glspl{VAE} with non-Gaussian priors~(GMVAE and MoGVAE), Vamp-VAE as well as Sylvester Normalizing Flows~\citep{}, which adds a normalizing flow to the variational distribution. 
We find that the GM-VAE~\citep{dilokthanakul2016deep} and MoG-VAE~(VAE with a fixed MoG prior) achieve performance gains by using non-Gaussian priors. 
The Vamp-VAE~\citep{tomczak2018vae} and a \gls{VAE} with a Sylvester Normalizing Flow~\citep{berg2018sylvester} encoder provide further gains by making the prior and encoder distributions more flexible respectively. However, the \method comfortably outperforms all of them.
%None of these methods beat \method.

\begin{wrapfigure}[7]{r}{0.45\textwidth}
\vspace{-1em}
 \centering
 \subfloat[VAE]{
 \label{fig:mnist01_vanilla}
 \centering
 \includegraphics[trim=10 10 10 10, clip, width=0.28\textwidth]{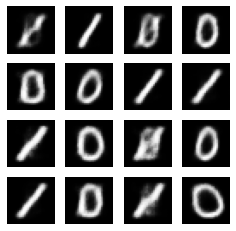}
 }
 \subfloat[MoG-VAE]{
 \label{fig:mnist01_MoG}
 \centering
 \includegraphics[trim=30 30 30 30, clip, width=0.28\textwidth]{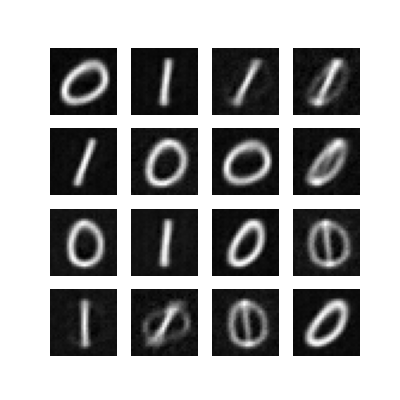}
 }%\\[-1ex]
 %
 %\subfloat[Vamp-VAE]{
 %\label{fig:mnist01_vamp}
 %\centering
 %\includegraphics[trim=30 30 30 30, clip, width=0.4\textwidth]{figure/mnist01_vamp.png}
 %}~~
 \subfloat[\method]{
 \label{fig:mnist01_ours}
 \centering
 \includegraphics[trim=10 10 10 10, clip, width=0.28\textwidth]{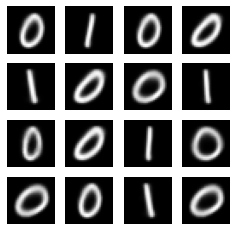}
 }
 \vspace*{-3.0ex}
 \caption{
 Generated samples for \textbf{MNIST-01}.
%  Generated samples on \textbf{MNIST-01}.
 \vspace{-1.5em}
 }
 \label{fig:mnist-01}
\end{wrapfigure}
To gain insight into how \methods achieve superior generation quality, we perform analysis on a simplified setting where
we select only the  `0' and `1' digits from the \textbf{MNIST} dataset to form a strongly clustered dataset, \textbf{MNIST-01}.
We further decrease the latent dimension to $1$ to make the problem more challenging. 
%On top of a vanilla VAE we benchmark \method against several \glspl{VAE} introduced to tackle multi-modal data---namely MoG-\gls{VAE}, which directly changes the prior to a MoG distribution, and Vamp-\gls{VAE}~\citep{tomczak2018vae}. We also tried GMVAE~\citep{dilokthanakul2016deep}, but encountered severe training instabilities, such that we were not able to generate any reliable results.
%Yet, we do not report results on GMVAE on account of its instability affecting usage.
%Yet, we eventually did not report any results from GMVAE as it is too unstable to give acceptable results.
Fig.~\ref{fig:mnist-01} shows that here the vanilla \gls{VAE} generates some samples which look like interpolations between '0' and '1', meaning that it still tries to learn a connected distribution containing '0' and '1'. 
Further, the general generation quality is poor, with blurred images and a lack of diversity in generated samples (e.g.~all the `1's have the same slant).
Despite using a clustered prior, the MoG-VAE still produces unwanted interpolations between the classes. By contrast, \method generates digits that are unambiguous and crisper.

\begin{table}
\centering

\caption{
Quantitative results on \textbf{MNIST-01}. \textbf{Uncertainty} is the proportion of images whose labels are `indistinguishable' by the pre-trained classifier, defined as having prediction confidence $<80\%$. \textbf{`1' proportion} is the proportion of images classified as `1'. 
%We can see that all methods generate more '0' than `1'.
\vspace*{-3ex}
}
\label{tab:mnist01}
{\scriptsize
% \footnotesize
\begin{tabular}{ lccccccc } 
 \toprule
 Method & Data & \gls{VAE} & GM-\gls{VAE}& MoG-\gls{VAE} & Vamp-\gls{VAE}& Flow &\method \\ \midrule
 \textbf{Uncertainty}(\%) &0.2 $\pm$ 0.1 &2.5 $\pm$ 0.4& 3.5 $\pm$ 1.8& 4.5 $\pm$ 0.8&2.4 $\pm$ 0.3& 16.2 $\pm$ 2.1 &\textbf{0.9 $\pm$ 0.8}\\
 \textbf{`1' proportion}(\%) &50.0 $\pm$ 0.2&48.8 $\pm$ 0.2&48.1 $\pm$ 0.3 &47.7 $\pm$ 0.4 &48.8 $\pm$ 0.1& 42.5 $\pm$ 1.0 &\textbf{49.5 $\pm$ 0.4}\\
 %
%   \textbf{Uncertainty}(\%) &${0.2}_{\pm 0.1}$ &${2.5}_{\pm 0.4}$& ${3.5}_{\pm 1.8}$& ${4.5}_{\pm 0.8}$&${2.4}_{\pm 0.3}$& ${16.2}_{\pm 2.1}$ &\textbf{${0.9}_{\pm 0.8}$}\\
%  %
%  \textbf{`1' proportion}(\%) &${50.0}_{\pm 0.2}$&${48.8}_{\pm 0.2}$&${48.1}_{\pm 0.3}$ &${47.7}_{\pm 0.4}$ &${48.8}_{\pm 0.1}$& ${42.5}_{\pm 1.0}$ &\textbf{${49.5}_{\pm 0.4}$}\\
 \bottomrule
\end{tabular}}
\vspace*{-1ex}
\end{table}

\begin{wraptable}[10]{r}{0.4\textwidth}
\centering
\vspace{-14pt}
\caption{
Learned proportions of `0's on \textbf{MNIST-01} for different ground truths.
Error bars are std.~dev.~from 10 runs.
%We find \method can accurately learn the data proportions and turning off a model cluster for real data with only one cluster.
}
\label{tab:mnist01_proportion}

{\small
\begin{tabular}{ lc } 
 \toprule
 True Prop.\ & Learned Prop.\ %& LNoC \\ 
 \\ \midrule
  0.5             & 0.47 $\pm$ 0.01 \\%& 2 \\
  0.4              & 0.36 $\pm$ 0.10  \\%& 2 \\
  0.25              & 0.25 $\pm$ 0.08 \\%& 2 \\
  0.2              & 0.16 $\pm$ 0.11 \\%& 2 \\
  0              & 0.02 $\pm$ 0.01 \\%& 1 \\
 \bottomrule
\end{tabular}}
\vspace*{-1ex}
\end{wraptable}
To quantify these results, we further train a logistic classifier on \textbf{MNIST-01} and use it to classify images generated by each method. For each method, we calculate the proportion of samples produced by the generative model that are assigned to each class by this pre-trained classifier, as well as the proportion of samples for which the classifier is uncertain.
%(defined as a predicted probability lower than 0.2).
% We report this uncertainty as the \emph{uncertainty} in \cref{tab:mnist01}, along with the proportion of `0' and `1' digits.
From \cref{tab:mnist01} we see that \method significantly outperforms its competitors in the ability to generate balanced and unambiguous digits. 
To extend this example further, and show the ability of \methods to learn aspects of $g_{\psi}$ during training, we further consider parameterizing and then learning the relative size of the clusters.
\cref{tab:mnist01_proportion} shows that this can be successfully learned by \methods on \textbf{MNIST-01}.

%AE~(Fig.~\ref{fig:mnist01_vamp}) and \method ~(Fig.~\ref{fig:mnist01_ours}) generate less blurry images, but we find that samples from \method are more diverse.
% In order to quantitatively compare each method, we train a logistic classifier to distinguish figures of `0' and `1'.
% Then we count the samples on which the classifier is less confident~($0.2 \le$ probability($label=1$) $\le 0.8$) from each method.
% In order to verify whether each method generates equal proportions of `0' and `1', we also calculate the ratio of generated figures which are identified as `0' and `1' by the classifier.
% From Tab.~\ref{tab:mnist01}, we find that \method outperforms baseline models in both metrics, which indicates that \method generate less blurry images and learn more accurate cluster proportions.

\vspace{-4pt}
\subsection{Sparsity}
\label{sec:sparse}
\vspace{-4pt}

Sparse features are often well-suited to data efficiency on downstream tasks~\citep{huang2006sparse}, in addition to being naturally easier to visualize and manipulate than dense features~\citep{ng2011sparse}. 
However, existing \gls{VAE} models for sparse representations trade off generation quality to achieve this sparsity~\citep{mathieu2019disentangling, tonolini2020variational, barello2018SparseCoding}.
Here, we show that \methods can instead \emph{simultaneously} increase feature sparsity and generation quality.
Moreover, they are able to achieve state-of-the-art scores on sparsity metrics.

Compared with our previous examples, the $g_{\psi}$ here needs to be more flexible so that it can learn to map points in a data-specific way and induce sparsity without unduly harming reconstruction.
%, because there are various sparse distributions, with very different structures.
%Luckily, \methods can automatically select the sparse distributions that best fit underlying data distributions.
%Concretely, 
To achieve this, we use the simple form for the mapping:
$g_\psi(y)= y \odot ~\text{DS}_\psi(y)$, 
where $\odot$ is pointwise multiplication, and DS is a `dimension selector' network that selects dimensions to deactivate given $y$.
DS outputs values between $[0, 1]$ for each dimension, with $0$ being fully deactivated and $1$ fully activated; the more dimensions we deactivate, the sparser the representation.
By learning DS during training, this setup allows us to learn a sparse representation in a data-driven manner.
% We add the parameters of DS to $\psi_l$ to tune them during \gls{VAE} training.
%
To control the degree of sparsity, we add a sparsity regularizer, $\ELBO_{sp}$, to the ELBO with weighting parameter $\gamma$ (higher $\gamma$ corresponds to more sparsity). Namely, we optimize $\ELBO_{\Y}(\theta,\phi,\psi)+\gamma \ELBO_{sp}(\phi,\psi)$, where
\begin{align}
\label{eq:sparsity_reg}
   \ELBO_{sp}(\phi,\psi) := \mathbb{E} \left[\frac{1}{M}\sum_{i=1}^M (H\left(DS(y_i))\right) - H\left(\frac{1}{M}\sum_{i=1}^M DS(y_i)\right)\right],
\end{align}
$H(v) = -\sum_i \left(v_i/\|v\|_1\right) \log \left(v_i/\|v\|_1\right)$ is the normalized entropy of an positive vector $v$, and the expectation is over drawing a minibatch of samples $x_1,\dots,x_M$ and then sampling each corresponding $y_i\sim q_{\phi}(\cdot|x=x_i)$.
$\ELBO_{sp}$ encourages DS to deactivate more dimensions, while also encouraging diversity in which dimensions are activated for different data points, improving utilization of the latent space. Please see \cref{sec:sparse_appendix} for more details and intuitions.
%$\ELBO_{sp}$ can be viewed as a regularizer penalizing features with more activated dimensions.
%The two losses are then weighted with a hyperparameter $\gamma$ to yield $\ELBO_{\Y}(\theta,\phi,\psi,\gamma) = \ELBO_{\Y}(\theta,\phi,\psi) + \gamma ~\ELBO_{sp}$.
%Due to space constraints we leave the design of $\ELBO_{sp}$ and related ablation studies to 
%See \cref{sec:sparse_appendix} for full details and additional ablations.
% indicating that we can easily change certain features of the figure, while keeping others unchanged. 
Initial qualitative results are shown in~\cref{fig:fashion_mnist_manipulation}, where we see that our \method is able to learn sparse and intuitive representations.

%For ICLR rebuttal only
\begin{figure}[t]
 %\vspace{-1.5em}
 \centering
 \subfloat{
 \centering
 %\raisebox{0.4cm}{\rotatebox{90}{(b)~\textbf{Fashion-MNIST}}}
 \includegraphics[width=0.4\textwidth]{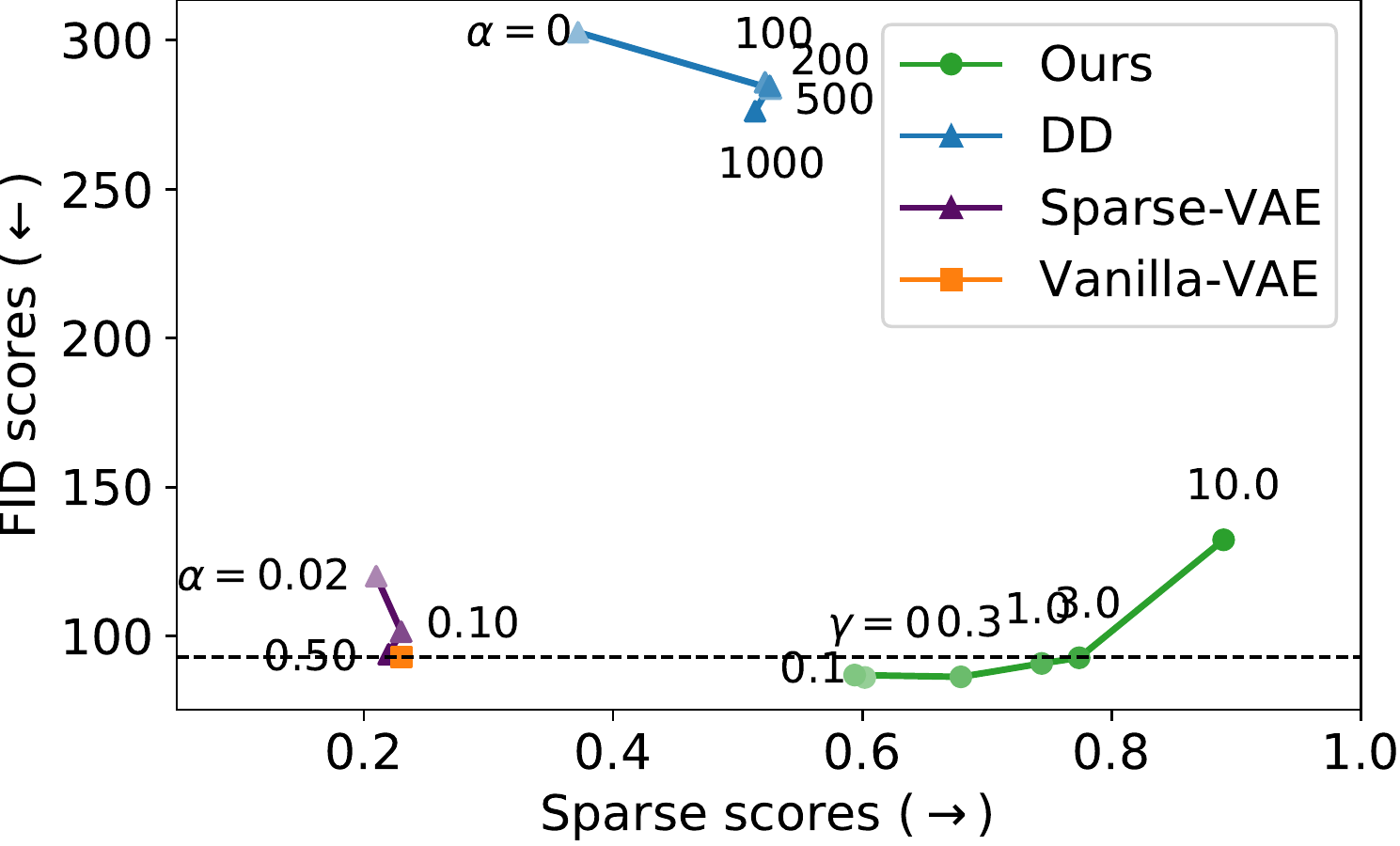}
 }
 \qquad
 \subfloat{
 \centering
 %\raisebox{0.5cm}{\rotatebox{90}{(b)~\textbf{Fashion-MNIST}}}
 \includegraphics[width=0.39\textwidth]{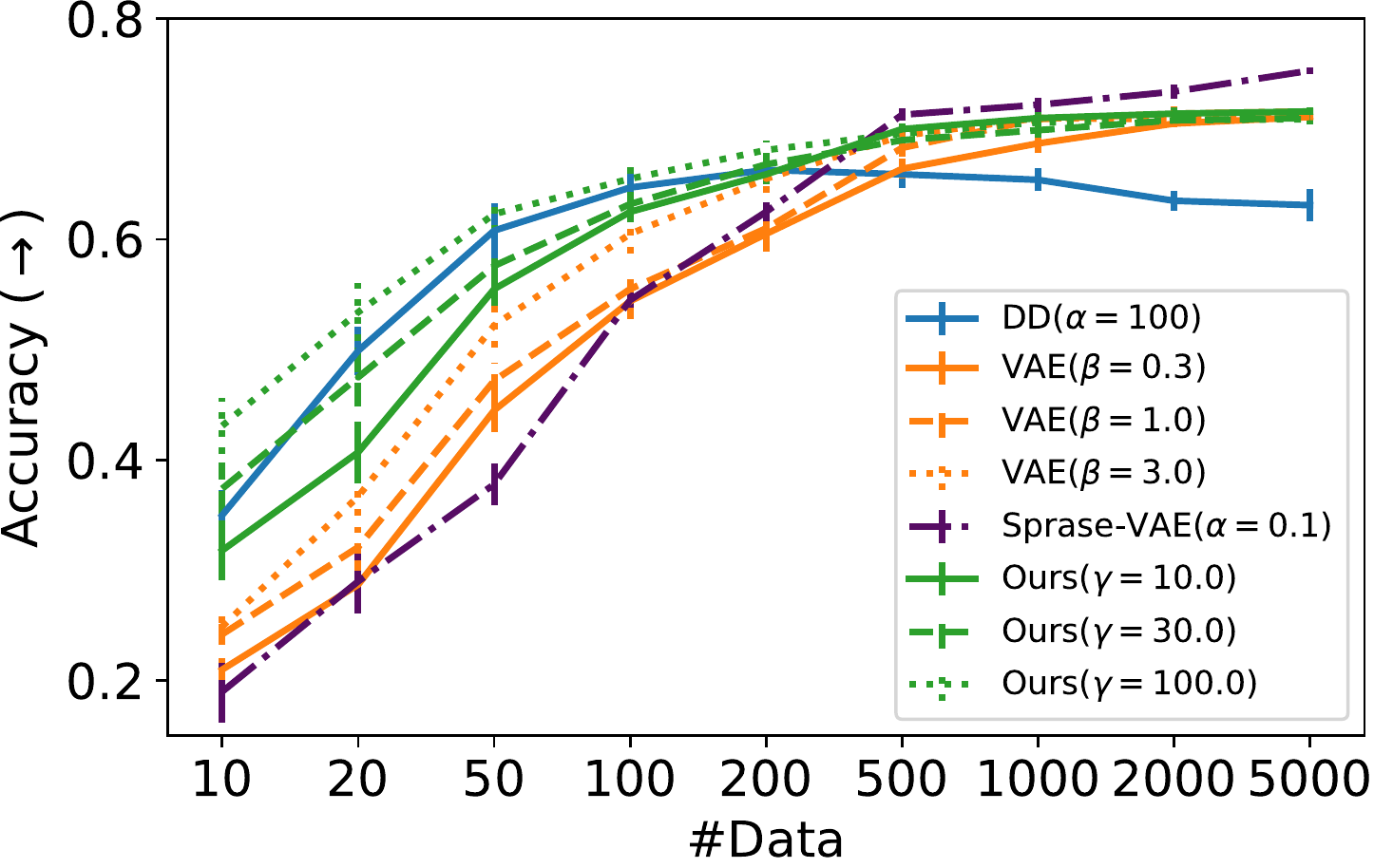}
 }
 \vspace{-6pt} 
 \caption{Results on \textbf{Fashion-MNIST}. The \textit{left} figure shows FID and sparsity scores. Lower FID scores~($\downarrow$) represent better sample quality while higher sparse scores~($\rightarrow$) indicate sparser features.  The \textit{right} figure shows the performance of sparse features from \method on downstream classification tasks.
 See~\cref{sec:sparse_appendix} for details and results for \textbf{MNIST}.
 \vspace{-20pt}}
 \label{fig:fid_sparse_fashion_mnist}
\end{figure}

\begin{wrapfigure}[24]{r}{0.5\textwidth}
 %\vspace{-8pt}
 \centering
 \sidesubfloat{
 \label{fig:fashion_mnist_0}
 \centering
 \includegraphics[width=0.9\textwidth]{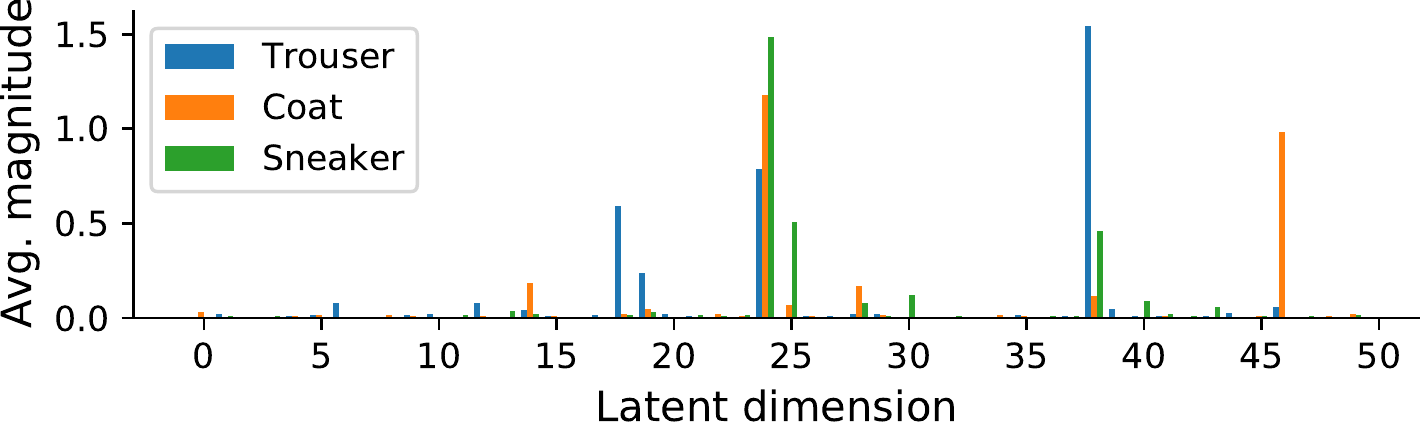}
 }
 \vspace{5pt}
 \setcounter{subfigure}{0}
%  \vspace{10pt}
 \sidesubfloat[]{
 \label{fig:fashion_mnist_1}
 \centering
 \includegraphics[width=0.8\textwidth]{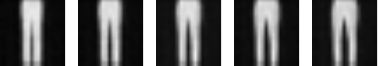}
 }
 \vspace{5pt}
 \sidesubfloat[]{
 \label{fig:fashion_mnist_2}
 \centering
 \includegraphics[width=0.8\textwidth]{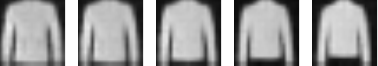}
 }
 \vspace{5pt}
 \sidesubfloat[]{
 \label{fig:fashion_mnist_3}
 \centering
 \includegraphics[width=0.8\textwidth]{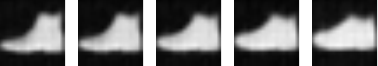}
 }
 \caption{
 Qualitative evaluation of sparsity. [Top] Average magnitude of each latent dimension for three example classes in \textbf{Fashion-MNIST}; less than $10\%$ dimensions are activated for each class. [Bottom] Activated dimensions are different between classes:
 (a-c) show the results of separately manipulating an activated dimension for each class. (a) Trouser separation  (Dim 18). (b) Coat length (Dim 46). (c) Shoe style (formal/sport, Dim 25).
%  \vspace{-0.5em}
 }
 \label{fig:fashion_mnist_manipulation}
\end{wrapfigure}
To quantitatively assess the ability of our approach to yield sparse representations and good quality generations, we compare against vanilla VAEs, the specially customized sparse-VAE of~\citet{tonolini2020variational}, and the sparse version of \cite{mathieu2019disentangling}~(DD) on \textbf{Fashion-MNIST}~\citep{xiao2017/online} and \textbf{MNIST}.
%\footnote{These datasets are common testbeds for sparse encodings, as the data is believed to lie on a union of low dimensional manifolds.}. 
%We also tried directly adding a $l_1$ loss to the latent variables. Because it fails to make the representations significantly sparser, but significantly degrade construction and generation, we omit it from the following discussion.
As shown in Fig.~\ref{fig:fid_sparse_fashion_mnist}~(\textit{left}), we find that \methods increase sparsity of the representations---measured by the Hoyer metric \citep{hurley2009comparing}---while increasing generative sample quality at the same time.
% In particular, \method increase the sparse score~() to more than 0.8 on both data sets.
Indeed, the FID score 
obtained by \method outperforms the vanilla VAE when $\gamma < 3.0$, while the sparsity score substantially increases with $\gamma$, reaching extremely high levels.
%from $0.23$ to about $0.7$. 
By comparison, DD significantly degrades generation quality and only provides a more modest increase in sparsity, while its sparsity also drops if the regularization coefficient is set too high.
The level of sparsity achieved by sparse-VAEs was substantially less than both DD and \methods.
%yet the sparsity decreases again once its regularization coefficient gets too high.
%Moreover, DD and sparse-VAE significantly sacrifices on generation quality---a direct consequence of prior-encoder mismatch illustrated in \cref{fig:encoder_result}.

To further evaluate the quality of the learned features for downstream tasks, we trained a classifier to predict class labels from the latent representations.
%To evaluate the performance of extracted features on downstream tasks, we calculate the prediction accuracy of class labels on both data sets. 
For this, we choose a random forest~\citep{breiman2001random} with maximum depth $4$ as it is well-suited for sparse features. %\emile{is this sentence necessary?}
We vary the size of training data given to the classifier to measure the data efficiency of each model.
% the performance on different levels of data scarcity.
\Cref{fig:fid_sparse_fashion_mnist}~(\textit{right}) shows that \method typically outperforms other the models, especially in few-shot scenarios.
%
%We also perform latent manipulations on \textbf{Fashion-MNIST}, whose results are shown in 

\begin{wraptable}[10]{r}{0.51\textwidth}
 \centering
 \vspace{-10pt}
   {\small
    \begin{tabular}{ lcc } 
         \toprule
         Method & FID~($\downarrow$) &Sparsity~($\uparrow$) \\ \midrule
         \gls{VAE} &68.6$\pm$1.1 &0.22$\pm$0.01\\
         Vamp-VAE&67.5$\pm$1.1 &0.22$\pm$0.01\\
         VAE with Sylvester NF&66.3$\pm$0.4 &0.22$\pm$0.01\\
         Sparse-VAE ($\alpha=0.01$) &328$\pm$10.1 &0.25$\pm$0.01\\
        % Sparse-VAE ($\alpha=0.02$) &339$\pm$12.4 &0.29$\pm$0.01\\
        % Sparse-VAE ($\alpha=0.05$) &337$\pm$9.2 &0.30$\pm$0.01\\
         %Sparse-VAE ($\alpha=0.1$) &332$\pm$3.4 &0.29$\pm$0.01\\
         Sparse-VAE ($\alpha=0.2$) &337$\pm$8.1 &0.28$\pm$0.01\\
         \method ($\gamma=30$) & 64.9$\pm$0.4 &0.25$\pm$0.01\\
        % \method ($\gamma=50$) & 65.8$\pm$0.5 &0.31$\pm$0.02\\
         \method ($\gamma=70$) & 68.0$\pm$0.6 &0.46$\pm$0.02\\
         \bottomrule
    \end{tabular}}
    \vspace*{-1ex}
 \vspace*{-1ex}
 \caption{Generation results on CelebA.
 }
 \label{tab:celeba_sparse}
\end{wraptable}
%We further run \method on \textbf{CelebA}~\citep{liu2015faceattributes} 
Finally, to verify \method 's effectiveness on larger and higher-resolution datasets, we also make comparisons on \textbf{CelebA}~\citep{liu2015faceattributes}.
%Meanwhile the sparse property of \textbf{CelebA} is not as obvious as MNIST and Fashion-MNIST, which makes it a more challenging task.
From \cref{tab:celeba_sparse}, we can see that \method increase sparse scores to 0.46 without sacrificing generation quality. By comparison, the maximal sparse score that sparse-VAE gets is 0.30, with unacceptable sample quality.
Interestingly, \methods with ly low regulation $\gamma$ achieved particularly good generative sample quality, outperforming even the Vamp-VAE and a VAE with a Sylvester NF encoder.

\textbf{Conclusions}~~ In this paper, we proposed \methods, a general schema for incorporating inductive biases into \glspl{VAE}. Experiments show that \methods can both provide representations with desired properties and improve generation quality, outperforming a variety of baselines such as directly changing the prior.
This is achieved while maintaining the simplicity and stability of standard VAEs. %framework.
%We believe our work provides a stepping stone for future works combining human knowledge with neural networks, especially in the field of deep generative models and feature extraction.
%\end{comment}
%\section*{Acknowledgements}
%\input{200acknowledgement.tex}

\bibliography{main}
\bibliographystyle{iclr2022_conference}

\newpage
\clearpage

\clearpage
\appendix
\numberwithin{equation}{section}
\numberwithin{figure}{section}
\numberwithin{table}{section}

\begin{appendices}
\section{Proofs}
\label{sec:proof}
% \setcounter{theorem}{0}
% \begin{theorem}
% Let $p_{\psi}(z)$ and $q_{\phi,\psi}(z|x)$ represent respectively the pushforward distributions of $\mathcal{N}(0,I)$ and $q_{\phi}(y|x)$ along $g_{\psi}$. If $g_\psi$ is an invertible function, then
% \begin{align}
%     \ELBO_{\Y}(x,\theta,\phi,\psi) = 
%     \E_{q_{\phi,\psi}(z|x)} [\log p_{\theta}(x|z)]-\KL{q_{\phi,\psi}(z|x)}{p_{\psi}(z)}.
%     \label{eqn:appendix:1}
% \end{align}
% \end{theorem}
\maintheorem*
\begin{proof}

%%% refrences
% data processing inequality proof based on generalized log sum inequality:% https://web.stanford.edu/class/stats311/lecture-notes.pdf
% generalized log sum inequalityhttps://arxiv.org/pdf/2005.03272.pdf
% DPI proof http://www.stat.yale.edu/~yw562/teaching/598/lec04.pdf
% DPI result https://ocw.mit.edu/courses/electrical-engineering-and-computer-science/6-441-information-theory-spring-2016/lecture-notes/MIT6_441S16_course_notes.pdf
% degenerate conditonal Jensen https://math.stackexchange.com/questions/1643216/equality-in-conditional-jensens-inequality
% direct proof https://math.stackexchange.com/questions/4216714/why-does-relative-entropy-decrease-under-pushforward/4216773#4216773
% related proof https://www.math.univ-toulouse.fr/~agarivie/sites/default/files/3_KL_PAC.pdf

We first prove the inequality from \cref{eq:bound_on_KL}, then we show that \cref{eq:bound_on_KL} is actually an equality when $g_{\psi}$ is invertible, and finally we prove that the reconstruction term is unchanged by $g_{\psi}$.

Let us denote by $\mathcal{F}$ and $\mathcal{G}$ the sigma-algebras of respectively $\Y$ and $\Z$, and we have by construction a measurable map $g_{\psi}: (\Y, \mathcal{F}) \rightarrow (\Z, \mathcal{G})$.
We can actually define the measurable space $(\Z, \mathcal{G})$ as the image of $(\Y, \mathcal{F})$ by $g_{\psi}$, then $g_{\psi}$ is automatically both surjective and measurable.\footnote{We recall that $g_{\psi}$ is said to be measurable if and only if for any $A \in \mathcal{G}$, $g_{\psi}^{-1}(A) \in \mathcal{F}$.}
% , where $A \in \mathcal{G}$, is measurable with respect to $\nu$ 
%  is measurable on $\Y$
We also assume that there exists a measure on $\Y$, which we denote $\xi$, and denote with $\nu$ the corresponding pushforward measure by $g_{\psi}$ on $\Z$.
We further have $\nu(A)=\xi(g_{\psi}^{-1}(A))$ for any $A \in \mathcal{G}$.\footnote{The notation $g_{\psi}^{-1}(A)$ does not imply that $g_{\psi}$ is invertible, but denotes the preimage of $A$ which is defined as $g_{\psi}^{-1}(A)=\{y \in \Y ~|~ g_{\psi}(y) \in A \}$.}

We start by proving \cref{eq:bound_on_KL}, where the Kullback-Leibler (KL) divergence between the two pushforward measures\footnote{We denote the pushforward of a probability measure $\chi$ along a map $g$ by $\chi \circ g^{-1}$.} 
$q_{\phi,\psi} \triangleq q_{\phi} \circ g_{\psi}^{-1}$ and $p_{\psi} \triangleq p \circ g_{\psi}^{-1}$ 
% $q_{\phi,\psi}(\cdot|x) \triangleq g_{\psi} \# q_{\phi}(\cdot|x)$ and $p_{\psi} \triangleq g_{\psi} \# p$ 
is upper bounded by $\KL{q_{\phi}(y|x)}{p(y)}$, where here we have $p(y)=\mathcal{N}(y;0,I)$ but we will use $p$ as a convenient shorthand.
At a high-level, we essentially have that \cref{eq:bound_on_KL} follows directly the data processing inequality \citep{sason2019DataProcessing} with a deterministic kernel $z=g_{\psi}(y)$.
%
% the KL divergence between $q_{\phi}(\cdot|x)$ and $p(\cdot)$.
Nonetheless, we develop in what follows a proof which additionally gives sufficient conditions for when this inequality becomes non-strict.
% To simplify the notation, we sometimes use 
We can assume that $\KL{q_{\phi}(y|x)}{\mathcal{N}(y; 0, I)}$ is finite, as otherwise the result is trivially true, which in turn implies $q_{\phi} \ll p$.\footnote{We denote the absolute continuity of measures with $\ll$, where $\mu$ is said to be absolutely continuous w.r.t.\ $\nu$, i.e.\ $\mu \ll \nu$, if for any measurable set $A$, $\nu(A)=0$ implies $\mu(A)=0$.}
%We first note that the KL divergence of interest is defined, i.e.\ $q_{\phi,\psi} \ll p_{\psi}$, since we assumed that .
For any $A \in \mathcal{G}$, we have that if 
$p_{\psi}(A) = p \circ g_{\psi}^{-1}(A) = p(g_{\psi}^{-1}(A)) = 0$ then this implies $q_{\phi}(g_{\psi}^{-1}(A)) = q_{\phi} \circ g_{\psi}^{-1}(A) = q_{\phi, \psi}(A) = 0$.
As such, we have that $q_{\phi,\psi} \ll p_{\psi}$ and so the $\KL{q_{\phi,\psi}(z|x)}{p_{\psi}(z)}$ is also defined.

Our next significant step is to show that
\begin{align} \label{eq:lemma}
\E_{p(y)} \left[\frac{q_{\phi}}{p} \given[\Big] \sigma(g_{\psi}) \right] = \frac{q_{\phi} \circ g_{\psi}^{-1}}{p \circ g_{\psi}^{-1}} \circ g_{\psi},
\end{align}
where $\sigma(g_{\psi})$ denotes the sigma-algebra generated by the function $g_{\psi}$.
To do this, let $h : (\Z, \mathcal{G}) \rightarrow (\R_+, \mathcal{B}(\R_+))$ be a measurable function s.t.\ $\E_{p(y)} \left[\frac{q_{\phi}}{p} \given[\Big] \sigma(g_{\psi}) \right] = h \circ g_{\psi}$. 
To show this, we will demonstrate that they lead to equivalent measures when integrated over any arbitrary set $A \in \mathcal{G}$:
\begin{align*} 
%\begin{split}
   \int_{\Z} \mathds{1}_{A} ~\frac{q_{\phi} \circ g_{\psi}^{-1}}{p \circ g_{\psi}^{-1}} ~ p \circ g_{\psi}^{-1} ~d\nu
   &= \int_{\Z} \mathds{1}_{A} ~q_{\phi} \circ g_{\psi}^{-1} ~d\nu %\\
   = \int_{\Z} \mathds{1}_{A} ~d(q_{\phi} \circ g_{\psi}^{-1}) \\
   &\stackrel{(a)}{=} \int_{\Y} (\mathds{1}_{A} \circ g_{\psi}) ~dq_{\phi} % ~d\xi \\
   = \int_{\Y} (\mathds{1}_{A} \circ g_{\psi}) ~q_{\phi} ~d\xi \\
   &\stackrel{(b)}{=} \int_{\Y} (\mathds{1}_{A} \circ g_{\psi}) ~\frac{q_{\phi}}{p} ~p ~d\xi \\
   &\stackrel{(c)}{=} \int_{\Y} (\mathds{1}_{A} \circ g_{\psi}) ~\E_{p(y)} \left[\frac{q_{\phi}}{p} \given[\Big] \sigma(g_{\psi}) \right] ~p ~d\xi \\
   &\stackrel{(d)}{=} \int_{\Y} (\mathds{1}_{A} \circ g_{\psi}) ~(h \circ g_{\psi}) ~p ~d\xi % \\
   = \int_{\Y} (\mathds{1}_{A} \circ g_{\psi}) ~(h \circ g_{\psi}) ~dp \\
   \displaybreak[0]
   &\stackrel{(e)}{=} \int_{\Z} \mathds{1}_{A} ~h ~d(p \circ g_{\psi}^{-1}) % \\
   = \int_{\Z} \mathds{1}_{A} ~h ~(p \circ g_{\psi}^{-1}) ~d\nu,
%\end{split}
\end{align*}
where we have leveraged the definition of pushforward measures in (a \& e); the absolute continuity of $q_{\phi}$ w.r.t.\ $p$ in (b); the conditional expectation definition in (c); and the definition of $h$ in (d).
By equating terms, we have that $q_{\phi} \circ g_{\psi}^{-1} / p \circ g_{\psi}^{-1} = h$, almost-surely with respect to $q_{\phi} \circ g_{\psi}^{-1}$ and thus that \cref{eq:lemma} is verified.

Let us define $f: x \mapsto x \log(x)$, which is strictly convex on $[0, \infty)$ (as it can be prolonged with $f(0)=0$).
We have the following
\begin{align*}
%\begin{split}
    \KL{q_{\phi,\psi}(z|x)}{p_{\psi}(z)}
    % &\stackrel{(a)}{=}\E_{q_{\phi,\psi}(z|x)} \left[\log\left(\frac{q_{\phi,\psi}}{p_{\psi}} (z) \right)\right] \\
    % &\stackrel{(c)}{=}\E_{p(y)} \left[f\left(\frac{q_{\phi,\psi}}{p_{\psi}} \circ g_{\psi}(y) \right)\right]\\
    % &\stackrel{(b)}{=}\E_{p_{\psi}(z)} \left[f\left(\frac{q_{\phi,\psi}}{p_{\psi}} (z) \right)\right] \\
    % &\stackrel{(d)}{=}\E_{p(y)} \left[f\left( \E_{p(y)} \left[\frac{q_{\phi}}{p} \given[\Big] \sigma(g_{\psi}) \right] \right)\right]\\
    % &\stackrel{(e)}{\le} \E_{p(y)} \left[ \E_{p(y)} \left[ f\left(\frac{q_{\phi}}{p}\right) \given[\Big] \sigma(g_{\psi}) \right] \right] \\
    % &\stackrel{(f)}{=}\E_{p(y)} \left[ f \left( \frac{q_{\phi}(y|x)}{p(y)} \right) \right] \\
    &\stackrel{(a)}{=} \int_{\Z} \log\left(\frac{q_{\phi,\psi}}{p_{\psi}} \right) q_{\phi,\psi} ~d\nu \\
    &\stackrel{(b)}{=} \int_{\Z} \log\left(\frac{q_{\phi,\psi}}{p_{\psi}} \right) \frac{q_{\phi,\psi}}{p_{\psi}} ~p_{\psi} ~d\nu \\
    &\stackrel{(c)}{=} \int_{\Z} f \left(\frac{q_{\phi,\psi}}{p_{\psi}} \right) ~p_{\psi}~ d\nu % \\
    = \int_{\Z} f \left(\frac{q_{\phi,\psi}}{p_{\psi}} \right) ~d(p \circ g_{\psi}^{-1}) \\
    &\stackrel{(d)}{=} \int_{\Y} f \left(\frac{q_{\phi,\psi}}{p_{\psi}} \circ g_{\psi} \right) ~dp% \\
    % = \int_{\Y} f \left(\frac{q_{\phi,\psi}}{p_{\psi}} \circ g_{\psi} \right) p ~d\xi \\
    \stackrel{}{=} \int_{\Y} f \left(\frac{q_{\phi} \circ g_{\psi}^{-1}}{p \circ g_{\psi}^{-1}} \circ g_{\psi} \right) p ~d\xi \\
    &\stackrel{(e)}{=} \int_{\Y} f\left( \E_{p(y)} \left[\frac{q_{\phi}}{p} \given[\Big] \sigma(g_{\psi}) \right] \right) p ~d\xi\\
    &\stackrel{(f)}{\le} \int_{\Y} \E_{p(y)} \left[ f\left(\frac{q_{\phi}}{p}\right) \given[\Big] \sigma(g_{\psi}) \right] p ~d\xi \\ 
    &\stackrel{(g)}{=} \int_{\Y} f \left( \frac{q_{\phi}}{p} \right) p ~d\xi \\
    &\stackrel{(h)}{=} \int_{\Y} \log \left( \frac{q_{\phi}}{p} \right) \frac{q_{\phi}}{p} ~p ~d\xi \\
    &\stackrel{(i)}{=}\E_{q_{\phi}(y|x)} \left[\log \left(\frac{q_{\phi}(y|x)}{p(y)} \right) \right] \\
    &\stackrel{(j)}{=}\KL{q_{\phi}(y|x)}{p(y)},
%\end{split}
\end{align*}
where we leveraged the definition of the KL divergence in (a \& j); the absolute continuity of $q_{\phi}$ w.r.t.\ $p$ in (b \& i); the definition of $f$ in (c \& h); the definition of the pushforward measure in (d); \cref{eq:lemma} in (e); the conditional Jensen inequality in (f) and the law of total expectation in (g).
Note that this proof not only holds for the KL divergence, but for any f-divergences as they are defined as in (b) with $f$ convex.

To prove \cref{eq:elbo_equivalence}, we now need to show that line (f) above becomes an equality when $g_{\psi}$ is invertible.
%we still need to study the conditions for which the conditional Jensen inequality in (e) is not strict, i.e.\ where it is an equality.
As $f$ is strictly convex, this happens if and only if $\frac{q_{\phi}}{p} = \E_{p(y)} \left[\frac{q_{\phi}}{p} \given[\Big] \sigma(g_{\psi}) \right]$.
A sufficient condition for this to be true is for $\frac{q_{\phi}}{p}$ to be measurable w.r.t.\ $\sigma(g_{\psi})$ which is satisfied when $g_{\psi}: \Y \mapsto \Z$ is invertible as $\sigma(g_{\psi}) \supseteq \mathcal{F}$, as required.
We have thus shown that the KL divergences are equal when using an invertible $g_{\psi}$.
% contains all open sets from $\Y$. 
%

% We now show that the reconstruction terms from $\ELBO_{\Y}(x,\theta,\phi,\psi)$ and $\ELBO_{\Z}(x,\theta,\phi,\psi)$ are equal where by definition we have
% \begin{align}
%     \ELBO_{\Y}(x,\theta,\phi,\psi) \triangleq \E_{q_{\phi}(y|x)} [\log p_{\theta}(x|g_{\psi}(y))] 
%         - \KL{q_{\phi}(y|x)}{p(y)},
%     \label{eqn:appendix:2}
% \end{align}
% with $p(y)=\mathcal{N}(y; 0, I)$ being a standard normal distribution.

% According to the definition of pushforward probabilities $p_{\psi}(z)$ and $q_{\phi,\psi}(z|x)$,
% %\begin{align}
% %    \int_{\Y} \mathds{1}_{g_{\psi}^{-1}(A)}(g_{\psi}(y)) p(y)d\xi &=\int_{\Z} \mathds{1}_{A}(z) p_{\psi}(z)d\nu,\\
% %    \int_{\Y} \mathds{1}_{g_{\psi}^{-1}(A)}(g_{\psi}(y)) q_{\phi}(y|x)d\xi &=\int_{\Z} \mathds{1}_{A}(z) q_{\phi,\psi}(z|x)d\nu.
% %\end{align}
% %And consequently,
% \begin{align}
%     \int_{\Y} f(g_{\psi}(y)) p(y)d\xi &=\int_{\Z} f(z) p_{\psi}(z)d\nu,\\
%     \int_{\Y} f(g_{\psi}(y)) q_{\phi}(y|x)d\xi &=\int_{\Z} f(z) q_{\phi,\psi}(z|x)d\nu.
% \end{align}
% for any measurable function $f$.

For the reconstruction term, we instead have
\begin{align*}
%\begin{split}
    \E_{q_{\phi}(y|x)} [\log p_{\theta}(x|g_{\psi}(y))]
    &=\int_{\Y} \log p_{\theta}(x|g_{\psi}(y)) q_{\phi}(y|x)d\xi\\
    &=\int_{\Z} \log p_{\theta}(x|z) q_{\phi,\psi}(z|x)d\nu\\
    &=\E_{q_{\phi,\psi}(z|x)} [\log p_{\theta}(x|z)].
%\end{split}
\end{align*}
\cref{eq:elbo_equivalence} now follows from the fact that both the reconstruction and KL terms are equal.
%we have that $\ELBO_{\Y}(x,\theta,\phi,\psi) = \ELBO_{\Z}(x,\theta,\phi,\psi)$.

\end{proof}

\newpage 
\section{Hierarchical Representations}

%\subsection{Hierarchical Representations}

%
\label{sec:hierarchical_appendix}
 \begin{wrapfigure}[15]{r}{0.4\textwidth}
 \vspace{-0.6cm}
  \centering
  \includegraphics[width=0.95\textwidth]{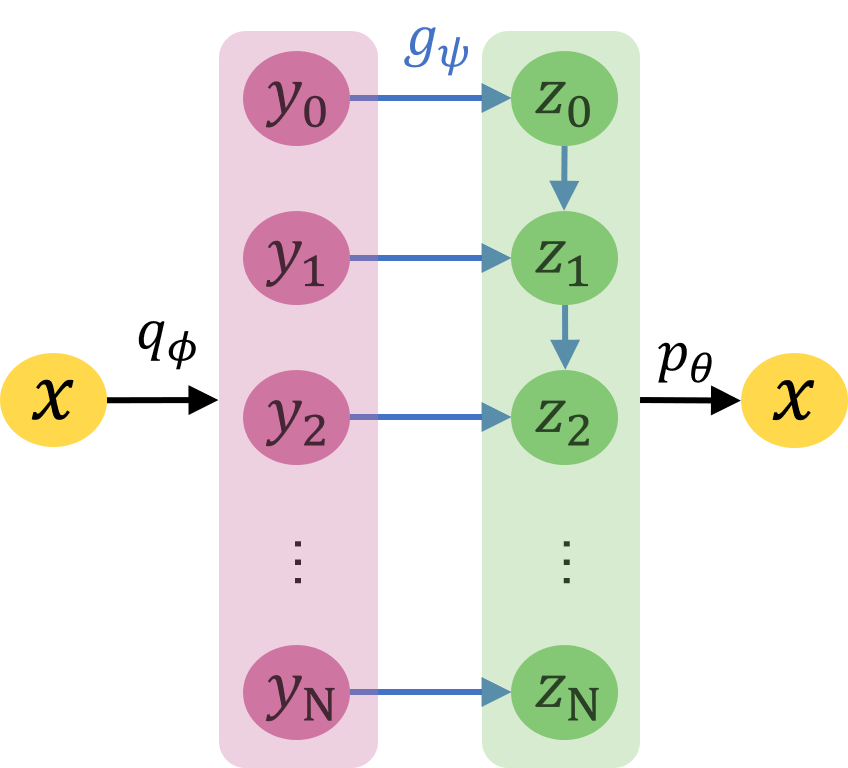}
  \vspace{-4pt}
  \caption{Graphical model for hierarchical
  \method
  }
  \label{fig:hierarchical_mapping}
\end{wrapfigure}
The isotropic Gaussian prior in standard \Glspl{VAE} assumes that representations are independent across dimensions~\citep{kumar2018variational}.
% However, hierarchies naturally exist in many data sets, which violates the independence assumption.
However, this assumption is often unrealistic~\citep{belghazi2018hierarchical,mathieu2019disentangling}.
For example, in Fashion-MNIST, high-level features such as object category, may affect low-level features such as shape or height. 
%In addition, splitting
Separately extracting such global and local information can be beneficial for visualization and data manipulation~\citep{zhao2017learning}. 
To try and capture this, we introduce an inductive bias that is tailored to model and learn hierarchical features.
We note here that our aim is not to try and provide a state-of-the-art hierarchical \gls{VAE} approach, as a wide variety of highly--customized and powerful approaches are already well--established, but to show how easily the \method framework can be used to induce hierarchical representations in a simple, lightweight, manner.
% In vanilla \gls{VAE}, one tends to learn a disentangled hidden variable $z$, meaning that dimensions $z_i$ are independent~\citep{kumar2018variational}.
% However, in some scenarios, $z_i$ may affect each other~\citep{belghazi2018hierarchical}.
% For example, when generating face images, high level features, such as gender or race, may affect low level features, such as shape of beard or eye color.
% As a result, we try to introduce hierarchies to priors to adapt to these scenarios and extract hierarchical features.

%Inspired by Ladder Variational Autoencoders~(LVAE, \cite{sonderby2016ladder}) and Variational Ladder Autoencoder~(VLAE, \cite{zhao2017learning}), 
%we design the structure of hierarchical mapping $g_\psi$ illustrated in Fig.~\ref{fig:hierarchical_mapping}.
\paragraph{Mapping design}
Following existing ideas from hierarchical \glspl{VAE}~\citep{sonderby2016ladder, zhao2017learning}, we propose a hierarchical mapping $g_\psi$.
%illustrated in Fig.~\ref{fig:hierarchical_mapping}.
%We first divide $y$ into several levels, where $y_j=y_{[I_j:I_{j+1}]}$ is the $j$-th level for a index list $I$.
%$y_0=y_{[0:i_0]}$ is the top level, $y_1=y_{[i_0:i_1]}$ is the second level and so on.
As shown in \cref{fig:hierarchical_mapping}, the intermediary Gaussian variable $y$ is first split into a set of $N$ layers
%a sequence of Gaussian vectors
$[y_0, y_1,...,y_N]$. 
The mapping $z=g_\psi(y)$ is then recursively defined as
$z_i = \text{NN}_i(z_{i-1}, y_i)$,
% \begin{equation}
%     z_i = \text{NN}_i(z_{i-1}, y_i),
% \end{equation}
%
%
where $\text{NN}_i$ is a neural network combining information from higher-level feature $z_{i-1}$ and new information from $y_i$.
%where $i\ge 1$ and $\text{NN}_i$ is a neural network combining information from $z_{i-1}$ and randomness from $y_i$.
As a result, we get a hierarchical encoding $z=[z_0, z_1,...,z_N]$, where high-level features influence low-level ones but not vice-versa.
This $g_\psi$ thus endows \methods with hierarchical representations.
%hierarchical/
% By simply adding a mapping layer to vanilla-\gls{VAE}, \method largely simplifies the structure of hierarchical \glspl{VAE}.
%Compared with LVAE and VLAE, \method decodes from features in all levels rather than only the bottom-level feature, which mitigates the problem of gradient and information decay.

\begin{wrapfigure}[24]{r}{0.44\textwidth}
 \vspace{-1.0em}
 \centering
 \setcounter{subfigure}{0}
 \vspace{-3pt}
 \sidesubfloat[]{
 \label{fig:hierarchical_1}
 \centering
 \includegraphics[width=0.8\textwidth]{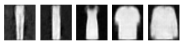}
 }
 \vspace{-3pt}
 \sidesubfloat[]{
 \label{fig:hierarchical_2}
 \centering
 \includegraphics[width=0.8\textwidth]{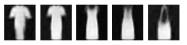}
 }\vspace{-3pt}
 \sidesubfloat[]{
 \label{fig:hierarchical_3}
 \centering
 \includegraphics[width=0.8\textwidth]{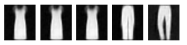}
 }\vspace{-3pt}
 \sidesubfloat[]{
 \label{fig:hierarchical_4}
 \centering
 \includegraphics[width=0.8\textwidth]{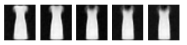}
 }\vspace{-3pt}
 \sidesubfloat[]{
 \label{fig:hierarchical_5}
 \centering
 \includegraphics[width=0.8\textwidth]{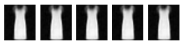}
 }
 \vspace{-1.2ex}
 \caption{
 Manipulating representations of a hierarchical \method. The features are split into 5 levels, with each of (a) [highest] to (e) [lowest] corresponding to an example feature from each.
 %row as a representative for one level from high to low. It can be easily seen 
 We see that high-level features control more complex properties, such as class label or topological structure, while low-level features control simpler details, (e.g.~(d) controls collar shape).
 \vspace{-1.2em}
 }
 \label{fig:example_hierarchical}
\end{wrapfigure}
\paragraph{Experiments}
While conventional hierarchical \glspl{VAE}, e.g.~\citep{sonderby2016ladder, zhao2017learning,vahdat2020nvae}, use hierarchies to try and improve generation quality, our usage is explicitly from the representation perspective, with our experiments set up accordingly.
%Our goal is to provide a simple and efficient way to extract useful hierarchical features, instead of competing in generation quality with potentially complex hierarchical \glspl{VAE}, such as NVAE~\citep{vahdat2020nvae}.
%Please notice that we don't aim to beat powerful hierarchical \glspl{VAE}, such as NVAE~\citep{vahdat2020nvae} in generation quality, because of their specially designed encoders. We only want to show that \method provides a simpler way to split the information into different levels, which might be useful for downstream task. 
Fig.~\ref{fig:example_hierarchical} shows some hierarchical features learned by \method on \textbf{Fashion-MNIST}. %
We observe that high-level information such as categories have indeed been learned in the top-level features, while low-level features control more detailed aspects.
%, such as collar shapes. 
% We can easily tell that high-level information, such as categories are saved in top-level features, while low-level features control more detailed aspects, such as collar shapes. 

%
To provide more quantitative investigation, we also consider the 
%We perform an extensive assessment on the
\textbf{CelebA} dataset~\citep{liu2015faceattributes} 
and investigate performance 
%to further show the benefits of hierarchical features 
on downstream tasks, comparing to vanilla-\glspl{VAE} with different latent dimensions. 
For this, we train a linear classifier to predict all $40$ binary labels from the learned features for each method.
In order to eliminate the effect of latent dimensions, we compare \method (with fixed latent dimension $128$) and vanilla \gls{VAE} with different latent dimensions~($1,2,4,8,16,32,64,128$).
We show experiment results on some labels as well as the average accuracy on all labels in \cref{tab:celebA_downstream} and \cref{fig:hierarchical_downstream}.
We first find that the optimal latent dimension increases with the number of data points for the vanilla-\glspl{VAE}, but is always worse than the \method. 
Notably, the accuracy with \method is quite robust, even as the number of data points gets dramatically low, indicating high data efficiency. 
To the best of our knowledge, this is the first result showing that a hierarchical inductive bias in \gls{VAE} is beneficial to feature quality.

\textbf{Related work}~~ Hierarchical \glspl{VAE}~\citep{vahdat2020nvae, ranganath2016hierarchical,sonderby2016ladder, klushynlearning, zhao2017learning} seek to improve the fit and generation quality of \glspl{VAE} by recursively correcting the generative distributions.
However, they require careful design of neural layers, and the hierarchical KL divergence makes training deep hierarchical \glspl{VAE} unstable~\citep{vahdat2020nvae}. 
In comparison, \method with hierarchical mappings is extremely easy to implement without causing any computational instabilities, while its aims also differ noticeably: our approach successfully learns hierarchical \emph{representations}---something that is rarely mentioned in prior works.

\begin{figure}[p]
 \includegraphics[width=1.0\textwidth]{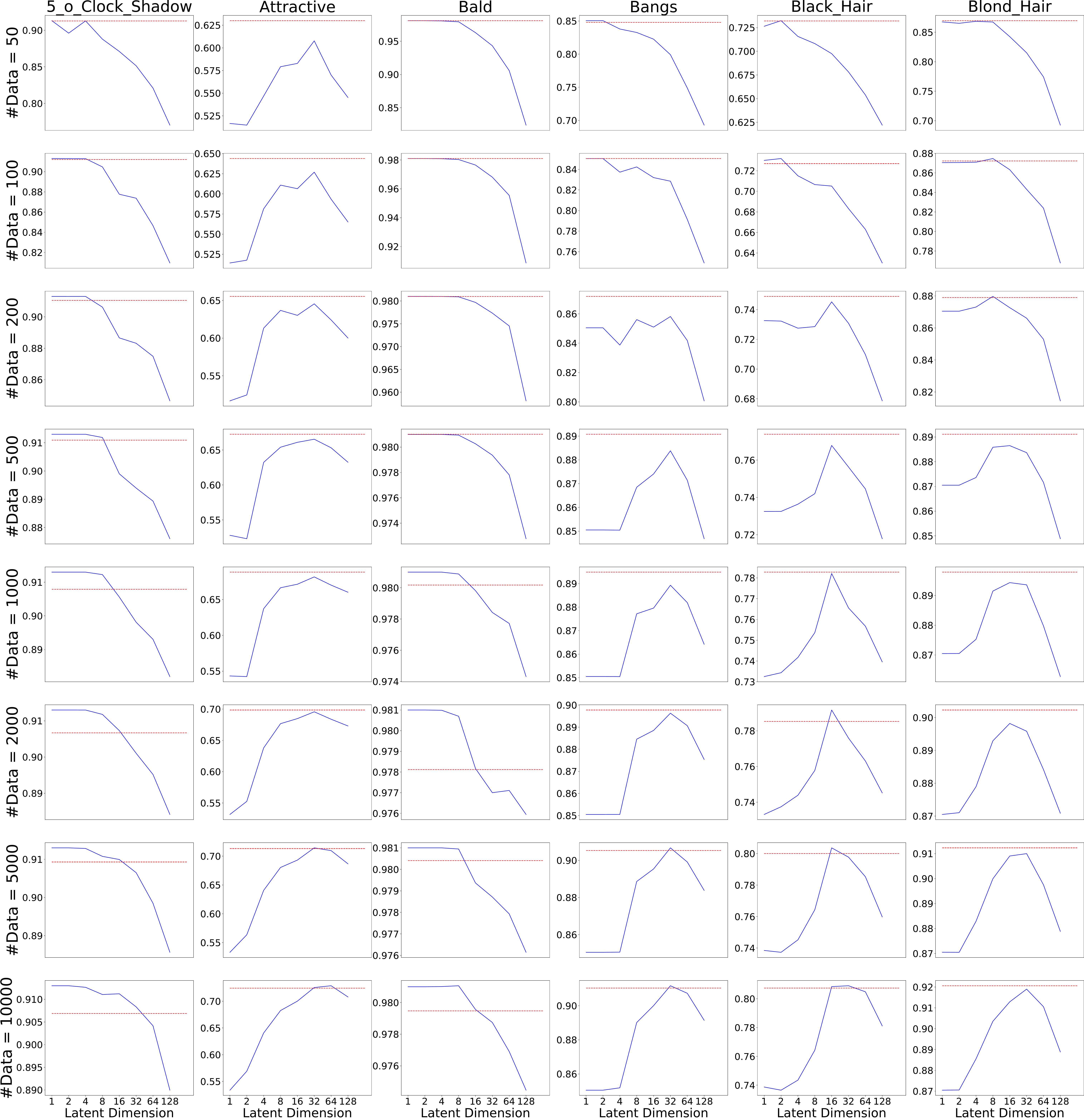}
 \caption{\method 's performance of attribute prediction on \textbf{CelebA} dataset. Each column shows results on the same feature with different data sizes and each column shows results on different features. In each graph, test accuracy of vanilla-\gls{VAE} with different latent dimensions are shown in blue line. And results of \method with hierarchical prior are shown in red. We find that our method (red line) achieves comparable or even better results compared with vanilla-VAE with all latent dimensions.}
 \label{fig:hierarchical_downstream}
\end{figure}

\begin{table}[tb]
\caption{
Average accuracy in predicting all 40 binary labels of \textbf{CelebA}. Overall best accuracy is shown in bold and best results of vanilla-\glspl{VAE} are underlined for comparison. Each experiment is repeated 10 times and differences are significant at the $5\%$ level for 
%p-values are $\le 0.05$ when comparing \method with each baseline with 
data size $\le 1000$.
}
\centering
\label{tab:celebA_downstream}
{\small \begin{tabular}{ llcccccc } 
 \toprule
 Model & Latent  dim&\multicolumn{6}{c}{Data size}\\\cmidrule{3-8}
 &&50&100&500&1000&5000&10000\\\midrule
 \gls{VAE}&8&\underline{0.791}&0.799&0.814&0.815&0.819&0.819\\
 &16&0.788&\underline{0.801}&0.820&0.824&0.829&0.831\\
 &32&0.769&0.795&0.825&\underline{0.832}&0.842&0.846\\
 &64&0.767&0.794&\underline{0.826}&\underline{0.832}&\underline{0.849}&\underline{0.855}\\
 &128&0.722&0.765&0.817&0.825&0.830&0.852\\\midrule
 \method&64&\textbf{0.817}&\textbf{0.824}&\textbf{0.841}&\textbf{0.846}&\textbf{0.854}&\textbf{0.857}\\
 \bottomrule
\end{tabular}}
\end{table}

\newpage
\label{sec:hierarchical}

\section{Full Method and Experiment Details}
\label{sec:mapping_detail}
In this section, we first provide complete details of the mapping designs used for our different \method realizations along with some additional experiments.  We then provide other general information about datasets, network structures, and experiment settings to facilitate results reproduction.

%This section shows the details of mapping designs and experiments. All experiments are run on a single-GPU server.
%\emile{Perhaps give more details on the gpus used}

\subsection{Multiple-connectivity}
\label{sec:mapping_multiply_connected_appendix}
\paragraph{Mapping design}
Full details for this mapping were given in the main paper.~\Cref{fig:glue} provides a further illustration of the gluing process.  Additional resulting including the Vamp-VAE are given in~\cref{fig:examples_low_dim}.
% As described in the main text, when using the inductive bias of a single hole, we use the following $g_{\psi}$ to map a Gaussian distribution approximately to a circular distribution,
% \begin{align}
%     g_{1}(y)=\frac{y}{||y||_2+\epsilon}.
% \end{align}
% As a result, $p_{\psi}(z)=g_{\psi}(p(y))$ is approximately the uniform distribution on $S^1$.

% To introduce an additional hole, we simply glue two points on $S^1$ together to make more holes. For example
% \begin{align}
%     g_2(y)&=\text{Concat}\left(g_{1}(y)_{[:,1]},~ g_{1}(y)_{[:,2]}\sqrt{(4/3-(1-|g_{1}(y)_{[:,1]}|)^2)}-\frac{1}{\sqrt{3}}\right),
% \end{align}
% which first map $y$ to approximately $S^1$, and then glues $(0, 1)$ and $(0, -1)$ together to create new holes. (see \cref{fig:glue} for an illustration.) 
% %In the glue function $z_1$ and $z_2$ are the first and second dimensions of the input.
% Furthermore, we can continue to glue points together to achieve a higher number of holes $h$, and thus more complex connectivity.
\begin{figure}[h]
 \centering
 \subfloat[Circular prior with $h=1$]{
 \label{fig:ot_1}
 \centering
 \includegraphics[width=0.25\textwidth]{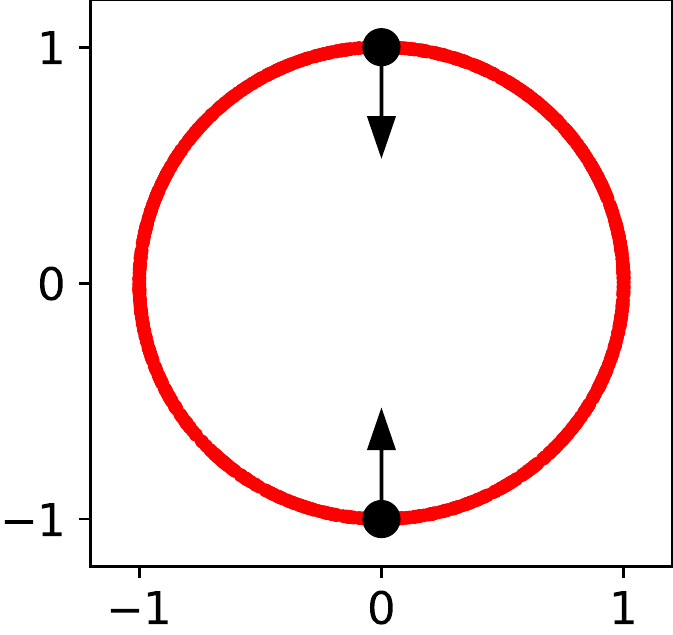}
 }
 \subfloat[Glue point pair]{
 \label{fig:ot_2}
 \centering
 \includegraphics[width=0.25\textwidth]{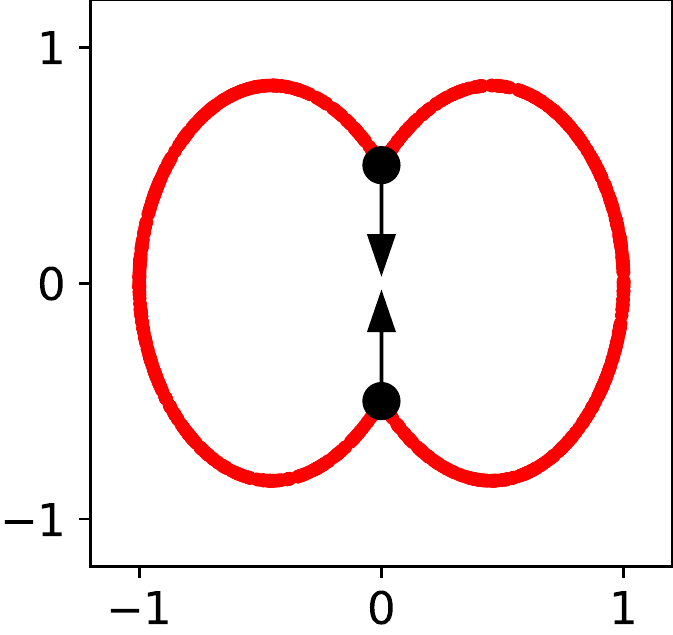}
 }
 \subfloat[Implied prior with $h=2$]{
 \label{fig:ot_3}
 \centering
 \includegraphics[width=0.25\textwidth]{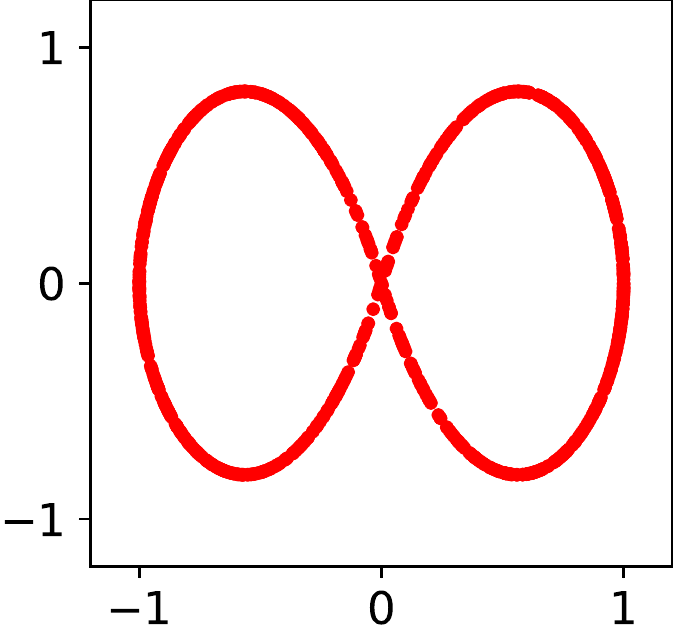}
 }
 \caption{An illustration of the glue function in multiply-connected mappings.}
 \label{fig:glue}
\end{figure}

\subsection{Multi-modality}
\paragraph{Mapping design}
\label{sec:clustered_appendix}
\begin{comment}
%ICLR
\begin{figure}[t]
 \centering
 \subfloat[Original Distribution]{
 \label{fig:cluster_dist_a}
 \centering
 \includegraphics[width=0.22\textwidth]{figure/cluster_dist_a.png}
 }
 \hspace{10pt}
 \subfloat[Split into $K$ sectors]{
 \label{fig:cluster_dist_b}
 \centering
 \includegraphics[width=0.22\textwidth]{figure/cluster_dist_b.png}
 }
 \hspace{10pt}
 \subfloat[Move sectors away from origin]{
 \label{fig:cluster_dist_c}
 \centering
 \includegraphics[width=0.22\textwidth]{figure/cluster_dist_c.png}
 }
 \caption{Illustration of clustered mapping. The circle represents a density isoline of Gaussian distribution.  Note that not all points in the sector are moved equally: points close to the boundaries between sectors are moved less, with points on the boundary themselves not moved at all as per~\cref{eqn:cluster}.}
 
 \label{fig:cluster_mapping}
\end{figure}
\end{comment}

In \cref{sec:clustered}, we see the general idea of designing clustered mappings. In this part, we delve into the details of mapping design as well as extending it to 1 dimensional and high-dimensional cases.
For simplicity's sake let us temporarily assume that the dimension of $\Y$ is $2$.
Our approach is based on splitting the original space into $K$ equally sized sectors, where $K$ is the number of clusters we wish to create, as shown in~\cref{fig:cluster_dist_b_large}.
For any point $y$, we can get its component~(sector) index $\text{ci}(y)$ %\emile{component index? $\sec$ already means $1/\cos$ so not ideal choice of name} 
as well as its distance from the sector boundary $\text{dis}(y)$. By further defining the radius direction for the $k$-th sector (cf~\cref{fig:cluster_dist_c_large}) as 
\[
\Delta(k) = \left(\cos\left(\frac{2\pi}{K}\left(k+\frac{1}{2}\right)\right), \sin\left(\frac{2\pi}{K}\left(k+\frac{1}{2}\right)\right)\right) \quad \forall k\in\{1,\dots,K\},
\]
we can in turn define $g(y)$ as:
\begin{align}
    \label{eqn:cluster}
    \text{r}(y) &= \Delta(\text{ci}(y)),\\
    g(y) &= y + {c_1}\text{dis}(y)^{c_2}\text{r}(y),
\end{align}
where $c_1$ and $c_2$ are constants, which are set to 5 and 0.2 in our experiments.
we make sure $g$ still continuous by keeping $g(y)=y$ on boundaries.

\begin{figure}[t]
 \centering
 \subfloat{
 \label{fig:sphere_real_appendix}
 \centering
 \includegraphics[width=0.22\textwidth]{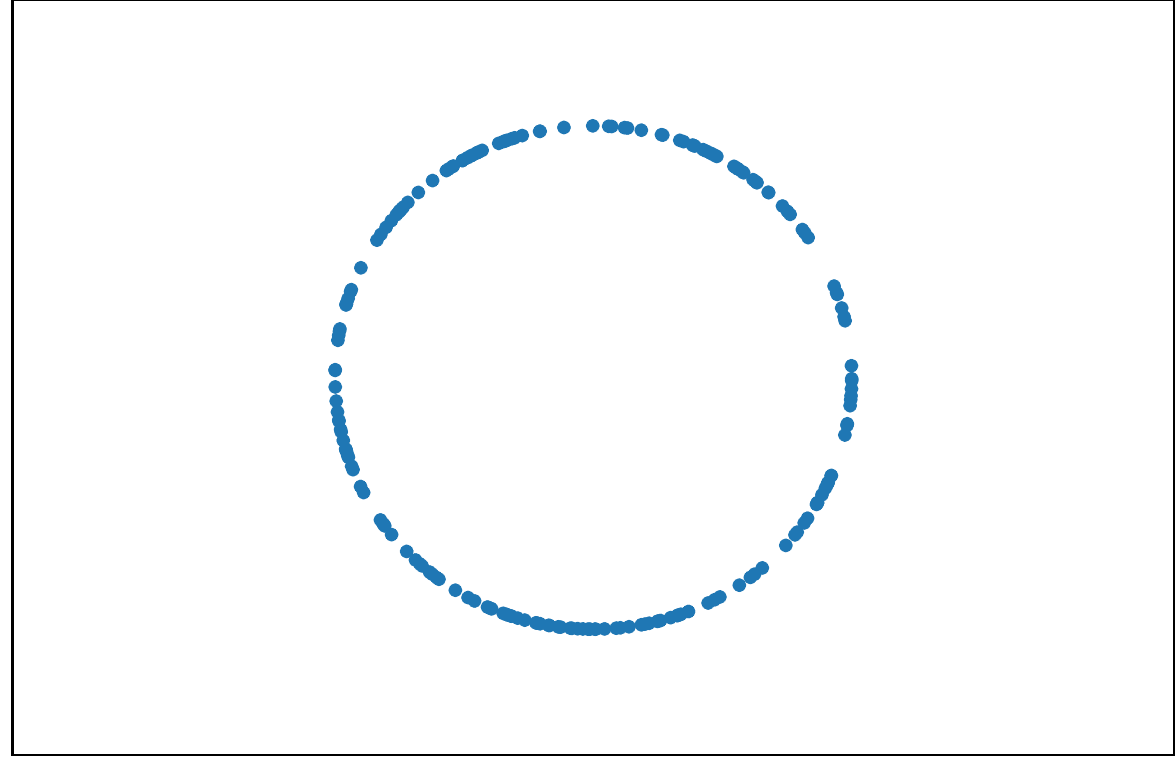}
 }
 \subfloat{
 \label{fig:sphere_vanilla_appendix}
 \centering
 \includegraphics[width=0.22\textwidth]{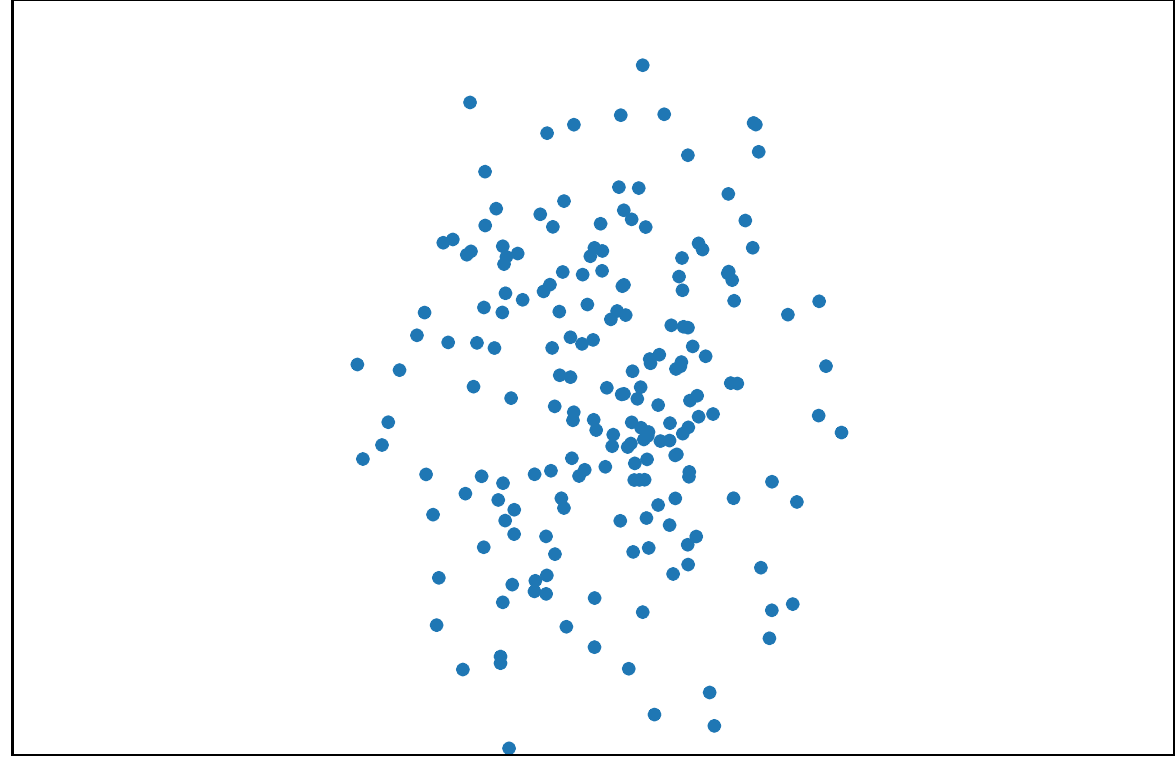}
 }
 \subfloat{
 \label{fig:sphere_vamp_appendix}
 \centering
 \includegraphics[width=0.22\textwidth]{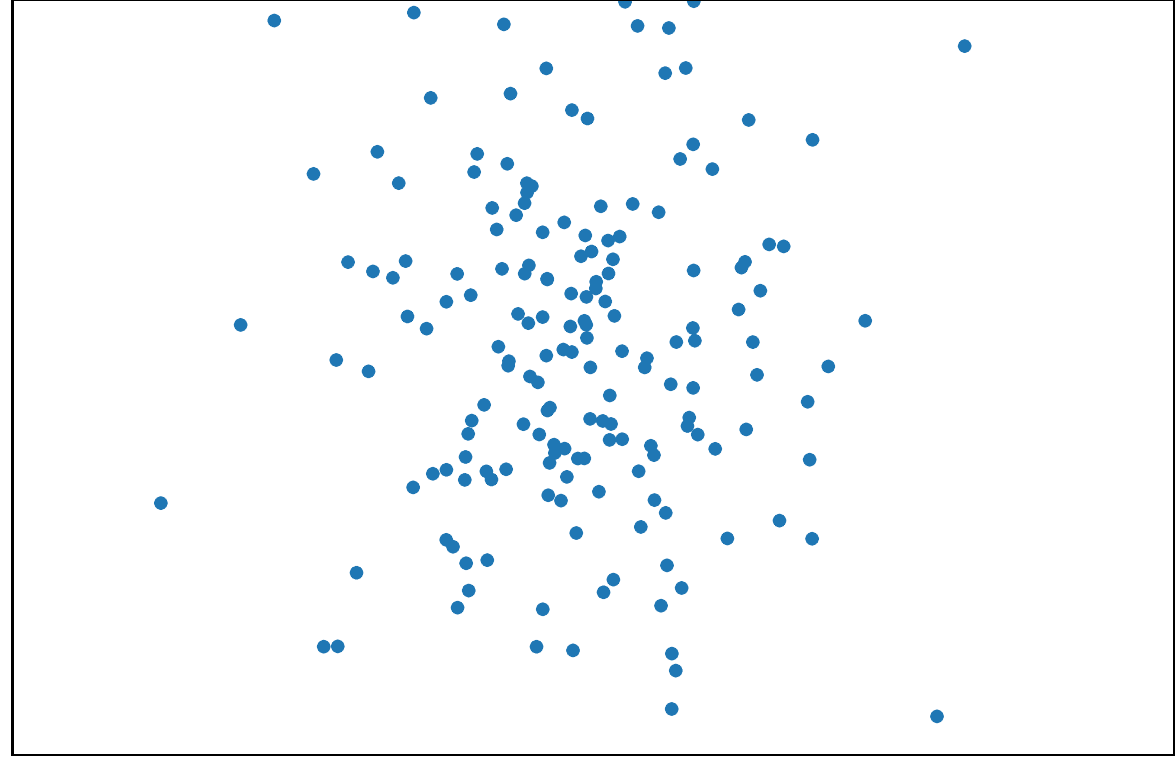}
 }
 \subfloat{
 \label{fig:sphere_ours_appendix}
 \centering
 \includegraphics[width=0.22\textwidth]{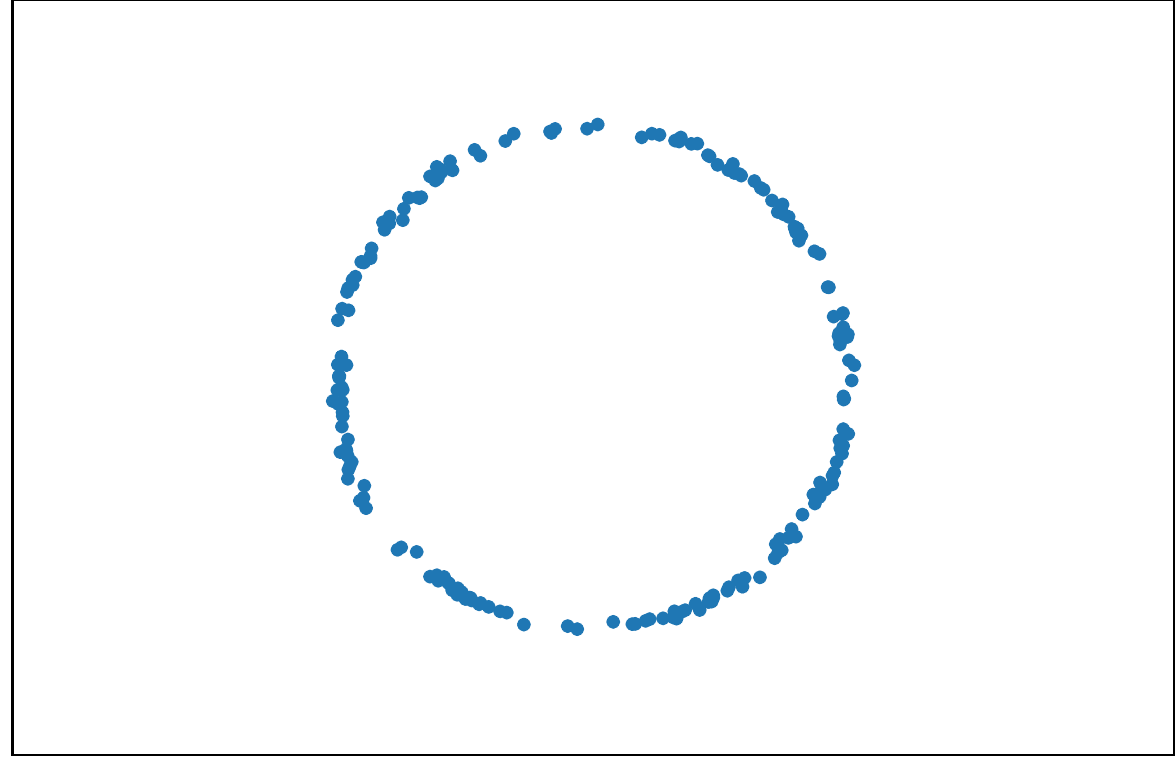}
 }
 \vspace{-10pt}
 
 \subfloat{
 \label{fig:square_real_appendix}
 \centering
 \includegraphics[width=0.22\textwidth]{figure/square_real.pdf}
 }
 \subfloat{
 \label{fig:square_vanilla_appendix}
 \centering
 \includegraphics[width=0.22\textwidth]{figure/square_vanilla.pdf}
 }
 \subfloat{
 \label{fig:square_vamp_appendix}
 \centering
 \includegraphics[width=0.22\textwidth]{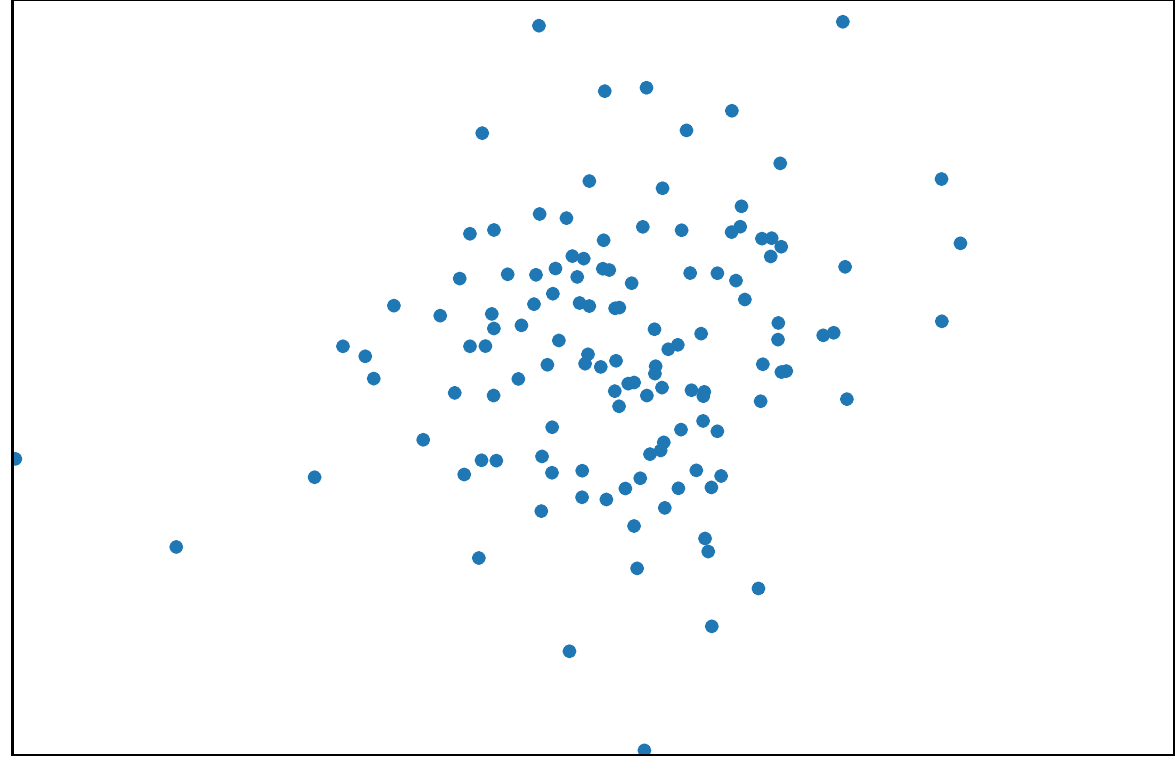}
 }
 \subfloat{
 \label{fig:square_ours_appendix}
 \centering
 \includegraphics[width=0.22\textwidth]{figure/square_ours.pdf}
 }

 \vspace{-10pt}
  \subfloat{
 \label{fig:star_real_appendix}
 \centering
 \includegraphics[width=0.22\textwidth]{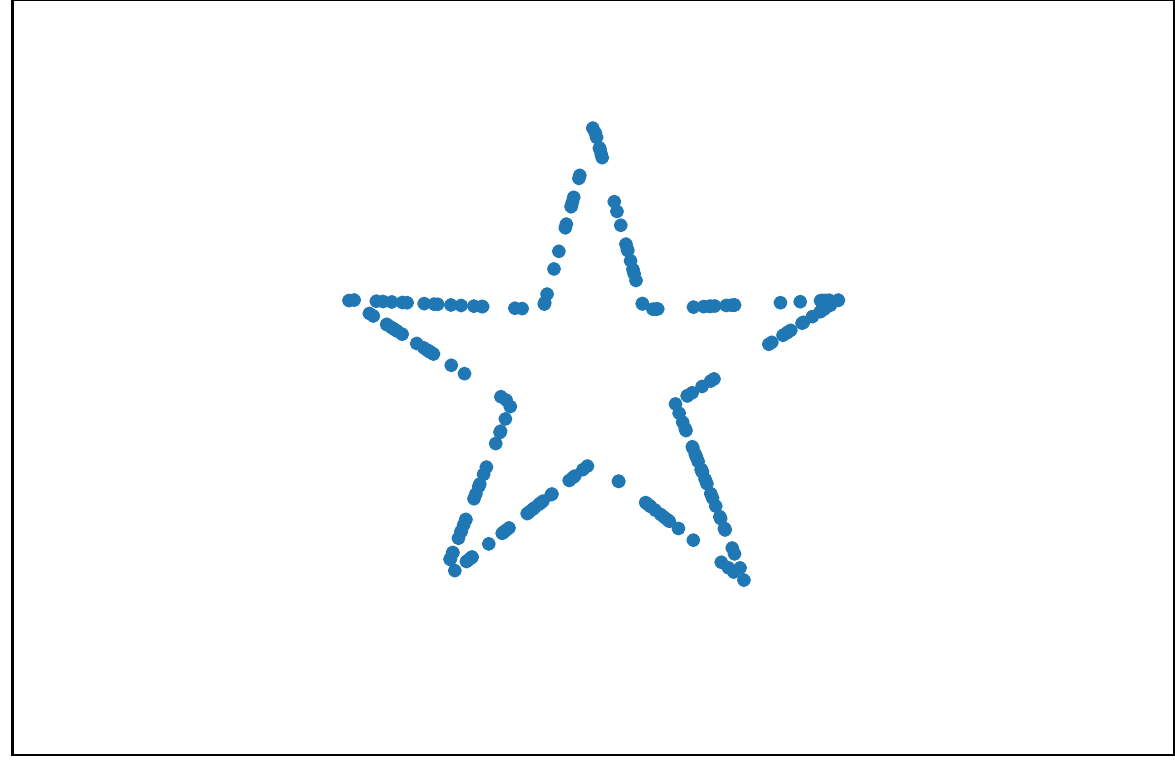}
 }
 \subfloat{
 \label{fig:star_vanilla_appendix}
 \centering
 \includegraphics[width=0.22\textwidth]{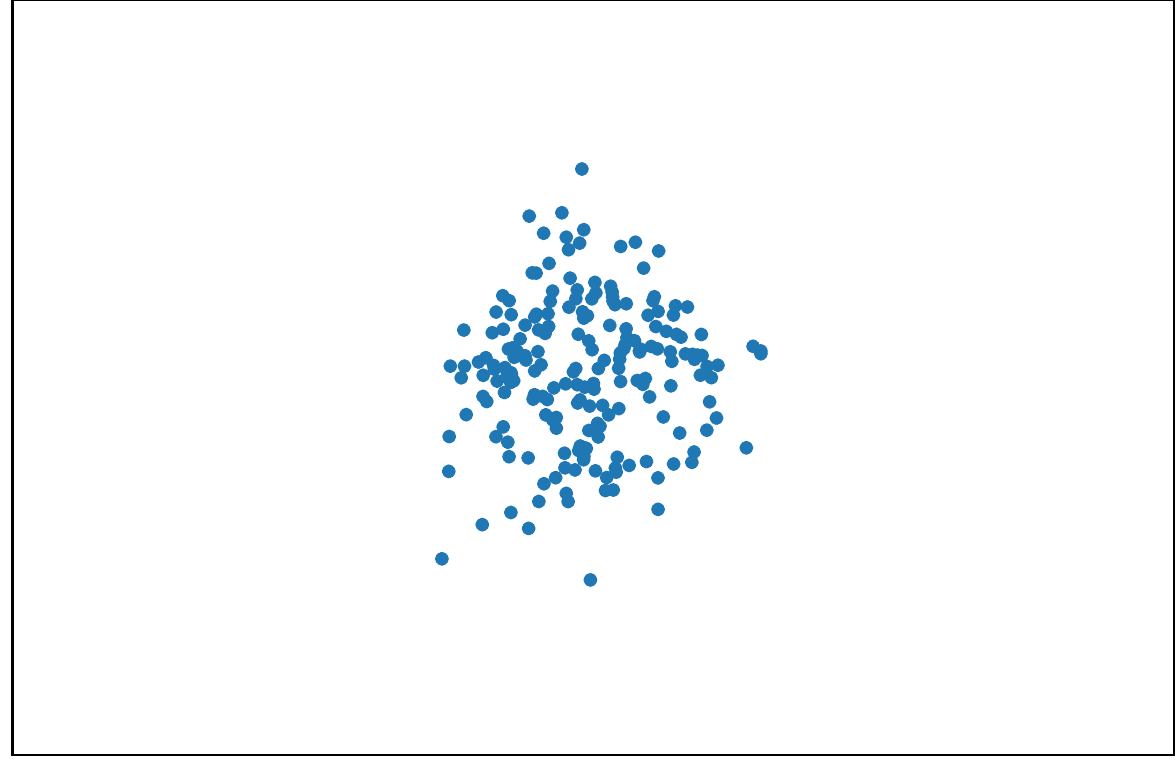}
 }
 \subfloat{
 \label{fig:star_vamp_appendix}
 \centering
 \includegraphics[width=0.22\textwidth]{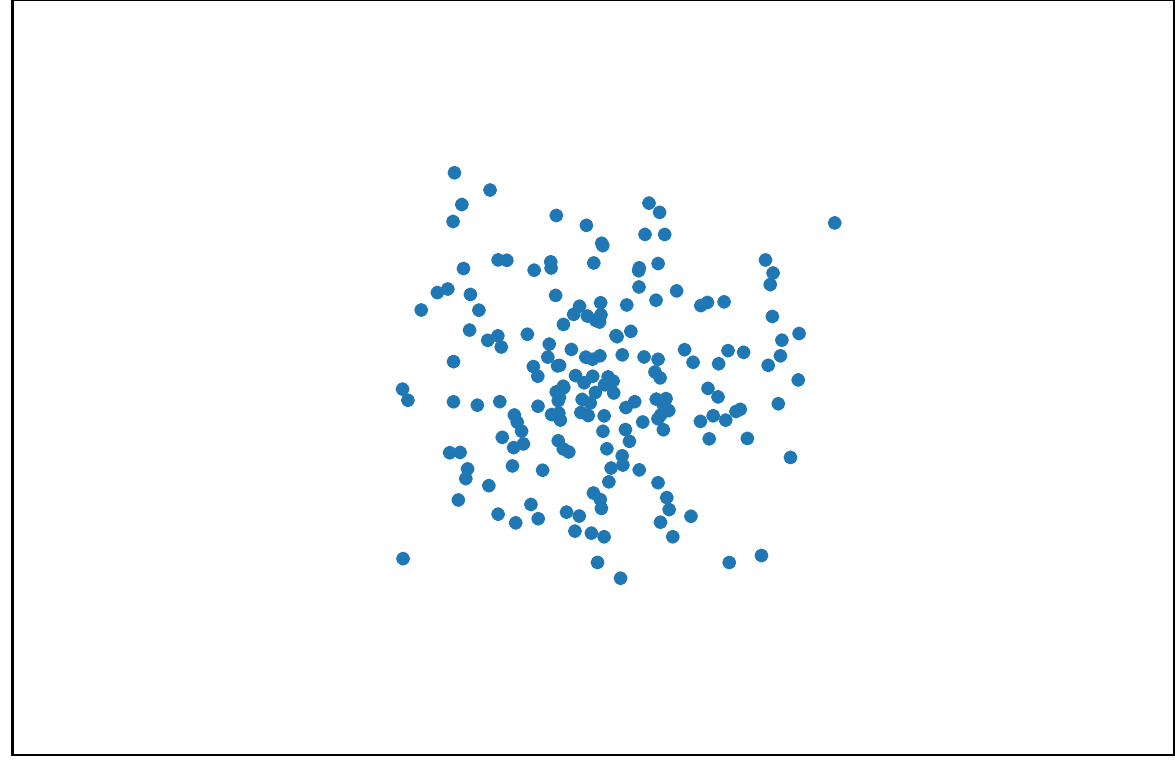}
 }
 \subfloat{
 \label{fig:star_ours_appendix}
 \centering
 \includegraphics[width=0.22\textwidth]{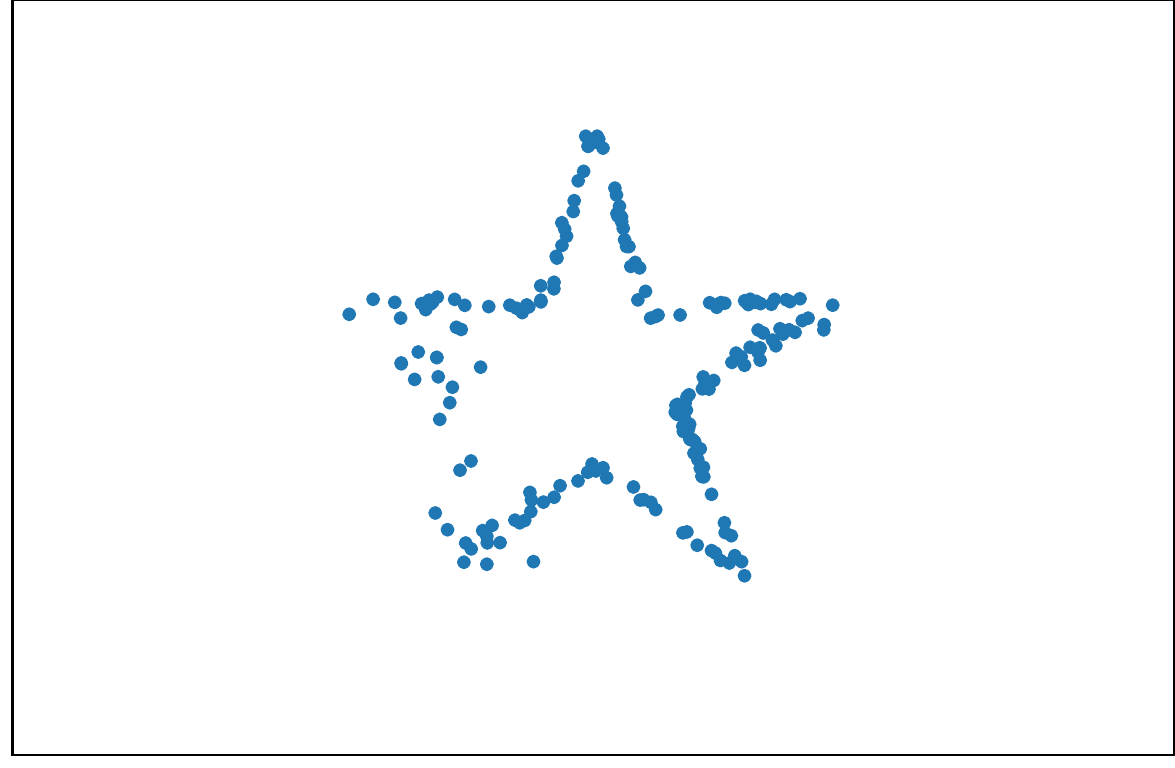}
 }
 
 \vspace{-10pt}
  \subfloat{
 \label{fig:infinity_real_appendix}
 \centering
 \includegraphics[width=0.22\textwidth]{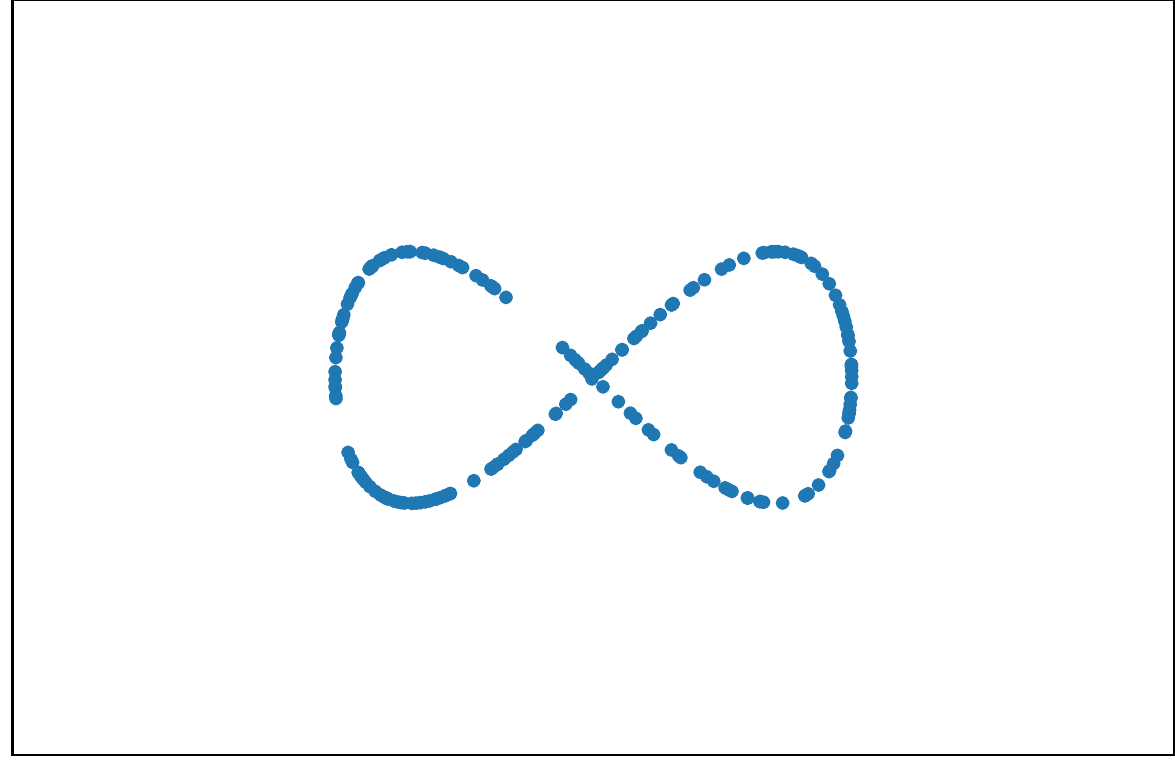}
 }
 \subfloat{
 \label{fig:infinity_vanilla_appendix}
 \centering
 \includegraphics[width=0.22\textwidth]{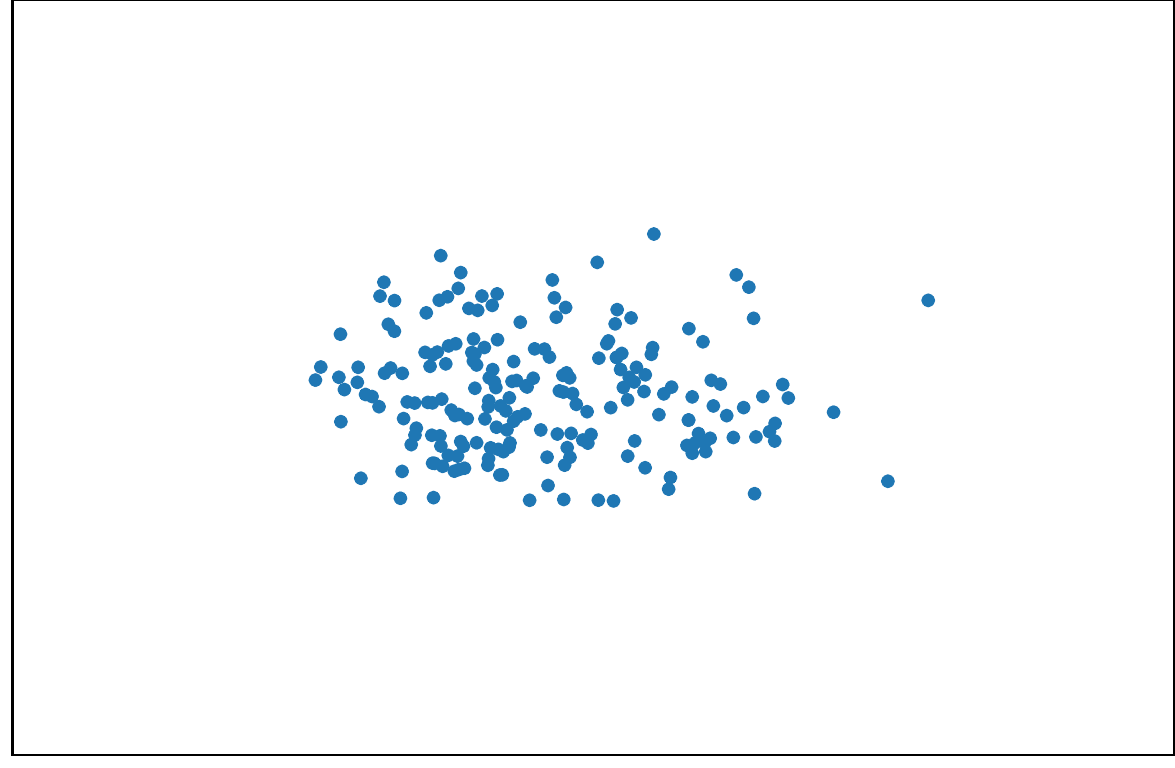}
 }
 \subfloat{
 \label{fig:infinity_vamp_appendix}
 \centering
 \includegraphics[width=0.22\textwidth]{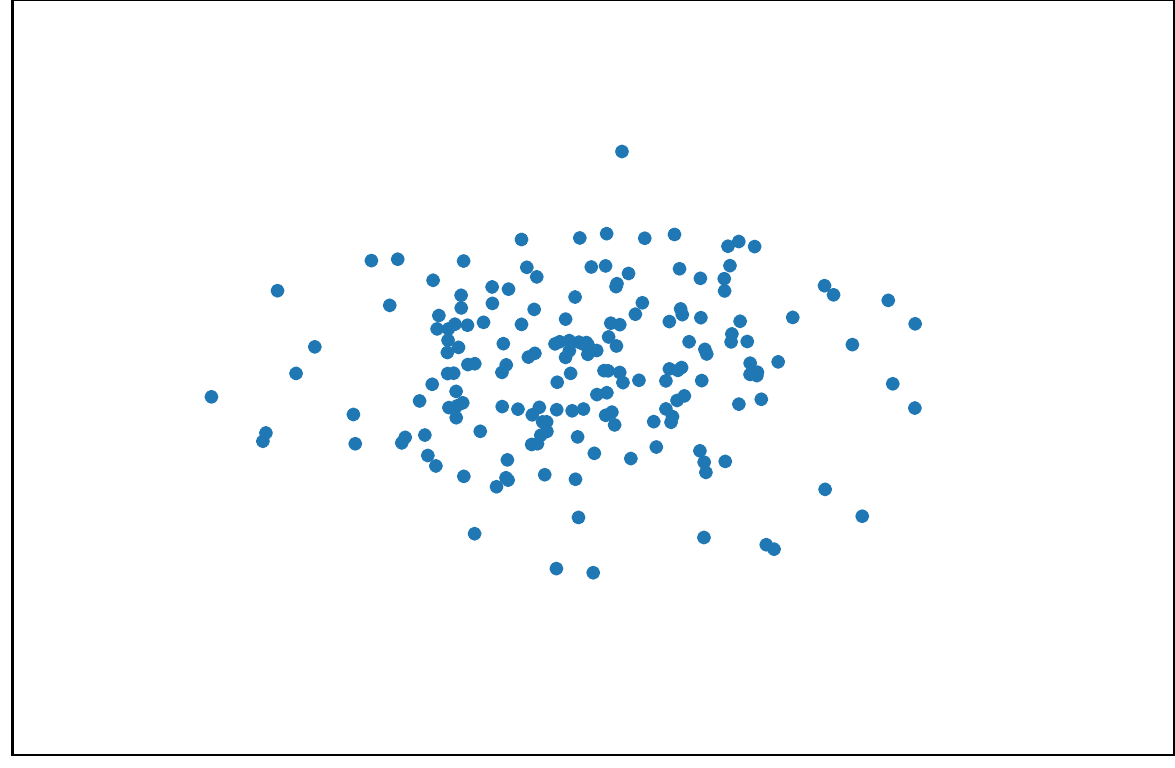}
 }
 \subfloat{
 \label{fig:infinity_ours_appendix}
 \centering
 \includegraphics[width=0.22\textwidth]{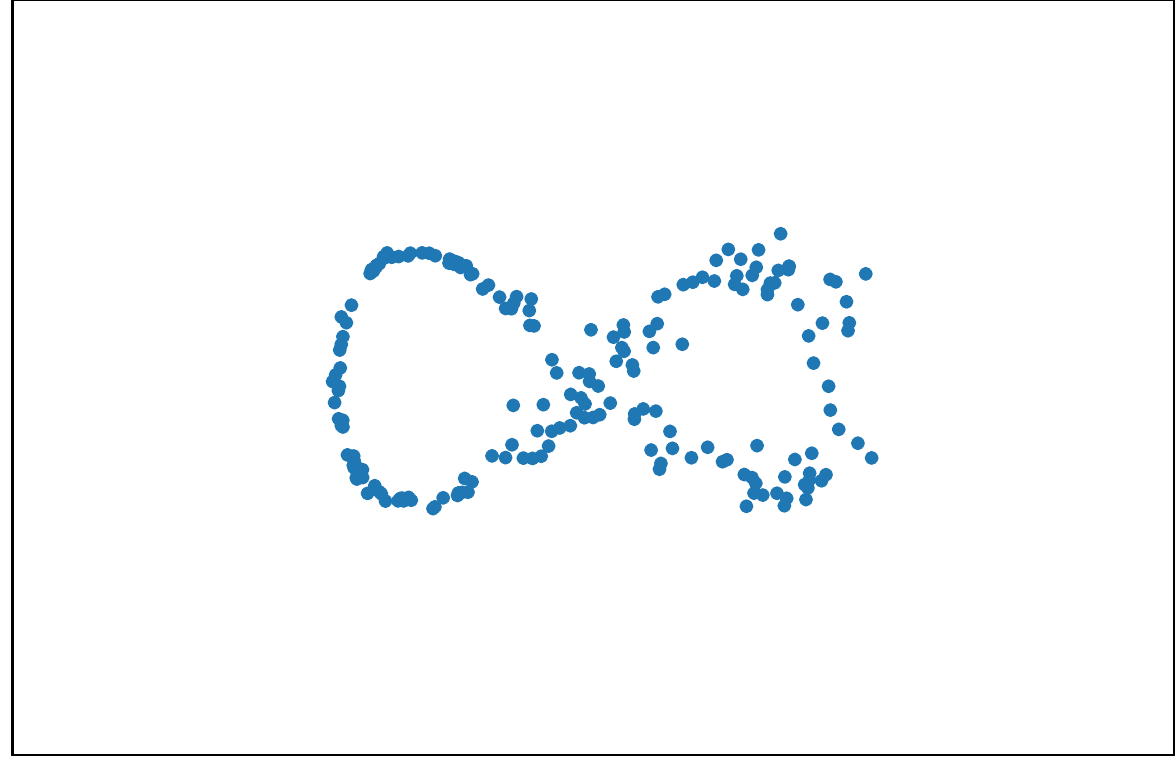}
 }
 
 \setcounter{subfigure}{0}
 \vspace{-10pt}
 \subfloat[Real distribution]{
 \label{fig:cluster_real_appendix}
 \centering
 \includegraphics[width=0.22\textwidth]{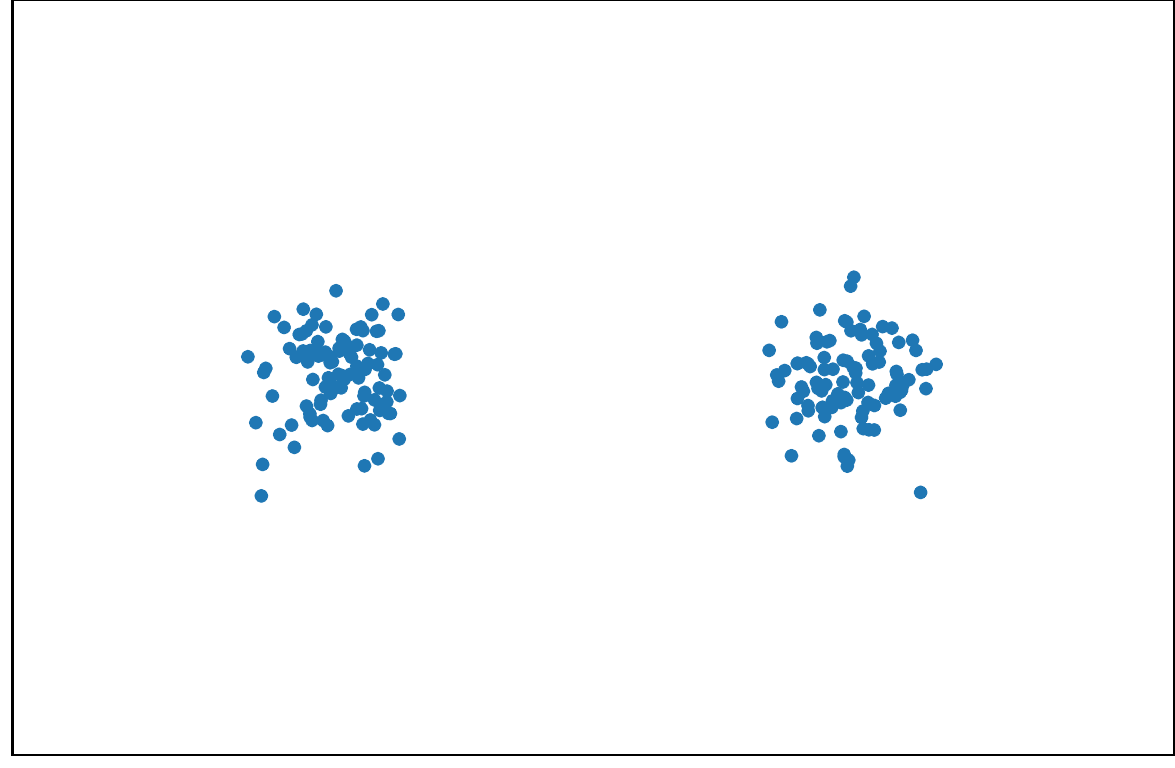}
 }
 \subfloat[VAE]{
 \label{fig:cluster_vanilla_appendix}
 \centering
 \includegraphics[width=0.22\textwidth]{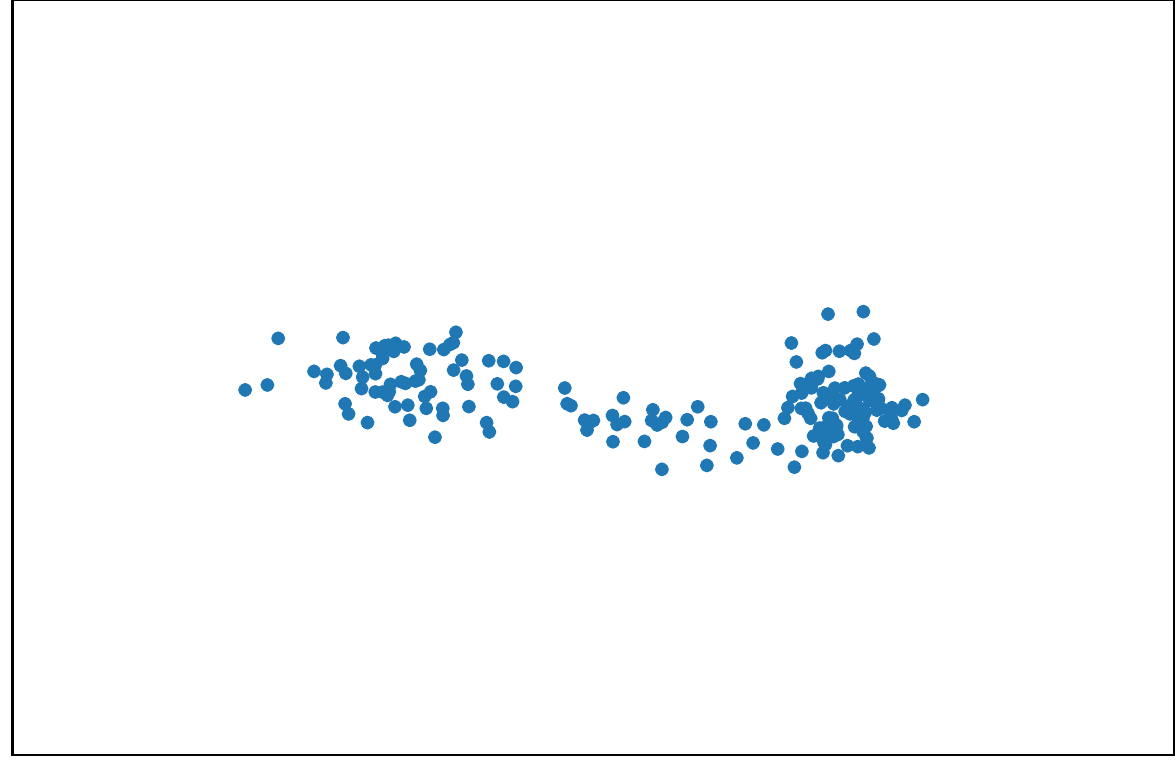}
 }
 \subfloat[Vamp-VAE]{
 \label{fig:cluster_vamp_appendix}
 \centering
 \includegraphics[width=0.22\textwidth]{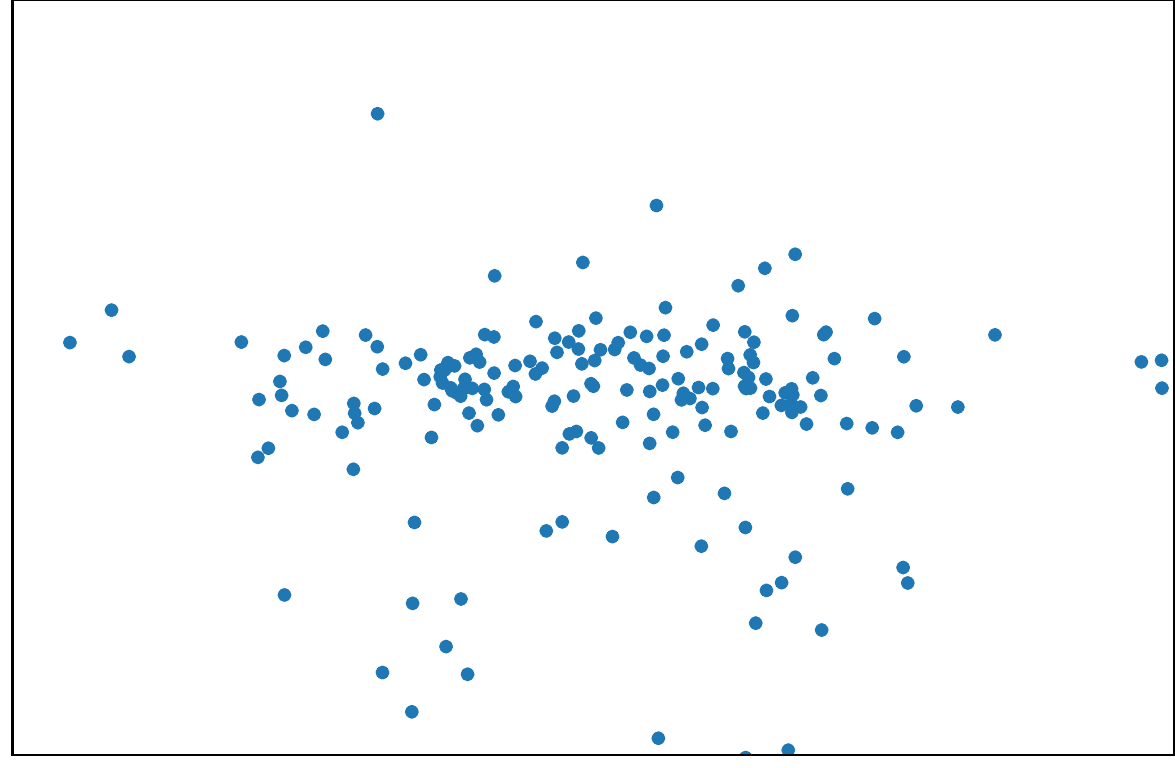}
 }
 \subfloat[\method]{
 \label{fig:cluster_ours_appendix}
 \centering
 \includegraphics[width=0.22\textwidth]{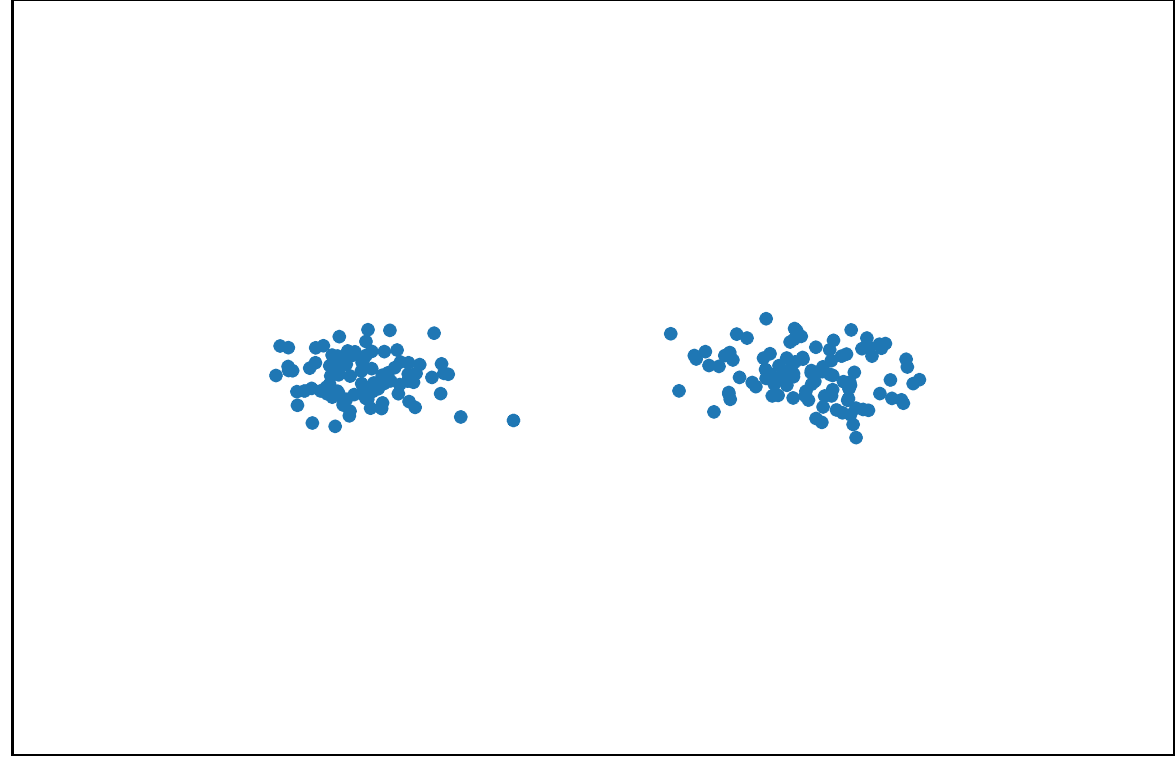}
 } 
 \caption{Extension of~\cref{fig:examples_low_dim} showing Vamp-VAE baseline and additional circular target distribution (top row, uses the same single hole $g_{\psi}$ as the second and third rows).
 }
%  \caption{Generation distributions of \method and vanilla-\gls{VAE} for multiply-connected and clustered distributions.
%  \method uses (1-3) circular prior with $h=1$, (4) multiply-connected prior with $h=2$, and (5) clustered prior, for each row respectively.
%  We find that with the circular prior, \method can learn a series of data distributions, such as squares and stars, as long as they have the same topological structure. From (4), we find that \method can learn distributions with more complex topological structures.
%  }
 \end{figure}

When dimension of $\Y$ is greater than $2$, we have more diverse choice for $g$. When $K$ is decomposable, i.e., $K=\prod_i K_i$, we can separately cut the plane expanded by $\Y_{2i}$ and $\Y_{2i+1}$ into $K_i$ sectors by the \cref{eqn:cluster}. As a result, $\Y$ is split into $K=\prod_i k_i$ clusters.
%$g$ only changes the first $2$ hidden dimensions. 
When $K=2$, we find that $g$ only changes the $1$-st dimension of $\Y$, so it can be applied to cases where latent dimension is $1$.

\paragraph{Learnable proportions}
We can also make the mapping more flexible by learning rather than assigning the cluster proportions. To do so, we keep a learnable value $u_i$ for each cluster and set the angle of the $i$-th sector as $2\pi \text{Softmax}(u)_i$. Things are simpler for the 1-dimensional case where we can uniformly translate $y$ by a learnable bias $b$ before splitting the space from the origin.

\subsection{Sparsity}
\label{sec:sparse_appendix}

\paragraph{Relationship to soft attention}
We note that our setup for the sparsity mapping shares some similarities with a soft attention layer~\citep{bahdanau2014neural}.
However, there are also some important points of difference.
Firstly, soft attention aims to find the weights to blend features from different time steps (for sequential data) or different positions (for image data). In contrast, the dimension selector (DS) selects which dimensions to activate or deactivate for the same latent vector. Secondly, the weights of features are usually calculated by inner products of features for soft attention, while DS relies on a network to directly output the logits.

\paragraph{Sparsity regularizer}
Our sparsity regularizer term, $\ELBO_{sp}$, is used to encourage our dimensionality selector network (DS) to produce sparse mappings. 
It is defined using a mini-batch of samples $\{y_i\}_{i=1}^{M}$ drawn during training as per~\eqref{eq:sparsity_reg}.
% and is given by
% \begin{align}
%     \ELBO_{sp}=\mathbb{E} \left[\frac{1}{M}\sum_{i=1}^M (H\left(DS(y_i))\right) - H\left(\frac{1}{M}\sum_{i=1}^M DS(y_i)\right)\right],
% \end{align}
% where $H(v) = -\sum_i \left(v_i/\|v\|_1\right) \log \left(v_i/\|v\|_1\right)$ is normalized entropy of an positive vector $v$, and the expectation is taken over the process of drawing a mini-batch of datapoints $\{x_i\}_{i=1}^M$ and then independently drawing an intermediary latent for each, $y_i\sim q_{\phi}(y|x=x_i)$. 
During training, the first term of $\ELBO_{sp}$ decreases the number of activated dimensions for each sample, while the second term prevents the samples from all using the same set of activated dimensions, which would cause the model to degenerate to a vanilla \gls{VAE} with a lower latent dimensionality. 
% The sparsity-regularized training objective now becomes
% \begin{align}
%     \ELBO_{\Y}(\theta,\phi,\psi) + \gamma ~\ELBO_{sp}(\phi,\psi),
% \end{align}
% where $\gamma$ is a hyper-parameter controlling the degree of sparsity enforced (with higher $\gamma$ corresponding to more sparsity).

We note that $\ELBO_{sp}$ alone is not expected to induce sparsity without also using the carefully constructed $g_{\psi}$ of the suggested \method.
%
%\paragraph{Ablation study}
We confirm this empirically by performing an ablation study on \textbf{MNIST} where we apply this regularization directly to a vanilla VAE.
%Since there is no dimension selector $DS$ in vanilla-VAE, we perform $L_{sp}$ on latent variable $z$ directly. 
We find that even when using very large values of $\gamma>30.0$ we can only slightly increase the sparsity score ($0.230\rightarrow 0.235$). Moreover, unlikely for the \method, this substantially deteriorates generation quality, with the FID score raising to more than $80.0$ at the same time.

\paragraph{Sparse metric}
We use the Hoyer extrinsic metric~\citep{hurley2009comparing} to measure the sparsity of representations. For a representation $z\in \R^D$,
\begin{align}
    \text{Hoyer}(z)=\frac{\sqrt{D}-||\hat{z}||_1/||\hat{z}||_2}{\sqrt{D}-1}.
\end{align}
Here, following~\citet{mathieu2019disentangling}, we crucially first normalized each dimension $d$ of $z$ to have standard deviation $1$, $\hat{z}_d = z_d/\sigma_d$, to ensure that we only measure sparsity that varies between data points (as is desired), rather than any tendency to uniformly `switch off' certain latent dimensions (which is tangential to our aims).
In other words, this normalization is necessary to avoid giving high scores to representations whose length scales vary between dimensions, but which are not really sparse.

By averaging $\text{Hoyer}(z)$ over all representations, we can get the sparse score of a method.
For the sparsest case, where each representation has a single activated dimension, the sparse score is $1$.
And when the representations get denser, $||\hat{z}||_2$ get smaller compared with $||\hat{z}||_1$, leading to smaller sparse scores.

{
%For ICLR rebuttal only
\begin{figure}[t]
 %\vspace{-1.5em}
 \centering
 \subfloat{
 \centering
 %\raisebox{0.4cm}{\rotatebox{90}{(b)~\textbf{Fashion-MNIST}}}
 \includegraphics[width=0.4\textwidth]{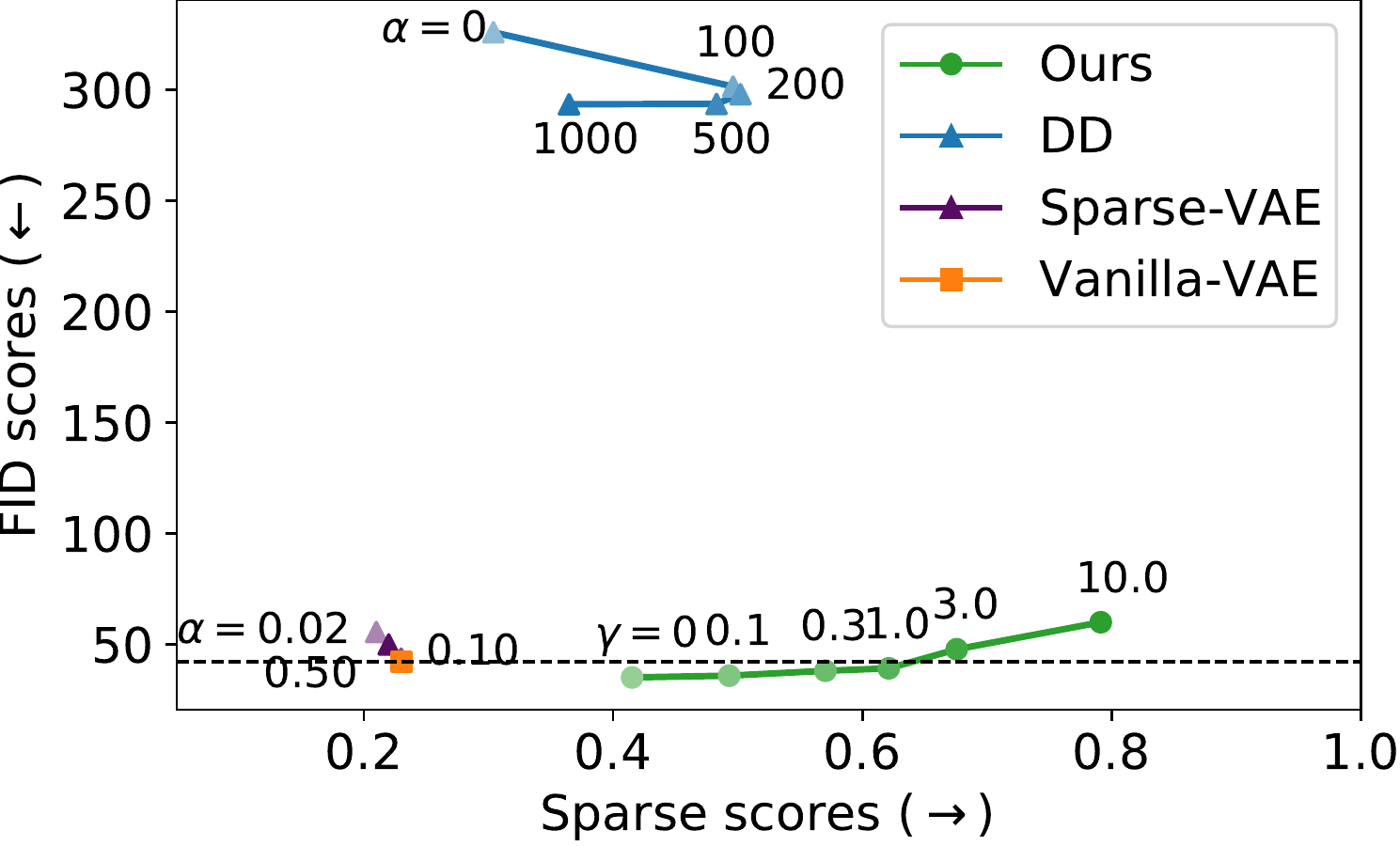}
 }
 \qquad
 \subfloat{
 \label{fig:sparse_downstream_fashion_mnist}
 \centering
 %\raisebox{0.5cm}{\rotatebox{90}{(b)~\textbf{Fashion-MNIST}}}
 \includegraphics[width=0.39\textwidth]{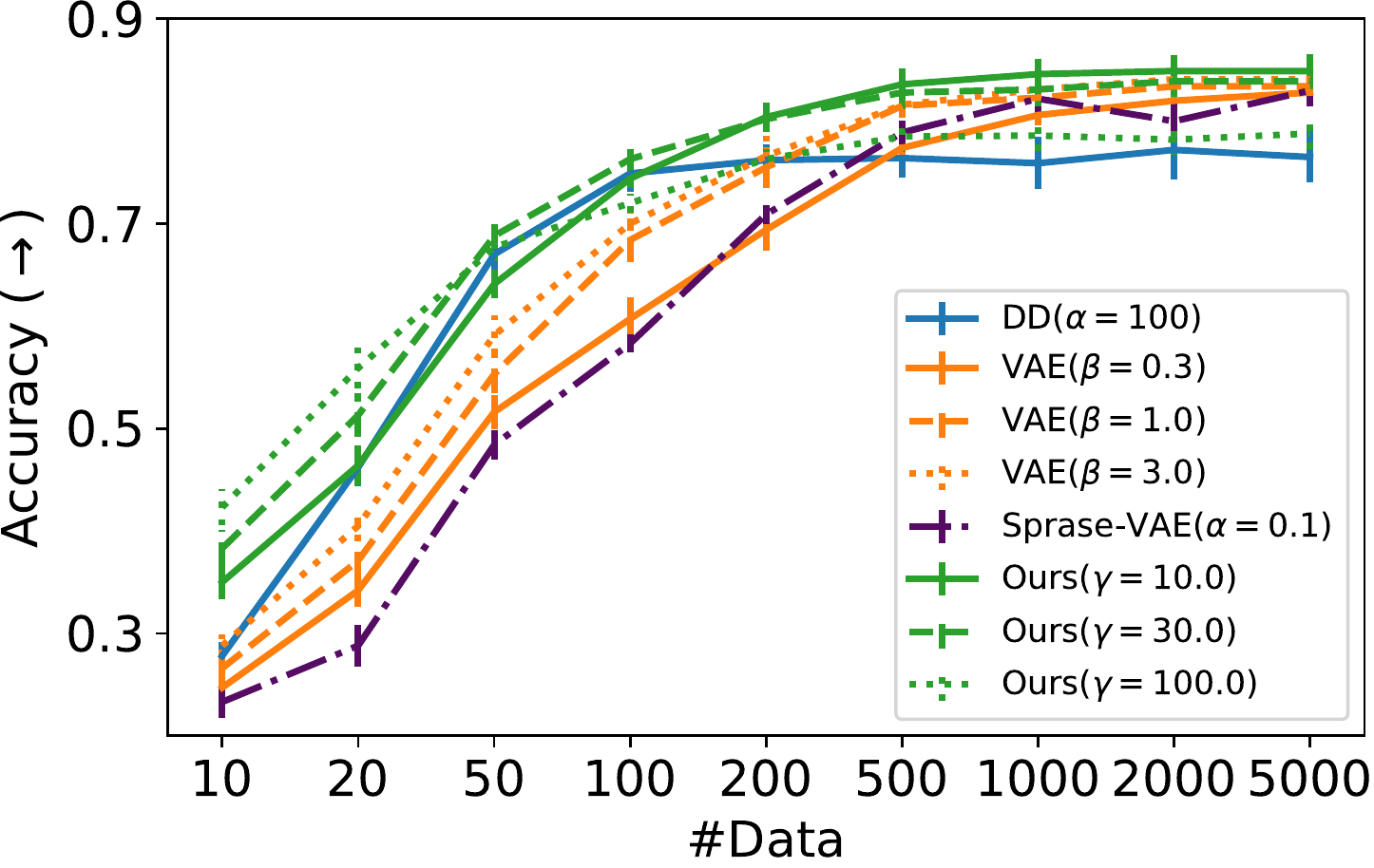}
 }
 \vspace{-6pt} 
 \caption{Results on \textbf{MNIST}. The \textit{left} figure shows FID and sparsity scores. Lower FID scores~($\downarrow$) represent better sample quality while higher sparse scores~($\rightarrow$) indicate sparser features.  The \textit{right} figure shows the performance of sparse features from \method on downstream classification tasks.
 See~\cref{sec:sparse} for details and results for \textbf{MNIST}.
 \vspace{-10pt}}
 \label{fig:fid_sparse_mnist}
\end{figure}}

\paragraph{Reproduction of Sparse-VAE}
We tried two different code bases for Sparse-VAE~\citep{tonolini2020variational}. 
The official code base\footnote{\url{https://github.com/ftonolini45/Variational_Sparse_Coding}} gives higher sparse scores for MNIST and FashionMNIST (though still lower than \method ), but is very unstable during training, with runs regularly failing after diverging and producing NaNs. 
This issue gets even more severe on CelebA which occurs after only a few training steps, undermining our ability to train anything meaningful at all.
To account for this, we switched to the codebase\footnote{\url{https://github.com/Alfo5123/Variational-Sparse-Coding}} from \cite{de2019replication} that looked to replicate the results of the original paper. We report the results from this code base because it solves the instability issue and achieves reasonable results on CelebA. Interestingly, though its generation quality is good on MNIST and Fashion-MNIST, it fails to achieve a sparse score significantly higher than vanilla-VAE. 
As the original paper does not provide any quantitative evaluation of the achieved sparsity, it is difficult to know if this behavior is expected.
We note though that the qualitative results shown in the paper appear to be substantially less sparse than those we show for the \method, cf their Figure 5 compared to the top row of our~\cref{fig:fashion_mnist_manipulation}.
In particular, their representation seems to mostly `switch off' some latents entirely, rather than having diversity between datapoints that is needed to score well under the Hoyer metric.

\begin{table}[t]
\centering
\small{
    \begin{tabular}{lrrrrr}
        \toprule
          \textbf{Parameters} & Synthetic & MNIST & Fashion-MNIST & MNIST-01 & CelebA\\\midrule
          Dataset sizes &Unlimited &55k/5k/10k&55k/5k/10k&10k/1k/2k&163k/20k/20k\\
          Input space & $\mathbf{R}^2$&Binary 28x28&Binary 28x28&Binary 28x28&RGB 64x64x3\\
          Encoder net &MLP &CNN&CNN&CNN&CNN\\
          Decoder net &MLP &CNN&CNN&CNN&CNN\\
          Latent dimension & 2-10&50&50&1-10&1-128\\
          Batch size &10-500&100&100&100&100\\
          Optimizer&Adam&Adam&Adam&Adam&Adam\\
          Learning rate&1e-3&1e-3&1e-3&1e-3&1e-3\\
         \bottomrule
    \end{tabular}
    \caption{Hyperparameters used for different experiments.}
    \label{tab:hyperparameter_appendix}
}
\end{table}

\begin{table}[t]
\centering
\small{
    \begin{tabularx}{5.5cm}{l}
        \toprule
          \textbf{Encoder} \\\midrule
          Input 64 x 64 x 3\\
          4x4 conv. 64 stride 2 \& BN \& LReLU\\
          4x4 conv. 128 stride 2 \& BN \& LReLU\\
          4x4 conv. 256 stride 2 \& BN \& LReLU\\
          Dense~($dim$)\\
         \bottomrule
    \end{tabularx}
    \quad\quad
    \begin{tabularx}{5.5cm}{l}
        \toprule
          \textbf{Decoder} \\\midrule
          Input $dim$\\
          Dense~(8x8x256) \& BN \& ReLU\\
          4x4 upconv. 256 stride 2 \& BN \& ReLU\\
          4x4 upconv. 128 stride 2 \& BN \& ReLU\\
          4x4 upconv. 3 stride 2\\
         \bottomrule
    \end{tabularx}
    \caption{Encoder and Decoder structures for CelebA, where $dim$ is the latent dimension.}
    \label{tab:celebA_model_appendix}
}
\end{table}

\subsection{Additional Experiment Details}

\paragraph{Datasets} Both synthetic and real datasets are used in this paper. All synthetic datasets~(sphere, square, star, and mixture of Gaussian) are generated by generators provided in our codes. For real datasets, We load MNIST, Fashion-MNIST, and CelebA directly from Tensorflow~\citep{tensorflow2015-whitepaper}, and we resize images from CelebA to $64$x$64$ following \cite{hou2017deep}.
For experiments with a specified number of training samples, we randomly select a subset of the training data. We use the same random seed for each model in the same experiment and different random seeds when repeating experiments.

\paragraph{Model structure} For low-dimensional data, the encoder and decoder are both simple multilayer perceptrons with 3 hidden layers~(10-10-10) and ReLU~\citep{glorot2011deep} activation. For MNIST and Fashion-MNIST, we use the same encoder and decoder as \cite{mathieu2019disentangling}. For CelebA, the structure of convolutional networks are shown in \cref{tab:celebA_model_appendix}.

\paragraph{Experiment settings} Other hyperparameters are shown in \cref{tab:hyperparameter_appendix}. All experiments are run on a GTX-1080-Ti GPU.

%\begin{wraptable}[16]{r}{5.5cm}
%\centering
%\small{
%    \vspace{-0.5cm}
%    \begin{tabularx}{5.5cm}{l}
%        \toprule
%          \textbf{Encoder} \\\midrule
%          Input 64 x 64 x 3\\
%          4x4 conv. 64 stride 2 \& BN \& LReLU\\
%          4x4 conv. 128 stride 2 \& BN \& LReLU\\
%          4x4 conv. 256 stride 2 \& BN \& LReLU\\
%          Dense~(dim)\\
%         \bottomrule
%    \end{tabularx}
%    \begin{tabularx}{5.5cm}{l}
%        \toprule
%          \textbf{Decoder} \\\midrule
%          Input $dim$\\
%          Dense~(8x8x256) \& BN \& ReLU\\
%          4x4 upconv. 256 stride 2 \& BN \& ReLU\\
%          4x4 upconv. 128 stride 2 \& BN \& ReLU\\
%          4x4 upconv. 3 stride 2\\
%         \bottomrule
%    \end{tabularx}
%    \caption{Encoder and Decoder structures for CelebA, where $dim$ is the latent dimension.}
%    \label{tab:celebA_model_appendix}
%}
%\end{wraptable}

\end{appendices}
\end{document}